\documentclass{article}

\usepackage{macros}
\usepackage{math_commands}
\usepackage{csquotes}
\usepackage{amsmath}
\newcommand{\abs}[1]{\lvert#1\rvert}

\usepackage{microtype}
\usepackage{graphicx}
\usepackage{subfigure}
\usepackage{multirow}
\usepackage{pifont}
\usepackage{xspace}
\usepackage[super]{nth}
\usepackage{booktabs} 
\usepackage[framemethod=TikZ]{mdframed}

\usepackage[a-3b]{pdfx} 

\usepackage{hyperref}

\usepackage[accepted]{icml2024}

\usepackage{amsmath}
\usepackage{amssymb}
\usepackage{mathtools}
\usepackage{amsthm}
\usepackage{microtype}
\usepackage{graphicx}
\usepackage{subfigure}
\usepackage{booktabs} 
\usepackage{enumerate}
\usepackage[shortlabels]{enumitem} 
\usepackage{bm}
\usepackage{bbm}
\usepackage{wrapfig}
\usepackage{booktabs}       
\usepackage{amsfonts}       
\usepackage{nicefrac}       
\usepackage{microtype}      
\usepackage{color, soul}
\usepackage{pifont}
\usepackage{nth}

\usepackage{algorithm}
\usepackage{algorithmicx}
\usepackage{algpseudocode}
\usepackage{adjustbox}
\usepackage{algorithm, listings}
\algrenewcommand\algorithmicrequire{\textbf{Input:}}
\algrenewcommand\algorithmicensure{\textbf{Output:}}
\definecolor{codegreen}{rgb}{0,0.6,0}
\definecolor{codegray}{rgb}{0.5,0.5,0.5}
\definecolor{codepurple}{rgb}{0.58,0,0.82}
\definecolor{backcolour}{rgb}{1,1,1}
\lstdefinestyle{mystyle}{
    backgroundcolor=\color{backcolour},   
    commentstyle=\color{codegreen},
    keywordstyle=\color{magenta},
    numberstyle=\tiny\color{codegray},
    stringstyle=\color{codepurple},
    basicstyle=\ttfamily\footnotesize,
    breakatwhitespace=false,         
    breaklines=true,                 
    captionpos=b,                    
    keepspaces=true,                 
    numbers=left,                    
    numbersep=5pt,                  
    showspaces=false,                
    showstringspaces=false,
    linewidth=.99\textwidth,
    showtabs=false,                  
    tabsize=2,
    numbers=none, 
}

\usepackage[capitalize,noabbrev]{cleveref}

\theoremstyle{plain}
\newtheorem{theorem}{Theorem}[section]
\newtheorem{proposition}[theorem]{Proposition}
\newtheorem{lemma}[theorem]{Lemma}
\newtheorem{corollary}[theorem]{Corollary}
\theoremstyle{definition}
\newtheorem{definition}[theorem]{Definition}

\theoremstyle{remark}
\newtheorem{remark}[theorem]{Remark}

\usepackage[textsize=tiny]{todonotes}

\newcommand{\cmark}{\ding{51}}
\newcommand{\xmark}{\ding{55}}

\newcommand{\window}{\textsc{tcWindow}\xspace}
\newcommand{\sysname}{\textsc{Based}\xspace}
\newcommand{\syntheticname}{\textsc{MQAR}\xspace}

\icmltitlerunning{Simple linear attention language models balance the recall-throughput tradeoff}

\begin{document}

\twocolumn[
\icmltitle{Simple linear attention language models balance the recall-throughput tradeoff}



\icmlsetsymbol{equal}{*}

\begin{icmlauthorlist}
\icmlauthor{Simran Arora}{equal,yyy}
\icmlauthor{Sabri Eyuboglu}{equal,yyy}
\icmlauthor{Michael Zhang}{equal,yyy}
\icmlauthor{Aman Timalsina}{sch}
\icmlauthor{Silas Alberti}{yyy}
\icmlauthor{Dylan Zinsley}{sch}
\icmlauthor{James Zou}{yyy}
\icmlauthor{Atri Rudra}{sch}
\icmlauthor{Christopher Ré}{yyy}
\end{icmlauthorlist}

\icmlaffiliation{yyy}{Stanford University}
\icmlaffiliation{sch}{University of Buffalo}

\icmlcorrespondingauthor{Simran Arora}{simarora@stanford.edu}
\icmlcorrespondingauthor{Sabri Eyuboglu}{eyuboglu@stanford.edu}
\icmlcorrespondingauthor{Michael Zhang}{mzhang20@stanford.edu}

\icmlkeywords{Machine Learning, ICML}

\vskip 0.3in
]



\printAffiliationsAndNotice{\icmlEqualContribution} 

\begin{abstract}
Recent work has shown that attention-based language models excel at \textit{recall}, the ability to ground generations in tokens previously seen in context. However, the efficiency of attention-based models is bottle-necked during inference by the KV-cache's aggressive memory consumption. 
In this work, we explore whether we can improve language model efficiency (\textit{e.g.} by reducing memory consumption) without compromising on recall.
By applying experiments and theory to a broad set of architectures, we identify a key tradeoff between a model's \textit{state size} and recall ability.
We show that efficient alternatives to attention (\textit{e.g.} H3, Mamba, RWKV) maintain a fixed-size recurrent state, but struggle at recall. 
We propose ${\sysname}$ a simple architecture combining linear and sliding window attention. 
By varying ${\sysname}$ window size and linear attention feature dimension, we can dial the state size and traverse the Pareto frontier of the recall-memory tradeoff curve, recovering the full quality of attention on one end and the small state size of attention-alternatives on the other. 
We train language models up to $1.3$b parameters and show that ${\sysname}$ matches the strongest sub-quadratic models (\textit{e.g.} Mamba) in perplexity and outperforms them on real-world recall-intensive tasks by 10.36 accuracy points.
We further develop IO-aware algorithms that enable {\sysname} to provide $24\times$ higher throughput on language generation than FlashAttention-2, when generating 1024 tokens using $1.3$b parameter models.
Overall, {\sysname} expands the Pareto frontier of the throughput-recall tradeoff space beyond prior architectures.
\end{abstract}

\section{Introduction}

The choice of sequence mixer (\textit{e.g.} attention, convolution) in a language model affects both its quality and efficiency~\citep{arora2023zoology,vaswani2018attention}. 
Prior work shows that attention excels at \textit{recall}, the ability to ground generations in previously seen tokens~\citep{olsson2022context,arora2023zoology}. 
On the other hand, the throughput of attention-based models is bottle-necked during training by quadratic compute complexity and during inference by aggressive memory consumption. 
The natural question is: \textit{can we improve the real-world speed and memory-use of language models without comprising on quality? }

\begin{figure*}[ht]
    \includegraphics[width=\textwidth]{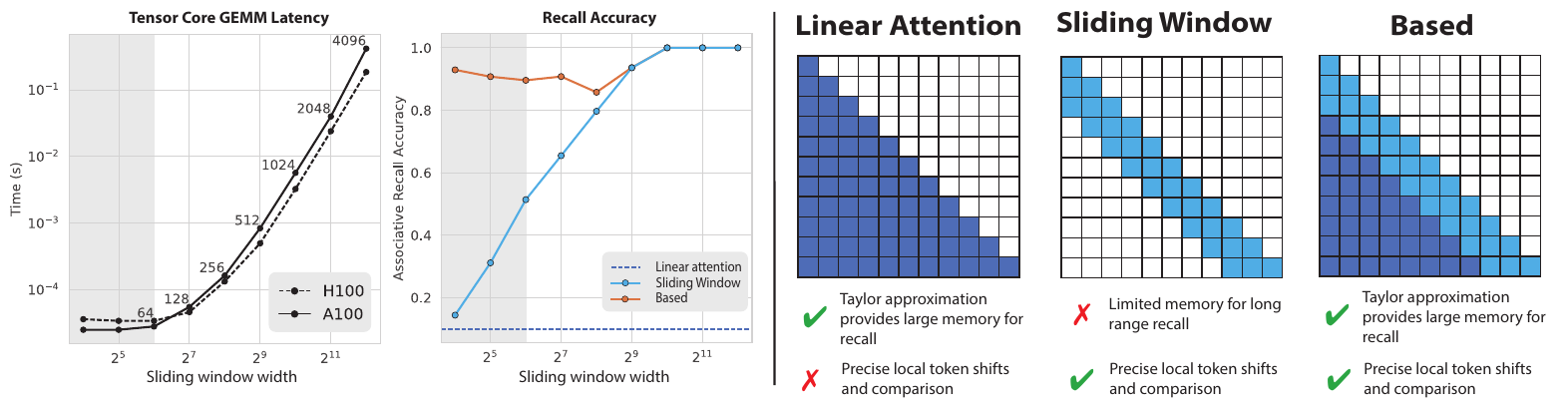}
    \caption{\textbf{\sysname overview.} Combining linear attention with \textit{tiny} sliding window softmax attention (e.g., 64 or 128 tokens in width) enables improved recall accuracy with limited efficiency overhead vs. smaller tile sizes. (\textit{Left}) Time to execute Cutlass GEMMs ($y$) vs. sliding window attention size ($x$), with batch size $512$ 
    on tensor cores.
    (\textit{Center)} Model recall accuracy ($y$) vs. sliding window attention size ($x$). 
    We compare linear attention alone (dark blue), sliding window attention alone (light blue), and their combination (\sysname, orange).  (\textit{Right}) Schematic diagram of \sysname illustrating how the two components complement each other.
    }
    \label{fig:based-design}
    \vspace{-2mm}
\end{figure*}

Recently, a number of architectures have been proposed that enable substantially higher throughput while competing with attention in perplexity~\citep{wang2022pretraining,gu2023mamba,yang2023gated,poli2023hyena,peng2023rwkv}. 
However, coarse metrics like overall perplexity can obscure important differences in model quality. For example, recent work shows that a specific class of architectures, \textit{gated-convolutions}, despite complexity scaling sub-quadratically in sequence length, are asymptotically less efficient than attention at performing  recall~\citep{arora2023zoology}. Building on this analysis, we evaluate a broader class of architectures across real-world recall-intensive tasks and show attention improves over a currently-popular attention-free alternative, Mamba, by 32.2 accuracy points~(\Cref{table:main-quality}). \footnote{Examples of recall-intensive tasks include information extraction, reading comprehension, summarization and code generation. These require using in context information (contrasting memorized information) during generation.}

Motivated by these observations, we explore the Pareto frontier of the tradeoff between high-recall and high-throughput models. 
We evaluate a range of architectures on a popular synthetic \textit{associative recall} task \cite{arora2023zoology, dao2022hungry, olsson2022context}. Since generation throughput is bottle-necked by memory consumption, we vary hyperparameters (e.g. model dimension) that affect the size of the \textit{recurrent} state during generation and demonstrate a fundamental recall-memory tradeoff that holds across architecture classes (\Cref{fig:tradeoff}).  
Attention performs associative recall perfectly, but the recurrent state (\textit{i.e.} the KV-cache) grows linearly with the sequence length. Sliding window attention (SWA) can cap the size of the recurrent state at the cost of worse long-range recall~\citep{mistral7b}.
However, Mamba, a recently proposed SSM architecture expands the Pareto frontier beyond SWA. 
This begs the question: \textit{are there other, perhaps simpler, models that can also expand the Pareto frontier?}

To reduce the memory consumption, we consider using two simple  techniques: SWA and softmax-approximating linear attention.
Our results on language modeling (\Cref{table:main-quality}) and synthetic recall experiments (\Cref{fig:based-design}, center) suggest neither primitive alone suffices to navigate the Pareto frontier.  
\begin{enumerate}[leftmargin=*]
    \item We find that \textit{linear attention} alone struggles to solve associative recall (\Cref{fig:based-design}, center). We hypothesize that this is because linear attention lacks the precision to perform local token shifts and comparisons~\cite{dao2022hungry,arora2023zoology}. 
    \item In \textit{sliding window attention}, associative recall range is limited by the width of the windows (\Cref{fig:based-design}, center). As we increase the window size, the recurrent state grows linearly and has a non-linear affect on speed during parallel training and inference (\Cref{fig:based-design}, left).
\end{enumerate}

We combine these two techniques into a single architecture, which we call \sysname (\Cref{fig:based-design}, right). We find that SWA
and linear attention
complement each other, enabling \sysname to expand the pareto frontier of the recall-memory tradeoff (\Cref{fig:tradeoff}). We suspect that (1) the large recurrent memory of linear attention could help model long-range token interactions in the sequence and (2) SWA handles the precise local shifts needed to perform associative recall.

To make \sysname competitive with SoTA attention~\citep{dao2023flashattention2} and recurrent~\citep{gu2023mamba} models under wall-clock and throughput metrics, we introduce several IO-aware optimizations.
\begin{enumerate}[leftmargin=*]
    \item Despite the theoretical efficiency benefits, \textit{linear attention} implementations are often \textit{slower} than well-optimized attention implementations ~\citep{dao2022flashattention}. To make our attention competitive in real-world wall-clock time and memory usage, we provide hardware-efficient CUDA algorithms for liner attention generation prefill (\Cref{alg:a0_fwd}) and decoding (\Cref{alg:lin_gen}). 
    
    In \sysname, we show that the 2nd-order Taylor approximation of softmax as the linear attention feature map is hardware-efficient. 
    With sequence length $N$ and head dimension $d$, this naïvely requires $\mathcal{O}(Nd^3)$ time and space complexity 
    ~\citep{hedgehog2023,keles2023on}. Relative to the baseline, our algorithm reduces data movement from HBM (slower-to-access memory) to SRAM (faster-to-access memory)
    by $\mathcal{O}(Nd^2)$ bytes and from SRAM to register by $O(Nd^3)$ bytes
    (Section~\ref{sec:efficiency}).

    \item \textit{Sliding window attention} exploits tensor cores, specialized units on modern GPUs for performing matrix multiplications (GEMMs). While prior architectures use large window sizes (e.g. 4096 for Mistral-7B \cite{mistral7b}), we propose to use small $64-128$ windows, guided by hardware properties. Size $64-128$ window sizes keep the tensor  cores occupied \Cref{fig:based-design} (left).
\end{enumerate}

In experiments, we show that ${\sysname}$ competes in quality with strong Transformer++ \citep{touvron2023llama} and SoTA sub-quadratic baselines in models up to the 1.3Bn parameters across language modeling on the Pile language, DNA modeling, and the LM Eval Harness \citep{eval-harness}. Beyond this, ${\sysname}$ outperforms a strong sub-quadratic architecture, Mamba, on the associative recall slice of the Pile and in downstream recall-intensive tasks by $10.36$ accuracy points. In efficiency, ${\sysname}$ enables up to $24\times$ higher throughput than the strong FlashAttention-2 implementation on generation. Code for this work is provided at: \url{https://github.com/HazyResearch/based}.

\vspace{-1mm}
\section{Preliminaries and Related Work}
\vspace{-2mm}
We discuss the key relevant work in this section and provide an extended discussion in \Cref{app:extended_related_work}.
\vspace{-2mm}

\paragraph{Attention}  The \textit{de facto} language modeling primitive, softmax attention~\citep{vaswani2018attention} takes inputs $\bm{x} \in \mathbb{R}^{N \times d}$ of length $N$ and head dimension $d$, and computes outputs $\bm{y} \in \mathbb{R}^{N \times d}$ via the softmax over projections 
\(
\bm{q}, \bm{k}, \bm{v} = \bm{x} \bm{W}_q, \bm{x} \bm{W}_k , \bm{x}\bm{W}_v
\),
\textit{i.e.},
\begin{equation}
\bm{y}_i = \sum_{j=1}^i\frac{\exp(\bm{q}_i^\top \bm{k}_j /\sqrt{d}) \bm{v}_j }{\sum_{m = 1}^{i} \exp(\bm{q}_i^\top \bm{k}_m /\sqrt{d})} 
\label{eq:softmax_attention}
\end{equation}
in the causal case where $\bm{W}_q, \bm{W}_k, \bm{W}_v \in \mathbb{R}^{d \times d}$ are learnable matrices . 
While effective at recall~\cite{arora2023zoology} and efficient to train (Eq~\ref{eq:softmax_attention} is parallelizable on GPUs and $\mathcal{O}(N)$ in memory with recent advances~\citep{dao2022flashattention}), attention remains expensive for generation. For every new output $\bm{y}_n$, we require $nd$ operations over a growing \textit{KV-cache} of prior $\{\bm{k}_i, \bm{v}_i\}_{i=1}^{n-1}$. This results in larger memory consumption and lower-throughput for longer sequences.

\vspace{-2mm}
\paragraph{Efficient attentions} Various works thus try to improve on attention's efficiency without sacrificing quality.
\emph{Sparse attentions} reduce attention's time and memory requirements by only attending over specific strided patterns or local \emph{sliding windows}~\cite{parmar2018image,child2019generating,beltagy2020longformer}. While further popularized in large language models (Mistral,~\citet{mistral7b}), prior works either underperform full attention with sparse patterns that fail to capture dense interactions, or use large window sizes that still permit large KV-caches and subsequent inefficiency. 

Meanwhile, \emph{linear attentions} replace the softmax in standard attention with alternative kernel functions~\cite{katharopoulos2020transformers, choromanski2020rethinking, choromanski2021hybrid, qin2022cosformer,keles2023on}. By removing  the $\exp(\bm{q}^\top \bm{k})$ in favor of feature map dot-products $\phi(\bm{q})^\top\phi(\bm{k})$, these methods use matrix product associativity to compute attention in $\mathcal{O}(Nd^2)$ time and space~\cite{katharopoulos-et-al-2020}. Furthermore, they permit a \emph{recurrent view} for constant memory and $\mathcal{O}(1)$ time per-token generation~\cite{kasai-etal-2021-finetuning,schlag2021linear}. However, present linear attention feature maps either fail to match standard attention on recall or remain expensive to compute~\citep{hedgehog2023}. Linear attentions are slower in wall-clock time compared to optimized attention implementations~\citep{dao2022flashattention}.

The line of work studying how to combine sparse and linear attention into a single layer is also closely related to our work~\citep{zaheer2020bigbird,beltagy2020longformer,chen2021scatterbrain,zeng2022mra}. 

\vspace{-2mm}
\paragraph{Attention alternatives} Finally, various models use attention-free sequence mixers such as state-space models (SSMs)~\cite{gu2021efficiently, sun2023retentive}, gated convolutions~\cite{dao2022hungry,poli2023hyena} and input-dependent recurrences~\cite{peng2023rwkv, gu2023mamba} to rival attention performance while improving its efficiency. However, while recent such models can match attention in overall perplexity, further study suggests they may underperform Transformers on tasks such as recall and in-context learning~\cite{arora2023zoology, akyurek2024incontext}. 

\section{No Free Lunch: Memory-Recall Tradeoff}

In this section, we demonstrate a fundamental tradeoff between a model's memory consumption during inference (\textit{i.e.,} the size of its recurrent state) and its capacity to perform recall. We use a combination of experiments on synthetic data and theoretical analysis.

\begin{itemize}[leftmargin=*]
    \item \textbf{Empirical study of memory-recall tradeoff :} In Section~\ref{sec:motivation_empirical}, we evaluate a number of popular architecture classes (\textit{e.g.} Mamba, Hyena) on a synthetic associative recall task, varying hyperparameters that affect the model's recurrent state size (\Cref{fig:tradeoff}). Within each architecture class, we observe a clear tradeoff: the larger the recurrent state size, the better recall.
    However, for a fixed recurrent state size, performance is not consistent across architectures.
    We observe that some sequence mixers fall well-below the Pareto-frontier. This motivates the design of sequence mixers that can expand the Pareto frontier. 
    \item \textbf{Lower bounds on memory required for recall}: In Section~\ref{sec:motivation_theory}, we lower bound the recurrent state size required to perform exact recall with \textit{any} recurrent model~\Cref{thm: gen-ac}. 
    This analysis reinforces our empirical observations on the throughput-recall tradeoff.
\end{itemize}

\begin{figure}[h]
    \begin{center}
    \includegraphics[width=0.36\textwidth] 
    {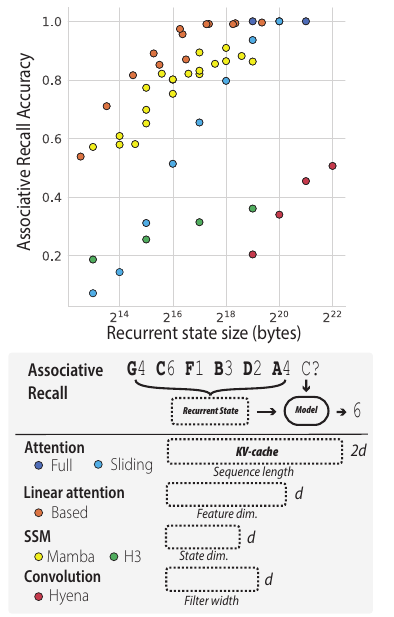}
    \caption{
        \textbf{Throughput (memory) - recall tradeoff.} $x$-axis shows state size (bytes) during generation; $y$-axis shows accuracy on the MQAR
        \syntheticname
        recall task \citep{arora2023zoology}.
        For each architecture, we train several models varying hyperparameters that affect the recurrent state size (\textit{e.g.} model dimension). 
        The plot shows a fundamental tradeoff between the recurrent state size and recall capacity that applies to broad class of models~\cite{arora2023zoology,gu2023mamba, dao2022hungry}.
    }
    \vspace{-2em}
    \label{fig:tradeoff}
    \end{center}
\end{figure}

\subsection{Empirical study of memory-recall tradeoff}
\label{sec:motivation_empirical}
\textbf{Setup.} We use a synthetic AR task called Multi-Query Associative Recall (${\Task}$) \cite{arora2023zoology} to demonstrate the trade-off.
In this task, input sequences consist of a number of key-value pairs followed by queries. For a given query, the model must recall the corresponding key-value pair from earlier in the sequence in order to predict the next token. 
For example, the correct output for input below would be $\text{4, 6, 1, 2, 3}$:
$$
\text{A 4 B 3 C 6}\underbrace{\text{F 1}}_{\mathclap{\textbf{Key-Value}}}  \text{E 2} \rightarrow \text{A ? C ?} \underbrace{\text{F ?}}_{\mathclap{\textbf{Query}}} \text{E ? B ?}
$$
We train on sequences of length 256 tokens containing between 4 and 64 key-value pairs. During evaluation, we measure accuracy on sequences of length 1,024 tokens containing between 4 and 256 key-value pairs.

We train and evaluate six sequence mixers: attention~\cite{vaswani2018attention}, sliding window attention~\cite{beltagy2020longformer}, Mamba~\cite{gu2023mamba}, H3~\cite{dao2022hungry}, Hyena~\cite{poli2023hyena}, and \sysname.  
For each, we vary hyperparameters that affect the memory consumption during inference (e.g., in sliding window attention we vary the window width).
We measure how \syntheticname accuracy varies with the size of the recurrent state and \Cref{app:exp-architectures} contains details on how state sizes are calculated.

\Cref{fig:tradeoff,fig:feature-map-tradeoff} can be reproduced or extended to new architectures using the scripts provided at \url{https://github.com/HazyResearch/zoology}.

\paragraph{Results}
In Figure \ref{fig:tradeoff}, we demonstrate a fundamental tradeoff between recurrent state size and accuracy on ${\Task}$ that holds within and across architecture classes. 
Within each architecture class (\textit{e.g.} H3 models), increasing the recurrent state size almost always leads to an improvement in accuracy. 
Across architecture classes, we see a tradeoff as well. Attention achieves perfect recall accuracy, but its recurrent state size grows with the length of the sequence. Other architecture classes like Mamba and H3 admit models with much smaller recurrent states, but these models have limited recall capacity.

Given a fixed recurrent state, not all architectures have the same recall capacity. 
Among architectures proposed in prior work, Mamba makes the best use of a limited memory budget. Notably, architectures with a convolutional view (\textit{e.g.} Hyena and H3) fall well below the Pareto frontier. Our proposed architecture, ${\sysname}$ (introduced in \Cref{sec:architecture}), expands the Pareto-frontier beyond Mamba. By varying hyper-parameters that determine its state size (\textit{e.g.} feature dimension and model dimension), we can smoothly navigate the tradeoff between efficient models and memory-hungry models with high recall capacity.

\subsection{Theoretical Analysis}
\label{sec:motivation_theory}
Our theoretical analysis provides further insight into the empirical observations described above. First, using results from communication complexity theory, we show that the recall capacity of \textit{any} causal model (\textit{e.g.} Mamba, Attention) is bounded by the size of its recurrent state~(\Cref{thm: space-reg} in \Cref{sec: theory}). 
\begin{theorem}
    Any recurrent model\footnote{For Mamba~\cite{gu2023mamba}, see \Cref{cor: mamba-ar}.} depending causally on input $\vu \in \{0,1\}^{N \times d}$ requires $\Omega(N)$-bits\footnote{Here, we need the entries of the state to be bounded.} in state size to solve $\Task$.
\end{theorem}
This result suggests that the tradeoff observed in \Cref{fig:tradeoff} is fundamental, not an artifact of architectural quirks. 

Next, we focus on \textit{gated-convolutions}, a broad class of architectures built from gating and convolutions (\textit{e.g.} H3, Hyena, RWKV v4). 
To make progress in theoretically analyzing the broad set of gated convolution proposals, prior work develops a \textit{canonical} gated-convolution, referred to as \BaseConv\, which can provably simulate \textit{any} architecture built from gating and convolution primitives.

Building on this work, we show that \BaseConv\ cannot solve $\Task$ in constant-many layers~(\Cref{thm: mqar-1hot} and \Cref{thm: mqar-phot} in \Cref{sec: theory}).

\begin{theorem}
\label{thm:loglogc-lb}
Given an input sequence $\vu \in \{0,1\}^{3\inputLength\times \modelDim}$, where $\inputLength$ and $\modelDim$ denote the sequence length and head dimension, respectively, a data-independent \BaseConv\ model needs $\log(2d)$-layers to solve $\Task$ for $d = \log_2(c)$, where $c$ denotes the vocabulary size\footnote{That is, each token from the vocabulary has the natural binary encoding in $\{0,1\}^{\log_2(c)}$}.
\end{theorem}
\begin{remark}
For a class of input encodings that generalizes one-hot encodings termed as $p$-hot encodings (\Cref{def: phot}), input-dependent \BaseConv\ needs at least $\floor{\log(2p)}$-layers to solve ${\Task}$ where $d = p \cdot \sqrt[p]{c}$.
\end{remark}

The above result is not as strong when $c\ll N$, for which we prove a complementary lower bound (\Cref{thm: poly-layer} in \Cref{sec: theory}):
\begin{theorem}
\label{thm: poly-layer-main}
    Given an input $\vu \in \{0,1\}^{\inputLength \times \modelDim}$ to the $\Task$ with any encoding such that $\log{c} \le d \le 2^{(\log{N})^{1-\epsilon}}$ for $\epsilon > 0$,
    and $c$ possible tokens from the vocabulary with $c \le N$,
    a data-independent \BaseConv\ model with model parameters taking $O(\log{N})$ bits needs $\Omega(\epsilon\log\log{N})$ layers to solve AR.
\end{theorem}

In contrast, \citet{arora2023zoology} show that attention solves ${\Task}$ in constant-many layers. 
This result helps to explain why the gated-convolution architectures (H3 and Hyena) in \Cref{fig:tradeoff} lie below the Pareto frontier established by newer architectures.

Note that \Cref{thm:loglogc-lb} and \Cref{thm: poly-layer-main} imply that we need $\Omega(\max(\log\log{c},\log\log{N}))$ many \BaseConv\ layers to solve MQAR. One might wonder if we can improve this lower bound. In \Cref{MQAR:math}, we show that this is the best possible lower bound by showing that for certain settings, $O(\max(\log\log{c},\log\log{N}))$ \BaseConv\ layers {\em are} enough to solve MQAR.

Finally, we show that we can simulate linear attention~\citep{katharopoulos2020transformers}, the foundation of ${\sysname}$, using \BaseConv\ \citep{arora2023zoology} with a poly-log blowup in the number of layers~(\Cref{prop: based-sim} in \Cref{sec: theory}), pointing to the relative efficiency of linear attention over gated-convolution architectures.

\section{The ${\sysname}$ Architecture}
\label{sec:architecture}

In this section, we introduce \sysname. Our objective in designing this architecture is to demonstrate how we can navigate the Pareto-frontier of the memory-recall tradeoff using well-known architectural building blocks.

Softmax attention excels at recall, but since its recurrent state, the KV-cache, grows unconstrained with the length of sequence, it is stuck in the upper right quadrant of \Cref{fig:tradeoff}.
We study two simple approaches for constraining the size of attention's recurrent state: linear attention and sliding window attention.
The recurrent state size of linear attention (\textit{i.e.} attention without softmax) does not grow with the sequence length and can be modulated by changing simple hyperparameters~\cite{katharopoulos2020transformers}. With sliding window attention, we cap the recurrent state size to be the width of the window.

However, our experiments on real-world language modeling (\Cref{table:ablations}) and synthetic associative recall (\Cref{fig:based-design} middle) suggest that neither primitive alone suffices to navigate the pareto frontier. 
Linear attention lacks the precision to perform local token shifts and comparisons \citep{fu2023simple,arora2023zoology}. 
In sliding window attention, associative recall range is limited by the width of the windows (Figure 2, center). As we increase the window size, the recurrent state grows linearly and has a non-linear effect on speed during parallel training and inference (Figure 2, left).

\sysname combines (1) softmax-approximating linear attention applied globally and (2) exact softmax attention applied locally in small sliding windows  (\Cref{fig:based-design}, right). This allows us to use softmax attention in surprisingly small sliding windows (\textit{e.g.,} $64-128$ tokens) that recover $90.8\%$ of full softmax attention's recall accuracy at 1e-5$\times$ its latency. 

\subsection{Taylor Linear Attention}
\label{sec:taylor_series_linear_attention}
By approximating softmax attention using linear feature maps, we can constrain the size of the recurrent state while maintaining global token interactions (\textit{i.e.} each token depends on every token before it in the sequence).

\citet{katharopoulos2020transformers,tsai2019transformer,choromanski2020rethinking} show that we can select a feature map $\phi: \mathbb{R}^d \rightarrow \mathbb{R}^{\tilde{d}}$ such that $\phi(\bm{q}_i)^\top \phi(\bm{k}_j) \approx \exp(\bm{q}_i^\top \bm{k}_j / \sqrt{d})$. We can then rewrite the formula for softmax attention in \Cref{eq:softmax_attention} as

\begin{equation}
\label{eq:linear_attention} 
    \sum_{j = 1}^i\frac{ \phi(\bm{q}_i)^\top  \phi(\bm{k}_j) \bm{v}_j 
    }{\phi(\bm{q}_i) \sum_{j = 1}^i \phi(\bm{k}_j)}
    = \frac{\phi(\bm{q}_i) \sum_{j = 1}^i \big( \phi(\bm{k}_j)^\top \bm{v}_j \big)
    }{\phi(\bm{q}_i) \sum_{j = 1}^i \phi(\bm{k}_j)}
\end{equation}
where every query attends to every past key in $\mathcal{O}(Nd^2)$ time and space complexity. 
Furthermore,~\citet{katharopoulos-et-al-2020} show that linear attention has a fixed size recurrent state during generation. Letting $\bm{s}_i = \sum_{j=1}^i \phi(\bm{k}_j)^\top \bm{v}_j$ and $\bm{z}_i = \sum_{j=1}^i \phi (\bm{k}_j)^\top$ be a ``KV-state'' and ``K-state'' respectively, we can compute \Cref{eq:linear_attention} as
\[
    \bm{s}_{i} = \bm{s}_{i-1} + \phi(\bm{k}_i)^\top \bm{v}_i,\;\;\; 
    \bm{z}_i = \bm{z}_{i - 1} + \phi (\bm{k}_i)^\top,
\]
\begin{equation}
\label{eq:linear_attention_recurrent_3} 
    \bm{y}_{i} = \frac{\phi(\bm{q}_i)\bm{s}_i}{ \phi(\bm{q}_i)\bm{z}_i}
\end{equation}
where $\bm{s}_{i} \in \mathbb{R}^{d \times \tilde{d}}$ and $\bm{z}_i \in \mathbb{R}^{\tilde{d}}$.

\paragraph{Feature map.} To approximate $\exp(\bm{q}_i^\top \bm{k}_j / \sqrt{d})$, we use the $2^\text{nd}$-order Taylor series feature map, picking $\phi: \mathbb{R}^d \rightarrow \mathbb{R}^{d^2} $ such that
\begin{equation}
    \label{eq:taylor_feature_map}
    \phi(\bm{q}_i)^\top \phi(\bm{k}_j) = 1 + \bm{q}_i^\top \bm{k}_j + \frac{(\bm{q}_i^\top \bm{k}_j)^2}{2}
\end{equation}

While~\citet{hedgehog2023} note that picking a feature map with $\tilde{d} = d^2$ results in linear attention with $\mathcal{O}(Nd^3)$ time and space complexity and large recurrent state of size $O(d^3)$, we can tradeoff efficiency for recall capacity by projecting queries and keys to smaller dimensions  \textit{i.e.,} $\bm{W}_q, \bm{W}_k \in \mathbb{R}^{d \times d'}$ with $d' = 16$. By changing $d'$ we modulate the size of the recurrent state.

\begin{figure}[h]
    \begin{center}
    \includegraphics[width=0.31\textwidth]
    {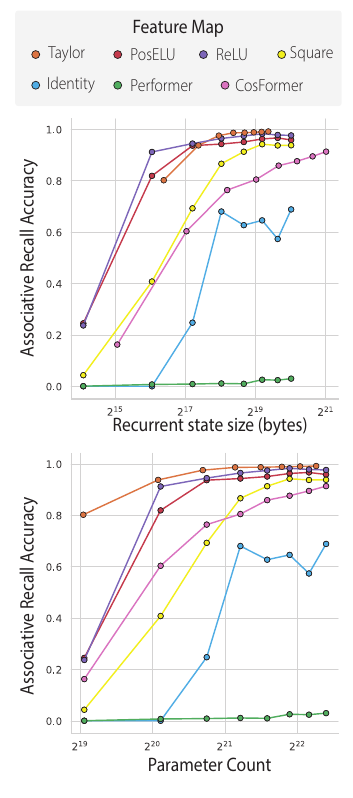}
    \vspace{-5mm}
    \caption{ \textbf{Linear attention feature maps on AR.} $x$: state size (bytes) during generation or param. count; $y$: \syntheticname accuracy. This  setting is harder than \cref{fig:tradeoff} (256 key-value pairs). }
    \vspace{-2.0em}
    \label{fig:feature-map-tradeoff}
    \end{center}
\end{figure}

\textit{How does the choice of feature map affect the memory-recall tradeoff?} Prior work demonstrates the strong performance of the Taylor feature map on associative recall~\citep{hedgehog2023}. Building on this analysis, we evaluate a broad set of feature maps 
($\phi_\text{ReLU}(x) = \max(x, 0)$, $\phi_\text{PosELU}(x) = \text{ELU}(x) + 1$, $\phi_\text{Square}(x) = x^2$, $\phi_\text{Identity}(x) = x$, $\phi_\text{CosFormer}$ as defined in \citep{qin2022cosformer}, and $\phi_\text{Performer}$ as defined in \citep{choromanski2020rethinking}) using the experimental setup described in \Cref{sec:motivation_empirical}.
In \Cref{fig:feature-map-tradeoff} (top), we plot the memory-recall tradeoff curves for these feature maps. The Taylor series feature map, along with the simple $\phi_\text{PosELU}$ and $\phi_\text{ReLU}$ feature maps, sits at the Pareto frontier. 
One advantage of the Taylor feature map over these alternatives is that it expands the recurrent state size (improving recall capacity) without changing the number of parameters. As shown in \Cref{fig:feature-map-tradeoff} (bottom), the Taylor series feature map requires fewer parameters than alternatives to achieve high recall capacity. This analysis and the ablations in \Cref{table:ablations} informed our decision to use the Taylor approximation, though other simple feature maps may be effective as well.

\subsection{Local Exact Attention with Small Sliding Windows}
\label{sec:tensor_core_window_attention}
To efficiently model fine-grained local interactions, ${\sysname}$ uses sliding window attention (SWA) with window sizes set at small multiples of $16$ (up to 64 tokens). Similar to past (causal) implementations~\citep{child2019generating,beltagy2020longformer}, for window size $w$ each query $\bm{q}_i$ only attends to past keys $\{\bm{k}_{i-w+1}, \ldots, \bm{k}_i\}$. This enables $\mathcal{O}(Nw)$ time and space complexity for linear scaling in sequence length $N$, with a $w$-sized KV-cache for constant-memory generation. 

However, unlike past SWA that keep $w$ at sizes 256~\citep{parmar2018image} to 4096~\citep{mistral7b}, \sysname uses only $w \leq 128$ to best exploit modern GPUs. In \Cref{sec:efficiency}, we discuss how this ``Tensor core-aware'' window (\window) achieves 1e-5$\times$ the latency than the $w=4096$ windows in modern LLMs (\textit{e.g.,} Mistral 7B~\citep{mistral7b}).

While the small $w$ in \window enable fast local and exact attention, it presents a challenge for long range modeling. With just $w=64$, for every layer of $w=4096$ Mistral sliding window attention we would require $64$ layers of \sysname to achieve the same receptive field. Controlling for model depth and sequence length, \Cref{fig:tradeoff} indeed shows smaller $w$ linearly decreasing in associative recall accuracy. \sysname's global \emph{linear attention} described above overcomes the lack of long-range modeling presented with low $w$. 

Finally, we find that including gated convolution layers with short convolutions (e.g., filter size $3$) gives additional benefit over only using \window layers. Short convolutions can help perform local, precise shifts for token comparisons since they operate over the full sequence, while \window does not. These local mixers can complement one-another.

Additional architectural details for {\sysname} are discussed in \Cref{app:extended-arch} and the hybridization of layers used in experiments are provided in \Cref{tab:based-training-details}. We include ablations of architectural choices in \Cref{table:ablations} and evaluate the overall quality and efficiency of {\sysname} in \Cref{sec:results}.

\section{Efficient Implementation}
\label{sec:efficiency}
In this section we focus on the efficiency of ${\sysname}$. A naïve implementation is \textit{slower} than the most efficient standard attention implementations (shown in \Cref{fig:forward_throughput}) as it requires large amounts of high latency memory movement. We first describe preliminaries of the GPU execution model and memory hierarchy. We next present the baseline and our hardware-aware algorithms for linear attention in \Cref{sec:eff-tay} and for sliding window attention in \Cref{sec:eff-slide}.

\subsection{Preliminaries} GPU operations, or \textit{kernels}, are executed by many parallel threads. GPU \textit{streaming multiprocessors} launch \textit{thread blocks} at the software level. These blocks are divided into \textit{warps} (\textit{e.g.} 32 threads) that are assigned to cores at the hardware level. Threads need to read inputs into their \textit{registers} to perform computations and write  the outputs. The time taken to read and write is referred to as the IO cost.

Operations could either be memory or compute bound, depending on the time to load data vs. perform computations on loaded data. In designing our IO-aware algorithms, we would like to exploit two key properties of modern GPUs. First, tensor core units (fast matrix multiply units) achieve 312 TFLOP/s speeds relative to 19 TFLOP/s for the non-matrix multiply cores. Second, GPUs face a memory hierarchy with large amounts of slow-to-access memory and smaller amounts of fast-to-access memory. 
For instance, the hierarchy on a modern NVIDIA 80GB A100 GPU is:  80GB of HBM with 2 TB/s bandwidth, 80MB of L2 cache, 192KB of L1 cache / shared memory (implemented via SRAM) with 19 TB/s bandwidth per SM, and 256 KB of register file per SM \cite{nvidia2022nvidia}. Register memory is private to an executing thread, so threads need to write to shared memory to communicate data to other threads in the block. To reduce the IO cost, a key principle is to \textit{fuse} multiple operations on the same data slice while it's in fast memory before writing it back to slow memory.  

\vspace{-2mm}
\subsection{Taylor Exponential Linear Attention}
\label{sec:eff-tay}
Despite the theoretical efficiency, the popular linear attention implementations are less efficient than well-optimized softmax attention implementations when measured in real-world wall-clock time and memory usage \citep{dao2022flashattention}. We next present hardware-aware algorithms to make Taylor linear attention efficient. We focus on two operations: (1) prefill (this section), corresponding to processing the prompt during generation or the forward pass during training,  and (2) next token prediction during generation (\Cref{app:implementation}), which also requires updating the recurrent hidden state state. 

In this section, we refer to the batch size as $B$, number of heads as $H$, head dimension as $d$, sequence length as $N$ and feature dimension as $d'$, following \Cref{sec:architecture}. For ease of notation, let $D = 1+d'+d'^2$ in this section. Additional details for these algorithms are in \Cref{app:implementation}

\vspace{-2mm}
\paragraph{Baseline Implementation} 
The naïve implementation detailed in Appendix \ref{app:implementation} only uses a CUDA kernel to compute the causal dot product between $\bm{q}$, $\bm{k}$, and $\bm{v}$ projections \citep{vyas_et_al_2020}, but computes the feature maps in python (non IO-aware).
This is inefficient given the computation required for the feature map computation.

\textit{Analysis} In overall IO cost, ignoring the input and output projections in the linear attention layer, this procedure requires $2BHND$ bytes for writing featurized $\bm{q}, \bm{k}$ to HBM. During the causal dot product, this requires $2BHND + BHNd$ bytes to read $\bm{q}, \bm{k}, \bm{v}$ tiles and $BHNd$ bytes to write the result. Throughout the computation, $\mathcal{O}(BHNDd)$ bytes (note this is the shape $KV$ state during the forward pass) are read in and out of thread registers to SRAM to update the running output and $KV$ state at 19TB/s bandwidth.

\paragraph{Algorithm} Our kernel computes \textit{both} the feature map and causal dot product, detailed in \Cref{alg:a0_fwd}. Here we describe the key insights. 
First, to handle \textit{causality} in the dot-product computation, for each \textit{tile} of output $y_i \in \mathbb{R}^{16 \times d}$, we split the computation as shown, where $\bm{q_i}, \bm{k_i}, \bm{v_i}$ are also now tiles of 16 tokens, handled in parallel by the kernel.
$$\bm{y_i} = \mathrm{Causal}(\bm{q_i}^T \bm{k_i})\bm{v_i} + \bm{q_i}\sum_{j=0}^{i-1} (\bm{k_j}\bm{v_j})$$
where the first term uses the quadratic attention view and requires applying causal masking. The second term uses the linear view and its causality has already been handled. 

Second the large KV-state, $\sum_{j=0}^{i-1} (\bm{k_j}\bm{v_j})$, $\in \mathbb{R}^{D \times d}$, needs to be stored as we iterate over the length-$16$ tiles. By partitioning across workers (warps), we can store the state \textit{in thread registers} (fastest memory). The partitioning is restricted by (1) each warp has a limited quantity of threads and (2) warps cannot access the thread memory of other warps.

\textit{Analysis}  In IO cost, again ignoring the input and output projections in the linear attention layer, our procedure requires $2BHNd'$ bytes for reading $q, k$ and $2BHNd$ bytes for reading $v$ and writing output $y$ between HBM and SRAM. Overall, our algorithm avoids in HBM $\mathcal{O}(2BHND)$ bytes in HBM to SRAM data movement. We further improve upon the baseline by storing the KV-state \textit{in-register} to avoid the $\mathcal{O}(BHNDd)$ bytes in SRAM to register data movement.

\noindent End-to-end benchmarks for \sysname implemented with these IO-aware algorithms are provided in \Cref{sec:results}. Micro-benchmarks for each kernel against the baseline implementations are provided in \Cref{app:implementation}.

\vspace{-2mm}
\section{Results}
\label{sec:results}
\begin{table*}[]
\centering
\scriptsize
\begin{tabular}{lccccccccccc}
\toprule
\multirow{3}{*}{Architecture} &
  \multirow{3}{*}{Params/Tokens} &
  \multicolumn{2}{c}{\textbf{Efficiency}} &
  \multicolumn{3}{c}{\textbf{Language Modeling (Pile)}} &
  \multicolumn{2}{c}{\textbf{Info. Extraction}} &
  \multicolumn{1}{c}{\textbf{QA}} & 
  \multicolumn{1}{c}{\textbf{Common}} 
   \\

   &
   &
   Prefill&
   Generate&
  All &
  AR &
  Other &
  SWDE &
  FDA &
  \multicolumn{1}{c}{SQUAD} &
  LM-Evals
   \\

   &
   &
   Tok./ms $\uparrow$&
   Tok./ms $\uparrow$&
  Ppl. $\downarrow$ &
  Ppl. $\downarrow$ &
  Ppl. $\downarrow$ &
  Acc $\uparrow$ &
  Acc $\uparrow$ &
  F1 $\uparrow$ &
  Avg. Acc. $\uparrow$ \\ \hline \hline
Transformer++ & 1.33b/10b       &103.50  &0.99     &\textbf{7.26}     &\textbf{1.74}  &\textbf{8.10}  &\textbf{71.92} &\textbf{73.23} &\textbf{36.19}  &\textbf{47.64}      \\
\sysname     & 1.35b/10b        &\textbf{161.71}  &\underline{24.28}        &\underline{7.43}  &\underline{1.87}  &\underline{8.26}   &\underline{48.06}  &\underline{24.41} &\underline{30.46}            &46.68      \\
Mamba     & 1.32b/10b           &\underline{112.22}    &\textbf{25.69}        &7.48     &1.96  &8.29  &34.74 &12.89 &28.20    &\underline{46.84}      \\ \hline
Transformer++ & 1.33b/50b  & 103.50  & 0.99  & \textbf{6.28}   & \textbf{1.65} & \underline{6.82} & \textbf{76.50} & \textbf{80.47} & \textbf{43.47} &  53.33 \\
\sysname     & 1.35b/50b      &\textbf{161.71}  &\underline{24.28}   &  6.30  & \underline{1.71} & \underline{6.82} & \underline{64.45} & \underline{30.40} & \underline{41.62} &  \textbf{53.81} \\
Mamba     & 1.32b/50b         & \underline{112.22}    &\textbf{25.69}      &   \textbf{6.28}  & 1.74 & \textbf{6.78} & 52.75 & 18.51 & 35.92 & \underline{53.50} \\
\hline
Transformer++ & 360m/10b        &207.77  & 23.82       & \textbf{8.39}    &\textbf{1.87}  &\textbf{9.42}  &\textbf{57.97} &\textbf{58.00} &\textbf{27.18} &\textbf{44.08}       \\
\sysname      & 363m/10b        &\textbf{514.57}      &\textbf{47.23}     & 8.65    &\underline{2.07}   & 9.64   &\underline{29.16} &\underline{11.71} &\underline{25.07}  &43.03        \\
Mamba         & 358m/10b        &\underline{267.09}   &\underline{39.95}  & \underline{8.64}    &2.21  &\underline{9.59}  &{23.67} & \underline{6.53}  &24.06   & \underline{43.51}      \\

GLA           & 362m/10b        &---     &---     & 9.12    &2.36  &10.68 &---      &---      &---          &---      \\
RWKV v5       & 362m/10b        &---     &---     &9.79     &2.40  &10.90 &---     &---      &---           &---      \\ 
H3            & 362m/10b        &---     &---     & 10.60    & 4.88  & 11.23 & 6.75  & 0.64  & 7.87  & 39.35      \\ \hline

Transformer++ & 360m/30b & 103.50  & 0.99 &  \textbf{7.68}  & \textbf{1.80} & \underline{8.40} & \textbf{70.75} & \textbf{63.79} &  25.07 & 44.75 \\
\sysname      & 363m/30b      &\textbf{161.71}  & \underline{24.28}   &  7.77  & \underline{1.93} & 8.46 & \underline{45.01} & \underline{16.45} & \textbf{32.67} & \underline{45.36} \\
Mamba & 358m/30b & \underline{112.22} & \textbf{25.69}  &  \underline{7.73}  & 2.02 & \textbf{8.38} & 27.63 & 8.71 & \underline{26.71} & \textbf{45.62} \\ \hline

\end{tabular}
\caption{\textbf{Evaluation of pre-trained language models.} Models were trained on the same sets of $10$b to $50$b tokens drawn from the Pile corpus~\citep{pile}. 
We report inference throughput on $4,096$ tokens ($16,384$ for 360m param.) of pre-fill and $2,048$ tokens of recurrent generation for a subset of architectures. 
We report language model perplexity on the overall Pile test set as well as perplexity on two slices of the test set: associative recall tokens and other tokens (see \Cref{sec:results-quality}, \citep{arora2023zoology}). 
We report zero-shot performance on three \textit{recall-intensive} tasks: information retrieval on SWDE and FDA as well as question answering on SQUAD. 
Finally, we report average performance on the set of \textit{LM Eval Harness} \citep{eval-harness} common sense reasoning tasks used in \citet{gu2023mamba}, details in \Cref{sec:appendix-downstream-evaluations}. 
These tasks do not require significant recall capacity because the input text is typically very short.  
See \Cref{sec:results-quality}. 
Some proposed architectures that do not implement recurrent views for generation are marked with a ---. 
}

\label{table:main-quality}
\vspace{-3mm}
\end{table*}
\begin{figure*}[ht]
    \vspace{2mm}
    \includegraphics[width=\textwidth]{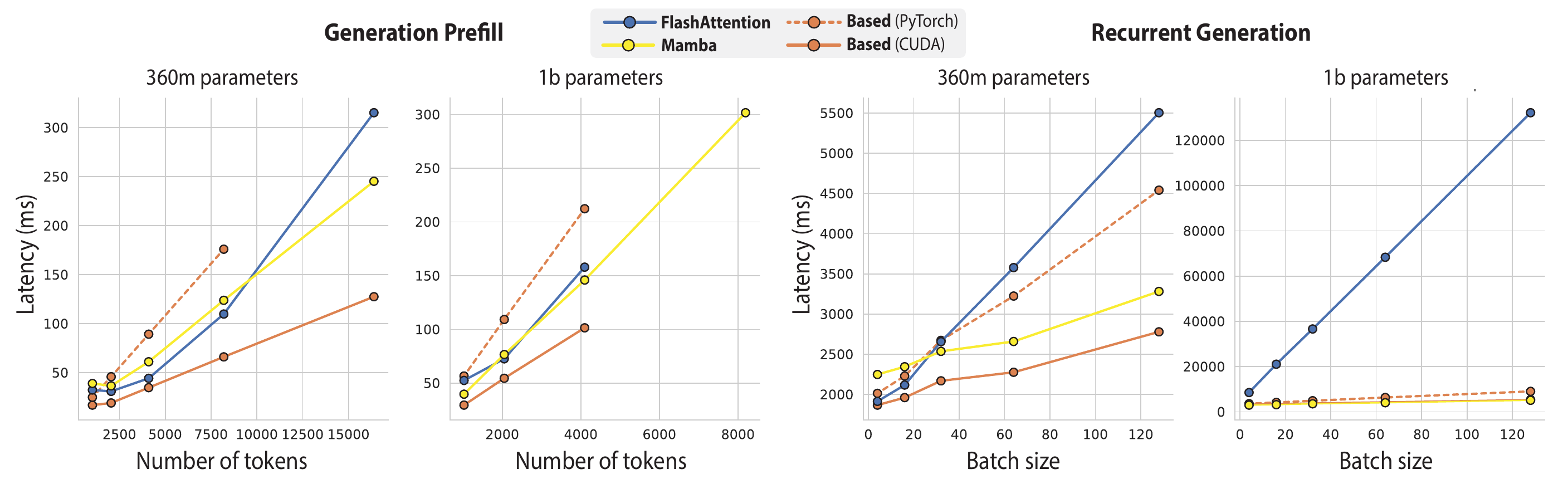}
    \caption{(\textbf{Left}) Throughput numbers for the varied prefill sequence lengths at a fixed batch size of $2$. \textbf{Right} generation throughput at varied batch sizes at a fixed generation length of $1024$ tokens. The $y$-axis shows the in latency (ms). Lines are cutoff when the model runs out of memory. We show results for both 360M and 1.3Bn, and all numbers are computed on a single NVIDIA H100 GPU.}
    \label{fig:forward_throughput}
    \vspace{-4mm}
\end{figure*}

In this section, we present results for the following claims: 
\vspace{-2mm}
\begin{enumerate}[leftmargin=*]
    \item \textbf{Language modeling overall.} We evaluate architectures in pretraining on the Pile \cite{pile} and on standard natural language understanding benchmarks. We find ${\sysname}$ matches or outperforms prior sub-quadratic architectures across these settings. 
    \item \textbf{Language modeling recall.}
    ${\sysname}$ closes the gap to attention on the challenging associative recall slice of the Pile corpus (see \Cref{table:main-quality}). 
    We apply these pretrained models zero-shot to a suite of recall-intensive tasks (\textit{e.g.} information extraction, QA), showing that \sysname systematically outperforms Mamba by large margins (10.36 accuracy points at $1.3$b parameters and $50$b tokens).
    \item \textbf{Generation throughput.} Our IO-aware implementation of recurrent generation in Based enables $40-60\%$ speedups relative to FlashAttention-2 and Mamba for prefill at $4\mathrm{k}$ sequence length and up to $24\times$ higher throughput over FlashAttention-2 in generating $1024$ tokens at batch size $128$ (see \Cref{fig:forward_throughput}).  
\end{enumerate}

\vspace{-4mm}
\paragraph{Baselines} We compare to several key baselines at the $360$m and $1.3$b parameter scales, up to $50$b tokens of training. We first consider Transformer++, Transformers with modern improvements such as rotary encodings \citep{su2023roformer} and gated linear units~\cite{touvron2023llama}. We then consider a class of popular efficient architectures built from gating and long-convolution primitives including Hyena \citep{poli2023hyena}, RWKV \citep{peng2023rwkv}, and H3 \citep{dao2022hungry}. We finally compare to the recently popular Mamba \citep{gu2023mamba} and Gated Linear Attention \citep{yang2023gated} linear recurrent 
architectures with input-dependent recurrent-state updates. We give each architecture the Transformer++ improvements as relevant and use the implementations provided by prior work during training. 

{\sysname} combines familiar local and global sequence mixers to achieve high quality. We train {\sysname} as a hybrid of $\approx 20\%$ linear attention, $\approx 20\%$ sliding window attention, and $\approx 60\%$ gated convolution layers as discussed in \Cref{sec:appendix_experiment_details}. In contrast to recent baselines, {\sysname} requires \textit{no input-dependent decays} whatsoever. 

\vspace{-2mm}
\subsection{Language Modeling Evaluations}
\label{sec:results-quality}
\paragraph{Language Modeling Benchmarks} We pretrain language models from scratch at two parameter scales ($360$m and $1.3$b parameters) on the Pile~\citep{pile}. Each model sees the same tokens of pretraining data in the same order. The Pile data is tokenized using the GPT-2 BPE tokenizer~\citep{radford2019language}. We measure perplexity on the Pile and report results in \Cref{table:main-quality} and further experimental details are provided in \Cref{app:exp-pretraining}. 

We additionally evaluate the pretrained models on key natural language understanding downstream benchmarks using the LM Eval Harness (SuperGLUE, ARC, PIQA, WinoGrande, HellaSwag, LAMBADA). A detailed breakdown of tasks and metrics can be found in \Cref{sec:appendix-downstream-evaluations}. 

In both pretraining and on the downstream tasks, ${\sysname}$ consistently competes with the strongest Transformer++ and Mamba baselines. While these overall metrics are helpful, we next turn to a fine-grained analysis of recall and in-context learning ability on real-world data. 

\vspace{-2mm}
\paragraph{Recall Evaluations} We evaluate our pretrained models on a suite of in-context learning tasks selected to test the downstream recall capacity in \Cref{table:main-quality}. These tasks fall into three categories: (1)  \textbf{Real-world AR} Beyond perplexity scores, we slice the next token predictions on the Pile to understand each architecture's AR quality ( \Cref{sec:appendix_experiment_details}). (2) \textbf{Information extraction (IE)} SWDE and FDA are popular semi-structured and unstructured document IE benchmarks respectively \citep{wu2021medai, deng2022domlm,  arora2023evaporate}. SWDE has HTML for $8$ Movie and $5$ University websites (e.g. IMDB, US News) and annotations for 8-274 attributes per website (e.g., Movie \verb|runtime|), and (3) \textbf{Question answering} from in-context passages. 

We find ${\sysname}$ outperforms the baseline sub-quadratic architectures across these evaluations, closing the gap to Transformer++. These trends track the ${\Task}$ synthetic results from \Cref{sec:motivation_empirical}. We further observe that as we train for longer (more tokens), the improvements from {\sysname} over Mamba grow (from $3.9$ to $9.0$ points on average at $360$m scale and from $9.0$ to $10.4$ points at the $1.3$b scale).

\vspace{-2mm}
\paragraph{Quality Ablations}
In \Cref{table:ablations}, we ablate the feature maps, feature dimensions, and sliding window and convolution dimensions using the same Pile setting as prior experiments. In feature maps, we consider replacing the Taylor approximation with CosFormer \cite{qin2022cosformer} or Performers \cite{choromanski2020rethinking}, and varying the state size using linear projections. We observe with larger sate size, CosFormer closes the gap to the Taylor map though note the projections increase the parameter count. In feature dimension, we find $24$ and $32$ provide diminishing returns. Further discussion is in \Cref{app:exp-architectures}.

\vspace{-2mm}
\subsection{Efficiency Benchmarks}
We benchmark the throughput of ${\sysname}$, with and without our proposed IO-Aware algorithms (Section \ref{sec:efficiency}, Figure \ref{fig:forward_throughput}). We consider both the \textbf{forward pass / generation prefill} and \textbf{next token prediction} stages. Experiments were run using an H100 NVIDIA GPU and averaged over $20$ repetitions. 

\vspace{-2mm}
\paragraph{End-to-end benchmarks} Using our efficient implementation (Section \ref{sec:efficiency}), ${\sysname}$ achieves 56\% faster prefill than FlashAttention-2 \citep{dao2023flashattention2} and 44\% faster than Mamba at $4k$ sequence length and $1.3$b parameters (28\% faster than FlashAttention-2 and 76\% faster than Mamba at $360m$ parameters). We find that next token generation, with \textit{no prefill}, provides $24\times$ higher throughput (tokens/second) over the highly optimized FlashAttention-2 implementation and achieves 95\% and the throughput of the recurrent Mamba architecture at batch size $128$ and $1.3$b parameters (98\% higher throughput vs. FlashAttention-2 and 118\% higher throughput vs. Mamba at $360$m parameters). All benchmarks is on a single NVIDIA H100 GPU, using CUDA cache graphs during next token prediction \citep{nvidia2019graph}.

In \Cref{fig:forward_throughput}, we also include results for the baseline implementation of ${\sysname}$ that uses the popular Fast Transformers CUDA kernel to compute the causal dot product \citep{vyas_et_al_2020} (discussed in \Cref{sec:efficiency}). The \textit{custom kernel} introduced in our work unlocks the efficiency of {$\sysname$}.

\vspace{-2mm}
\paragraph{Micro benchmarks} As the end-to-end ${\sysname}$ architecture is a hybrid architecture, we provide micro benchmarks of the individual kernels against key baseline implementations in \Cref{app:implementation}. Kernels are accessible at: \url{https://github.com/HazyResearch/ThunderKittens}.

\vspace{-2mm}
\section{Conclusion}
This work identifies a fundamental tradeoff between recall, a critical skill for in-context learning, and throughput through theory and experiments.
Attention performs recall perfectly, but requires retaining a KV cache that grows with the sequence length. As an alternative, we propose the ${\sysname}$ architecture, which combines two simple techniques --- local fine-grained attention and long-range linear attention via a Taylor approximation of the softmax exponential function -- that are sub-quadratic during training and permit an efficient recurrent inference view. To enable wall clock efficiency, we introduce IO-aware algorithms for the Taylor linear attention inference that lead ${\sysname}$ to perform generation up to $24\times$ faster than FlashAttention-2 at the $1.3$b parameter scale (generating 1024 tokens, batch size $128$). Beyond competing in overall perplexity, ${\sysname}$ outperforms prior sub-quadratic architectures in recall quality by 10.36 accuracy points on average. 
Overall, our results show that {\sysname} extends the Pareto frontier of the recall-throughput tradeoff space beyond prior architectures.

\section*{Acknowledgments}
We thank Benjamin Spector, Dylan Zinsley, Songlin Yang, Daniel Fu, Jessica Grogan, Eric Nguyen, Michael Wornow, Alyssa Unell, and Gautam Machiraju for their helpful feedback and discussion during this work. We thank the Hazy Research lab and Together AI for supporting this work. 
We gratefully acknowledge the support of NIH under No. U54EB020405 (Mobilize), NSF under Nos. CCF2247015 (Hardware-Aware), CCF1763315 (Beyond Sparsity), CCF1563078 (Volume to Velocity), and 1937301 (RTML); US DEVCOM ARL under Nos. W911NF-23-2-0184 (Long-context) and W911NF-21-2-0251 (Interactive Human-AI Teaming); ONR under Nos. N000142312633 (Deep Signal Processing), N000141712266 (Unifying Weak Supervision), N000142012480 (Non-Euclidean Geometry), and N000142012275 (NEPTUNE); Stanford HAI under No. 247183; NXP, Xilinx, LETI-CEA, Intel, IBM, Microsoft, NEC, Toshiba, TSMC, ARM, Hitachi, BASF, Accenture, Ericsson, Qualcomm, Analog Devices, Google Cloud, Salesforce, Total, the HAI-GCP Cloud Credits for Research program,  the Stanford Data Science Initiative (SDSI), and members of the Stanford DAWN project: Facebook, Google, and VMWare. The U.S. Government is authorized to reproduce and distribute reprints for Governmental purposes notwithstanding any copyright notation thereon. Any opinions, findings, and conclusions or recommendations expressed in this material are those of the authors and do not necessarily reflect the views, policies, or endorsements, either expressed or implied, of NIH, ONR, or the U.S. Government. AR's research is supported by NSF grant CCF\#2247014.

\section*{Impact Statement}
This paper presents work whose goal is to advance the field of Machine Learning. We intend for \sysname to aid in reducing the costs of machine learning and in unlocking new capabilities. There are many potential societal consequences of our work, none which we feel must be specifically highlighted here. Detailed discussions of the risks of using and developing LLMs are in  ~\citet{bommasani2021opportunities, weidinger2021ethical}. 

\bibliographystyle{unsrtnat}
\bibliography{references} \newpage

\newpage
\appendix
\onecolumn

The appendix is organized as follows:
\begin{enumerate}
    \item \Cref{app:extended_related_work} includes an extended related works discussion. 
    \item \Cref{app:implementation} includes details on the IO-aware implementation and benchmarking for \sysname.
    \item \Cref{app:extended-arch} includes additional discussion of \sysname architectural details. 
    \item \Cref{sec:appendix-downstream-evaluations} provides additional experimental results.
    \item \Cref{app:exp} provides experimental details.
    \item \Cref{sec: theory} includes theoretical results and proofs.
\end{enumerate}

To facilitate reproducing this work we release:
\begin{enumerate}
    \item Code for model training and inference at \url{https://github.com/HazyResearch/based}
    \item Model checkpoints at \url{https://huggingface.co/collections/hazyresearch/}
    \item CUDA kernels at \url{https://github.com/HazyResearch/ThunderKittens}
    \item Code for synthetic MQAR experiments at \url{https://github.com/HazyResearch/zoology}
\end{enumerate}

\clearpage

\section{Extended Related Work}
\label{app:extended_related_work}

Our work relates broadly to various developments in efficient sequence modeling. In this section, we organize these related works into (1) model-based or algorithmic contributions (\cref{sec:related-eff-arch}) and (2) implementation or systems-based contributions (\cref{sec:related-eff-sys}).

\subsection{Efficient Language Modeling Architectures}
\label{sec:related-eff-arch}

While Transformers often achieve state-of-the-art language modeling quality, their design motivates various efficiency improvements when both processing input sequences and generating outputs. In particular, various works try to retain their modeling quality, while improving on their quadratic scaling ($\calO(N^2)$ in input sequence length $N$) when processing inputs and $\calO(NM)$ time and space when decoding outputs for outputs of length $M$ (when caching prior keys and values in the attention mechanism).  

We note that most related lines of work build on one of two primitives: \emph{attention approximations} (\textit{e.g.,} linear attentions, sparse attentions, sparse and low-rank attentions), or \emph{state-space models} (SSMs) (which have alternative parameterizations as either ``long'' convolutional models or recurrent neueral networks). Both model classes achieve subquadratic time and space complexity when processing inputs, while linear attentions and SSMs also enable better than $\calO(NM)$ decoding via their ability to process inputs recurrently like a recurrent neural network (RNN). 

We describe each of these model classes next. 

\subsubsection{Efficient Attentions}
\label{sec:related-eff-attn}

We focus on two of the most related paradigms for efficiently computing attention here, \emph{structured sparse attentions} and \emph{linear attentions}. We acknowledge a great deal of prior work to compute attention more efficiently, such as via locality-sensitive hashing~\citep{kitaev2020reformer}, random sparse attentions~\citep{zaheer2020bigbird}, and sequence compression~\citep{wang2020linformer, zhu2021long, alberti2023sumformer}. Please see \cite{tay2022efficient} for a comprehensive survey. 

\paragraph{Structured sparse attentions} Structured sparse attentions reduce attention's time and memory requirements by only attending over specific strided patterns or local \emph{sliding windows}~\citep{parmar2018image,child2019generating,beltagy2020longformer}. For example, \cite{parmar2018image} propose computing attention only over a local window of the past $w$ tokens, such that processing sequences $N$ tokens long only takes $\calO{Nw}$ time and space. \cite{child2019generating} note that this window alone may not all capture all desired dependencies (such as long-term interactions), and propose two strided patterns to compute dot products between queries and keys further away. \cite{beltagy2020longformer} further propose allowing specific tokens to attend to all other tokens in a dense manner. 

While further popularized in recent large language models (Mistral,~\citet{mistral7b}), we note that these implementations use large window sizes that still leave room for improving efficiency. In Based, we introduce a hardware-guided design (using small windows) and sliding window implementation that allows us to capitalize on sparse attention's efficiency.

\paragraph{Linear attentions} Linear attentions preserve the same ``sequence-mixing'' operations as standard attention, computing dot products between queries and keys to weight corresponding values. However, their key insight is to replace the softmax in standard attention with alternative kernel functions~\cite{katharopoulos2020transformers}. Mechanically, by removing the $\exp(\bm{q}^\top \bm{k})$ in favor of feature map dot-products $\phi(\bm{q})^\top\phi(\bm{k})$, these methods use matrix product associativity to compute attention in $\mathcal{O}(Nd^2)$ time and space~\citep{katharopoulos-et-al-2020} (\Cref{eq:linear_attention}). Furthermore, they permit a \emph{recurrent view} for constant memory and $\mathcal{O}(1)$ time per-token generation~\cite{kasai-etal-2021-finetuning,schlag2021linear} (\Cref{eq:linear_attention_recurrent_3}).

Prior works propose different feature maps $\phi$ to improve linear attention modeling quality.  \cite{katharopoulos2020transformers} originally use the \emph{positive elu} function $1 + \text{elu}$ such that  $\phi(\bm{q})^\top\phi(\bm{k})$ remains positive and attention weights remain affine. \cite{qin2022cosformer} instead use the $\text{ReLU}$ function combined with a cosine-based reweighting function to add a locality bias. Other approaches propose feature maps that aim to approximate the Softmax, such as Random Fourier Features \citep{choromanski2020rethinking, choromanski2021hybrid} the Nystrom method~\citep{xiong2021nystromformer, chen2021skyformer}, or deterministic low-degree polynomial approximations~\citep{hedgehog2023,de2015exploration,keles2023on}. Finally, recent works treat the feature map as a learnable function~\citep{kasai-etal-2021-finetuning}, and optionally train the feature map explicitly to recover the softmax kernel~\citep{hedgehog2023}.

\paragraph{Combining sparse and linear attentions}
Finally, our work is closely related to a long line of work on combining sparse and linear attention. Scatterbrain~\citep{chen2021scatterbrain}, building on works such as BigBird \citep{zaheer2020bigbird} and Longformer \citep{beltagy2020longformer}, shows how a sparse and low-rank approximations can be combined into a single unbiased approximation. This approximation is inspired by robust PCA \cite{candes2009robust}. As motivation, they show that any low rank approximation of attention's $\exp(QK^T)$ will have a much larger approximation error than a sparse plus low rank approximation. Note that the Scatterbrain method is largely agnostic to the details of any specific architecture or choice of hyperparameters used in the sparse and low-rank approximations. The focus is on how to combine them so as to maintain an unbiased estimate. In contrast, our work studies how the choice of architecture and hyperparameters affect the model’s efficiency and quality (we’re agnostic to the specific approach for combining the attention). For example, Scatterbrain uses a fixed low-rank approximation (i.e. $\tilde{d} << d$) in experiments. In contrast, we focus on the recall-memory tradeoff and study what happens when we increase the size of $d$. A major takeaway from our study of this tradeoff is that we actually need $d_\text{query} > d_\text{model}$ to match attention’s recall capacity. Our IO-aware implementation shows how to achieve large speedups even when $d_\text{query} > d_\text{model}$.

There are a number of other works which can also be viewed as combinations of sparse and linear attention. 
Multi-resolution analysis attention (MRA-2) uses wavelets to approximate the attention matrix \citep{zeng2022mra}. 
A special form of MRA-2 can be viewed as a combination of sparse and low rank attention for a specific wavelet decomposition. H-transformer-1D uses a hierarchy of matrices including full dense attention on the diagonal and low-rank approximations elsewhere~\citep{zhu2021htrans}. 
TransNormer \citep{qin-etal-2022-devil} apply normalizations such as LayerNorm~\citep{ba2016layer} or RMSNorm~\citep{zhang2019root} to linear attention outputs in certain layers, and apply softmax attention in local chunks in other layers.

\subsubsection{Attention Alternatives}
We now review other attention alternatives, which focus on improving upon the quadratic scaling of attention. Initial work in this vein uses linear time invariant state space models (SSMs) or long convolutions, which can efficiently process sequences of length $N$ in $O(N \log N)$ time invoking the FFT-convolution theorem \citep{cooley1965algorithm}, as the sequence mixer \citep{gu2021efficiently, romero2022ckconv, gupta2022diagonal, gu2022parameterization, mehta2022long, ma2022mega, wang2022pretraining, fu2023simple}. SSMs can also be rewritten as recurrences to permit fast $O(1)$ inference. 

Subsequent work identified that the long convolution alone is not expressive enough to perform particular sub-tasks in language modeling. Prior work shows pure linear SSMs cannot perform associative recall, a skill that is correlated with a model's in-context learning capability \cite{elhage2021mathematical, olsson2022context}, and introduces multiplicative interactions (via gating or Hadamard product \citep{dauphin2017language}) between tokens to allow the model to compare tokens in the sequence \citep{dao2022hungry, poli2023hyena, peng2023rwkv}. However, \citet{arora2023zoology} show empirically and theoretically the class of gated convolution architectures, any architectures built from the two gating and convolution primitives, struggles to learn associative recall (on synthetic and real language data) as efficiently as attention. They show that while attention solves AR in constant many layers / with model dimension that is independent of sequence length, any gated convolution architecture uses dimensionality that scales with the sequence length --- we build upon their upper bound theoretical results with a lower bound argument in \Cref{sec:motivation_theory}. We also study a broader set of architectures in this work beyond gated convolutions. 

\citet{gu2023mamba, arora2023zoology, yang2023gated} identify that the use of \textit{input-dependent} sequence mixers is important for an architecture to perform AR as efficiently as attention. AR requires shifting information that appears prior in a sequence to interact with the current (last) tokens in the sequence, in order to predict the next token \citep{dao2022hungry}. While gating is one way to introduce data-dependence \citep{poli2023hyena}, allowing comparing tokens in two (e.g. a shifted and unshifted) sequences, it is difficult to \textit{select which information} from the prefix of the sequence to shift forwards in the first place, using gating alone. Intuitively, the information to shift \textit{depends on the input's properties}. Thus, several subquadratic architectures consider alternate strategies to introduce input-dependence \citep{katharopoulos-et-al-2020, gu2023mamba, ren2023sparse, ma2022mega, yang2023gated}. We present another strategy for efficient input-dependent sequence mixing in our work.

\subsection{Efficient Implementations}
\label{sec:related-eff-sys}

Beyond designing new model architectures, various works introduce systems-level innovations to improve training and inference efficiency. These include alternative implementations of architecture primitives such as attention~\citep{dao2023flashattention2, liu2023ring, kwon2023efficient}, long convolutions~\citep{fu2023flashfftconv, fu2023simple}, and linear attention~\citep{katharopoulos2020transformers, yang2023gated}. They frequently achieve both reduced memory and increased computational speed on modern GPUs by ``fusing'' operations such as matrix multiplications into a single CUDA kernel, and designing ``IO-aware'' ways to distribute and compute the results of various read and write operations between different levels of GPU memory.

\subsubsection{Efficient Attention Implementations}
\cite{dao2022flashattention} introduce FlashAttention, an alternative yet exact implementation of softmax attention that improves memory and speed by both fusing attention operations into a single CUDA kernel and distributing the attention operations to better exploit High Bandwidth Memory (HBM) and Static Random Access Memory (SRAM). They first  compute attention's query-key-value dot-products, masking, and softmax, together as a single kernel. By doing so after a single load to SRAM before moving the output back to HRAM, they exploit SRAM's fast compute and reduce the total number of read-write operations. To get around SRAM's small memory size and avoid attention's quadratic memory size over input sequence length, they use \emph{tiling} to split up the query, key, and value inputs into smaller ``blocks'', compute the attention operations for each block, and adjust the outputs after computing all blocks to properly normalize the softmax~\citep{rabe2021self,8980322}. To perform backpropagation fast on SRAM, they get around SRAM's limited storage by \emph{recomputing} the gradients rather than storing them. Despite the extra operations, this IO-aware implementation still significantly improves wall-clock time during training. 

Similarly making use of block-wise computation, \cite{liu2023ring} instead compute attention blocks across different \emph{devices} in RingAttention, enabling training and inference over much larger context lengths that scale with device count. They distribute and compute the attention operations in each block across multiple hosts in parallel, likewise keeping track of summary statistics to gather results correctly into exact attention. However, they introduce an ``overlapping'' mechanism to coordinate communication of blocks to reduce overhead. They further make use of Blockwise Parallel Transformers~\citep{liu2023blockwise} to reduce memory, which similar to FlashAttention removes the quadratic in memory scaling of attention by dividing the attention operation into separate blocks before gathering back the adjusted softmax output with block-wise normalization statistics.

As a complement to attention training and inference, \cite{kwon2023efficient} improve attention generation with PagedAttention. PagedAttention similarly uses block-wise computation to address memory utilization issues during generation, where the KV cache can grow an undetermined amount. Existing systems may naïvely handle this by pre-allocating large amounts of contiguous memory. However, this can result in low utilization and computational bottlenecks. Accordingly, PagedAttention divides attention's growing KV cache into \emph{KV blocks} that can be stored separately on physical memory. This enables more flexible memory management, where smaller chunks can be allocated in different locations when needed to reduce memory-based bottlenecks.

In Based, we use similar blocking strategies to more efficiently compute both the second-order Taylor series linear attention and the sliding window softmax attention, and for both training and inference.

\subsubsection{Efficient Attention-Alternative Implementations}
Beyond optimizations for attention, various works also introduce similar ``IO-aware'' implementations to improve memory usage and speed for convolutional and recurrent operations. We overview the most relevant works to Based, which make use of similar techniques such as fusing operations and blocking (tiling) to compute results in SRAM.
\paragraph{Long convolutions}
\cite{fu2023flashfftconv} improve the efficiency of long convolutions on modern GPUs. They build on using the Fast Fourier Transform (FFT), which enables computing convolutions with filter sizes equal to input sequence length from $\mathcal{O}(N^2)$ (if $N$ is filter size and sequence length) to $\mathcal{O}(N \log N)$. However, to compute this algorithm efficiently on GPUs, they break down the convolution into separate matrix multiply operations via a \emph{Monarch} decomposition of the FFT, which allows both (1) fusing multiple steps of the FFT together (for reduced read-write operations) and (2) scheduling these operations for fast computation in SRAM while remaining under the smaller SRAM memory constraints.
\paragraph{Recurrence} \cite{gu2023mamba} improve the efficiency of recent neural state-space models (SSMs)~\citep{gu2021efficiently} using several similar techniques to FlashAttention, specifically with regard the recurrent view. They load the SSM parameters into SRAM for computation before saving results back in HBM, and also use \emph{recomputation} where during backpropagation the intermediate states are not saved but rather recomputed when inputs are loaded from HBM to SRAM. They finally improve wall-clock time by parallelizing the recurrent view of the SSM as a parallel scan. 
\paragraph{Linear Attention} Finally, several works propose techniques to improve the real-world wall-clock time and memory-usage of linear attention. \cite{katharopoulos2020transformers} fuse several operations in the causal dot product of linear attention. \citep{yang2023gated} use blocking to divide the linear attention matrices into SRAM-computable chunks in FlashLinearAttention. As a trade-off between the slow yet memory-efficient RNN view of linear attention and faster but memory-intensive parallel ``standard attention'' view, they further optimize a ``chunk-wise'' implementation of linear attention~\citep{hua2022transformer}. When processing input sequences, the input is first divided into several non-overlapping chunks, where we save memory by computing ``kv states'' at the end of each chunk, and save time by computing the tokens in a given chunk in parallel.

\clearpage

\section{IO Aware Implementations}

\label{sec:app_benchmarking}
\label{app:implementation}

In this section, we provide additional details pertaining to the benchmarking experiments and we provide micro-benchmarking results for the individual kernels we contribute, to complement the end-to-end benchmarking results in the \Cref{sec:results}. Each kernel operates over $16 \times 16$ tiles of data, where dimension $16$ is motivated by the matrix multiplication sizes computed by GPU tensor cores.

\subsection{Forward / Generation Prefill} 

\paragraph{Baselines} In Figure \ref{fig:forward_throughput}, we implement \sysname using our IO-aware Taylor linear attention Algorithm \ref{alg:a0_fwd}. The baseline approach presented in \cite{hedgehog2023}, prior to our kernel, uses the popular linear attention CUDA kernel from Fast Transformers for computing the causal dot product \cite{katharopoulos2020transformers, vyas_et_al_2020}. \footnote{\url{https://github.com/idiap/fast-transformers/blob/master/fast_transformers/attention/causal_linear_attention.py}}. The listing below shows the baseline implementation for reference (where line 76-77 can be computed using pure PyTorch or the Fast Transformers kernel) \cite{hedgehog2023}.

\paragraph{Micro Benchmark} To complement the end-to-end architecture benchmarks in \Cref{sec:results}, we provide micro benchmark results for only the linear attention forward pass in \Cref{fig:la_fwd_micro}.

\begin{figure}[h]
    \centering
    \includegraphics[width=0.47\linewidth]{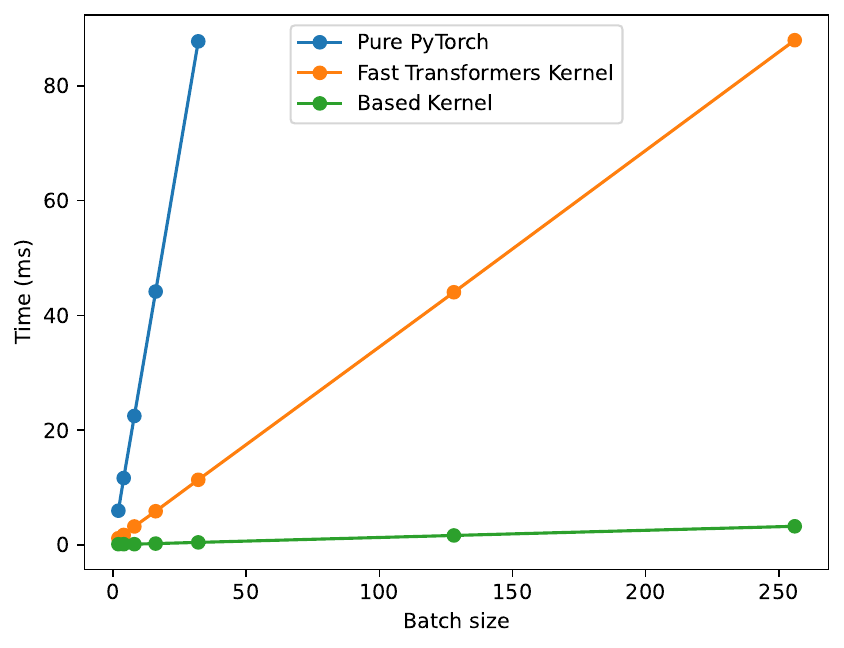}
    \includegraphics[width=0.47\linewidth]{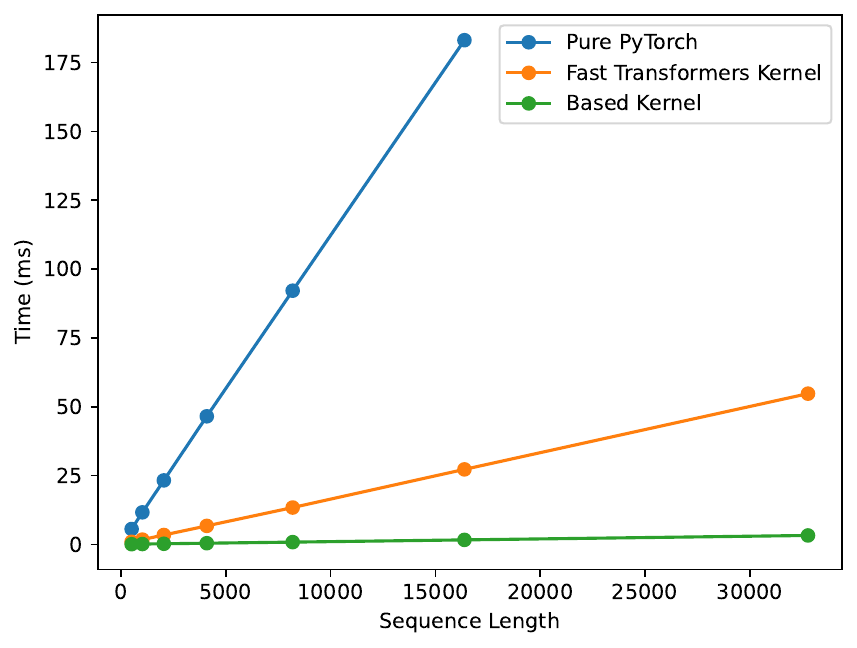}
    \caption{Time (ms) for different ways of computing the Taylor linear attention forward pass --- using Pure PyTorch (shown in the Listing and introduced in \cite{hedgehog2023}), Fast Transformers kernel (as indicated in the listing) \cite{vyas_et_al_2020, katharopoulos-et-al-2020}, or our \sysname  kernel (\Cref{alg:a0_fwd}).  \textbf{(Left)}  Varying the batch size at fixed sequence length $1024$. \textbf{(Right)} Varying the sequence length at fixed batch size $4$. \textbf{(All)} Benchmarking uses $16$ feature dimension, $16$ heads, $64$ head dimension, and focuses on the \textit{numerator} of the linear attention. Each point represents the median across $10$ iterations is measured on a single NVIDIA H100 GPU. Lines terminate on out-of-memory errors.}
    \label{fig:la_fwd_micro}
    \vspace{-2mm}
\end{figure}

\begin{lstlisting}[language=Python,frame=single, linewidth=\textwidth,style=mystyle,caption={PyTorch implementation of Taylor linear attention.}]
from einops import rearrange
import torch
from torch import nn

class TaylorExp(nn.Module):
    """
    Feature map to compute 2nd-order Taylor approx. of exp(q^T k / sqrt(d))
    """

    def __init__(self, input_dim, head_dim_idx, temp=None, eps=1e-12):
        super().__init__()

        self.input_dim = input_dim
        self.head_dim_idx = head_dim_idx
        self.temp = 1.0 if temp is None else temp
        self.eps = eps

        self.r2 = math.sqrt(2)
        self.rd = math.sqrt(self.input_dim)
        self.rrd = math.sqrt(self.rd)

    def forward(self, x: torch.Tensor):
        # Get 2nd-order terms (rearrange(x * x), '... m n -> ... (m n)')
        x2 = (x.unsqueeze(-1) * x.unsqueeze(-2)).flatten(start_dim=-2) / self.r2
        term1 = torch.ones(x[..., :1].shape).to(x.device)
        term2 = x / self.rrd
        term3 = x2 / self.rd
        terms = [term1, term2, term3]
        return torch.cat(t for t in terms), dim=self.head_dim_idx)


class TaylorLinAttn(nn.Module):
    def __init__(self):
        super().__init__()
        self.d_model = d_model
        self.feature_dim = 16
        self.num_heads = 16
        self.num_key_value_heads = 16
        self.head_dim = self.d_model // self.num_key_value_heads
        self.eps = 1e-12

        feature_map_kwargs = {
            "input_dim": self.feature_dim,
            "head_dim_idx": -1,
            "eps": 1e-12,
        }
        self.feature_map = TaylorExp(**feature_map_kwargs)
        self.proj_q = nn.Linear(
            self.d_model, self.feature_dim * self.num_heads, bias=False
        )
        self.proj_k = nn.Linear(
            self.d_model, self.feature_dim * self.num_heads, bias=False
        )
        self.proj_v = nn.Linear(
            self.d_model, self.num_key_value_heads * self.head_dim, bias=False
        )
        self.proj_o = nn.Linear(
            self.num_heads * self.head_dim, self.d_model, bias=False
        )

    def forward(self, hidden_states: torch.Tensor, *args, **kwargs):
        b, l, _ = hidden_states.size()
        q = self.proj_q(hidden_states)
        k = self.proj_k(hidden_states)
        v = self.proj_v(hidden_states)
        q = q.view(b, l, self.num_heads, self.feature_dim).transpose(1, 2)
        k = k.view(b, l, self.num_key_value_heads, self.feature_dim).transpose(1, 2)
        v = v.view(b, l, self.num_key_value_heads, self.head_dim).transpose(1, 2)

        # Linear attention
        q, k = self.feature_map(q), self.feature_map(k)
        q, k, v = q.unsqueeze(-2), k.unsqueeze(-2), v.unsqueeze(-1)

        # Compute attention causal (alternatively use the Fast Transformers kernel)
        num = (q * (k * v).cumsum(dim=2)).sum(dim=-1) 
        denom = (q * k.cumsum(dim=2)).sum(dim=-1) + self.eps
        y = (num / denom)

        y = rearrange(y, "b h l d -> b l (h d)")
        y = self.proj_o(y)
        return y
\end{lstlisting}

\begin{algorithm*}
  \caption{\label{alg:a0_fwd} Computing the $0^{th}$ ($T0$), $1^{st}$ ($T1$), $2^{nd}$ ($T2$) Order Taylor Linear Attention Terms}
  \small
  \begin{algorithmic}
    \Require{Input projected hidden states $q, k, v \in \mathbb{R}^{N \times d}$.}
    \Ensure{Output $y = T0 + T1 + T2 \in \mathbb{R}^{N \times d}$}
    \State \textbf{Parallelize} into $\mathrm{batch} \times \mathrm{heads}$ parallel computations, with $\mathrm{n_{warps}} = 8$ warps per block.
    \State \textbf{Within a block:} 
    \Statex Define tile size $T$ \Comment{$T = 16$ in ${\sysname}$}
    \Statex Define $\mathrm{n_{tiles}} = \frac{N}{T}$ \Comment{Block along the sequence dimension}
    \Statex Define $\mathrm{n_{blocks}} = \mathrm{n_{tiles}} / \mathrm{n_{warps}}$  \Comment{Block along the number of warps}
    \Statex Define $\mathrm{tic}=0$, $\mathrm{toc}=1$ \Comment{Flags for asynchronous data loading} \newline

    \Statex Create SRAM buffers $B_q$, $B_k$ (Size $2 \times \mathrm{n_{warps}} \times  T \times T$) and $B_v$ (Size $2 \times \mathrm{n_{warps}} \times T \times 4T$)
    \Statex Create SRAM buffers $\mathrm{A0}, \mathrm{A1}, \mathrm{A2}$ (Size $\mathrm{n_{warps}} \times T \times 4T$) for storing interim. results for $T0, T1, T2$ as warps process the sequence
    \Statex Create SRAM buffers $total_{A0}$ and $total_{A1}$ to hold cumulative (``KV'') state corresponding to $T0, T1$
    \Statex Create SRAM buffers $y$ of (Size $\mathrm{n_{warps}} \times T \times 4T$) for storing the final output \newline

    \Statex Create register fragments $\mathrm{q_a}, \mathrm{q_b}, \mathrm{k_a}, \mathrm{k_b}, \mathrm{q_{frag}}, \mathrm{k_{frag}}$, $\mathrm{qk_{accum}}$ of size $16 \times 16$. We create register fragments $\mathrm{v_{frag}}, \mathrm{a0_{frag}}$, $\mathrm{a1_{accum}}$, $A2_0$, $A2_1$, $\mathrm{qA2_{accum}}$, $\mathrm{o_{accum}}$ of size $16 \times 64$. These fragments are for holding data during in-register computation. Initialize the fragments to $0$.
    
    \State Each warp loads initial tiles $B_q[\mathrm{tic}][\mathrm{warp id}] \leftarrow Q_{t}, B_k[\mathrm{tic}][\mathrm{warp id}] \leftarrow K_{t}$ and $B_v[\mathrm{tic}][\mathrm{warp id}] \leftarrow V_{t}$ \Comment{HBM into SRAM} \newline
    \For{$\mathrm{cur_{block}} \in [0 .. \mathrm{n_{blocks}-1}]$; $\mathrm{tic}=0 \oplus = 1$, $\mathrm{toc} \oplus 1$} \Comment{XORs $\mathrm{tic}$ and $\mathrm{toc}$ to toggle.}
    \State Warp loads $B_q[\mathrm{toc}][\mathrm{warp id}] \leftarrow Q_{t}$ for $\mathrm{cur_{block}}+1$ \Comment{HBM to SRAM} 
    \State Warp loads $B_k[\mathrm{toc}][\mathrm{warp id}] \leftarrow K_{t}$ for $\mathrm{cur_{block}}+1$
    \State Warp loads $B_v[\mathrm{toc}][\mathrm{warp id}] \leftarrow V_{t}$ for $\mathrm{cur_{block}}+1$ \newline 

    \State 
    Warp loads $\mathrm{q_{frag}} \leftarrow q[\mathrm{tic}][\mathrm{warp id}]$ \Comment{SRAM into register}
    \State  Warp loads $\mathrm{k_{frag}} \leftarrow k[\mathrm{tic}][\mathrm{warp id}]$
    \State  Warp loads $\mathrm{v_{frag}} \leftarrow v[\mathrm{tic}][\mathrm{warp id}]$
    \newline

    \State 
    Compute the warp-local cumulative sum on $\mathrm{v_{frag}} \rightarrow \mathrm{a0_{frag}}$. \Comment{\textcolor{blue}{T0} computation}  
    \State Add the running $A0$ to the current $\mathrm{a0_{frag}}$ 
     \newline

    \State 
    Compute $\mathrm{q_{frag}} \mathrm{k_{frag}}^T$ (attention) and make it causal and store in a $\mathrm{qk_{accum}}$ \Comment{\textcolor{red}{T1} computation}
    \State Compute $\mathrm{qk_{accum}}\mathrm{v_{frag}} \rightarrow \mathrm{o_{accum}}$ \Comment{Store causal $qk^Tv$}
    \State Warps store $\mathrm{k_{frag}^T} \mathrm{v_{frag}} \rightarrow \mathrm{a1_{accum}}$ and write $\mathrm{a1_{accum}} \rightarrow A1[\mathrm{warpid}$] \Comment{Register to SRAM} 
    \State Compute cumulative sum over $A1$ in SRAM, updating $A1$ entries
    \State Warps read $A1$ tiles back to registers \Comment{Each warp now contains its preceeding $A1$}
    \State Warps multiply the values in register with $\mathrm{q_{frag}}$ to update $\rightarrow \mathrm{o_{accum}}$ \Comment{Add in \textcolor{red}{T1} to the running result} 
    \State 
    Update $\mathrm{a0_{frag}} \rightarrow \mathrm{o_{accum}}$
    \Comment {Add in \textcolor{blue}{T0} to the running result} 
    \State 
    Square $\mathrm{qk_{accum}}$, multiply with $\mathrm{v_{frag}}$ and add $\rightarrow \mathrm{o_{accum}}$ 
    \Comment {Add in \textcolor{green}{diagonal T2} to the running result}  
    \State Sum the values of $\mathrm{o_{accum}}$ into $y[\mathrm{warpid}]$ 
    \newline

    \For {$\mathrm{block}$ in $\mathrm{n_{warps}}$ iterations}    \Comment{\textcolor{green}{Remaining T2} computation; Assumes feature dimension $16$} 
    \State  
    Each of $8$ warps copies the \textit{same} slice of $q[\mathrm{tic}][\mathrm{warp id}]$ to $2$ registers $q_a$, $q_b$ 
    \State Each thread $j$ in the warp computes $q_a[:, 2j] q_a$ for dimension $2j$, and for $2j+1$ (and for $q_b$). Together the threads compute the $256$ elements resulting from the second order outer product in the feature map.
    \State 
    Each warp stores two slices of $A2$: $A2_0$ and $A2_1$ \Comment{Partitioning the large $A2$ across warp registers}
    \State Accumulate both $q_a A2_0$ and $q_b A2_1$ $\rightarrow \mathrm{qA2_{accum}}$
    \State 
    Warp writes $\mathrm{qA2_{accum}} \rightarrow A2[\mathrm{warp id}]$ 
    \Comment{Register to SRAM}
    \State Sum results across all in $A2[\mathrm{warp id}]$ and store the sum in $y[\mathrm{block}]$ \Comment{Add in \textcolor{green}{T2}} \newline
    
    \State Each of $8$ warps copies the \textit{same} slice of $k[\mathrm{tic}][\mathrm{block}]$ to $2$ registers $k_a$, $k_b$ 
    \Comment{\textcolor{orange}{KV state} update}
    \State Square $k_a$ and $k_b$
    \State Each of the $8$ warps loads $v[\mathrm{tic}][\mathrm{block}]$ to $\mathrm{v_{frag}}$ in register
    \State Multiply $k_a$ and $\mathrm{v_{frag}}$, $k_b$ and $\mathrm{v_{frag}}$ and accumulate the results into $A2_0$ and $A2_1$, the two in-register slices of $A2$ for the warp, respectively 
    \EndFor{}
    \EndFor{}
    \State \textbf{End.} Store $y$. Optionally store $A0$, $A1$, $A2$ (comprising the ``KV state'') for generation. \Comment{SRAM to HBM}
  \end{algorithmic}
\end{algorithm*}

\paragraph{Algorithm} Here we revisit the key equations we aim to compute and then describe \Cref{alg:a0_fwd} in detail.\newline

\noindent \textit{Objective} First recall from \Cref{sec:architecture}:
\begin{equation}
    \label{eq:linar_attention_reshow}
     \bm{o}_i = \sum_{j = 1}^i\frac{ \phi(\bm{q}_i)^\top  \phi(\bm{k}_j) \bm{v}_j 
    }{\phi(\bm{q}_i) \sum_{j = 1}^i \phi(\bm{k}_j)}
    = \frac{\phi(\bm{q}_i) \sum_{j = 1}^i \big( \phi(\bm{k}_j)^\top \bm{v}_j \big)
    }{\phi(\bm{q}_i) \sum_{j = 1}^i \phi(\bm{k}_j)}
\end{equation}
where $q_i$ reflects the $i^{th}$ of $N$ total tokens in the sequence and every query attends to every past key in $\mathcal{O}(Nd^2)$ time and space complexity for embedding dimension $d$. 

To approximate $\exp(\bm{q}_i^\top \bm{k}_j / \sqrt{d})$, we use the $2^\text{nd}$-order Taylor series feature map, picking $\phi: \mathbb{R}^d \rightarrow \mathbb{R}^{d^2} $ such that
\begin{equation}
    \label{eq:taylor_feature_map_reshow}
    \phi(\bm{q}_i)^\top \phi(\bm{k}_j) = 1 + \bm{q}_i^\top \bm{k}_j + \frac{(\bm{q}_i^\top \bm{k}_j)^2}{2}
\end{equation}
In this section, we will refer to $q_i$ as a \textbf{tile} of data (e.g. of $16$ tokens) instead of as a single token since the hardware operates on chunks of data in parallel. 
\newline

\noindent \textit{Algorithm description} In \Cref{alg:a0_fwd}, we allow each thread block to compute the result for a particular $(\mathrm{batch}, \mathrm{head})$ input. Within the thread block, we use $8$ warps / workers to produce the result. We initialize data structures $B_q, B_k, B_v$ in SRAM and $\mathrm{q_a}, \mathrm{q_b}, \mathrm{k_a}, \mathrm{k_b}, \mathrm{q_{frag}}, \mathrm{k_{frag}}, \mathrm{v_{frag}}$ in register to hold chunks or \textit{tiles} of the $q, k, v$ inputs. We initialize data structures $A0, A1, A2$ in SRAM and $\mathrm{a0_{frag}}$, $\mathrm{a1_{accum}}$, $\mathrm{qA2_{accum}}$ in register to hold computation for the running $KV$ state for the $0^{th}, 1^{st}, 2^{nd}$ order Taylor polynomial terms.

We partition the computation along the sequence dimension into $\mathrm{n_{blocks}}$, where in each loop from $1$ to $\mathrm{n_{blocks}}$, the warps load the next $8$ chunks into fast memory. Note that for $2048$ sequence length and $8$ warps, $16$ tile size, we end up with $\mathrm{n_{tiles}} = 128$ and $\mathrm{n_{blocks}} = 16$. In each iteration, each warp loads in $16 \times 16$ tiles of $q, k$ and $16 \times 64$ tiles of $v$, where $16$ indicates a chunk of $16$ tokens along the sequence dimension and $16, 64$ are the feature and head dimensions respectively. Once tiles are streamed in, we do not need to reuse them, which is key to the efficiency of linear attention.

\textbf{Overall approach} Our overall approach is to compute $\bm{o_i}$ by splitting the $\bm{q}, \bm{k}, \bm{v}$ multiplications as such:
$$\bm{y_i} = \mathrm{Causal}(\bm{q_i}^T \bm{k_i})\bm{v_i} + \bm{q_i}\sum_{j=0}^{i-1} (\bm{k_j}\bm{v_j})$$

where the first term uses the quadratic attention view and requires applying causal masking. Imagining the square attention matrix, we refer to the first term as computing the interactions \textbf{on the diagonal}. The second term uses the linear view and its causality has already been handled. We refer to this term as \textbf{off-diagonal}.


\textbf{Zeroeth order Taylor terms:} During the computation, for the $0^{th}$ term in the Taylor polynomial,  $q, k$ are $1$ after we apply the feature map (\Cref{eq:taylor_feature_map_reshow}). Therefore, computing a cumulative sum over $q(k^Tv)$ reduces to maintaining a cumulative sum of $v$ as we iterate across the sequence. 

\textbf{First order Taylor terms:} Next we consider the $1^{st}$ order terms.
\textbf{On-diagonal:} First consider the on-diagonal blocks, e.g. with respect to tiles $q_i, k_i, v_i$. For these, we simply multiply $q^Tk$, masking (making it causal), and then multiplying with $v$, following the order of operations in standard attention (i.e., a \textit{quadratic attention} view). This makes it easy to apply the masking ($0$ out non-causal elements).
Now each warp contains a local result for its set of on-diagonal tiles of $q_i,k_i,v_i$.

\textbf{Off-diagonal:} However, we need to obtain a \textit{global} cumulative sum where $(q_i^Tk_j)v_j$ depends on all $j \in [1 ..i]$ (\Cref{eq:linar_attention_reshow}). 
Each warp is therefore missing values for tiles $j \in [1 .. i-1]$. To incorporate this computation, we will now compute the cumulative $KV$ hidden state for the warp up until $i-1$ and multiply this with the local tile of $q$ (i.e. $\mathrm{q_{frag}}$). 
To accomplish this, in \Cref{alg:a0_fwd}, we multiply $\mathrm{k_{frag}}^T$ and $\mathrm{v_{frag}}$ to compute local tiles of the hidden state, local to each warp, in thread register. To perform the global cumulative sum across the $8$ warps' local results, we write from registers (thread specific) to $A1$ in SRAM (shared across warp threads). After computing the global cumulative sum in shared memory, each warp loads back the $KV$ state (in $A1$) into its registers such that it contains all the preceeding $KV$ (history) for tiles $[1 .. i-1]$. We then multiply the local $\mathrm{q_{frag}}$ in register with this $KV$ state to update the final output for the $1^{st}$ up until the current $\mathrm{n_{blocks}}$. Note that we maintain the running $KV$ state corresponding to the $1^{st}$ order term in $A1$ shared memory for the next iteration along $\mathrm{n_{blocks}}$. 

\textbf{Second order Taylor terms:}  We finally need to compute the $2^{nd}$ order term. Similar to the $1^{st}$ order term, we'll consider 
\textbf{On-diagonal:}
We can leverage the computation from above. We'll square the causal $(qk^T)^2$ from above and multiply with $\mathrm{v_{frag}}$ to obtain the \textit{portion} of the $2^{nd}$ order term corresponding to the on-diagonal tiles $q_i, k_i, v_i$. 
\textbf{Off-diagonal:} Again, we also need to compute the result with respect to tiles $[1..i-1]$.
\begin{itemize}
    \item \textbf{Partitioning KV hidden state for $2^{nd}$ order} Because the hidden state for the second order term is large ($\mathcal{O}(d^2D)$ in feature dimension $d$ and head dimension $D$) and warps have a limited number of registers, we slice its storage across the registers of the $8$ warps. Considering the the $16^2 \times 64$ ($d^2 \times D$) hidden state (stored in $A2$ SRAM in \Cref{alg:a0_fwd}), we divide this into $16$ slices along the sequence dimension and let each of the $8$ warps handle $2$ of the $16 \times 64$ slices (stored in $A2_0, A2_1$ fragments in thread registers in \Cref{alg:a0_fwd}). Warp $i$ will maintain slices $2i$ and $2i+1$ in two registers per thread. 
    
    \item \textbf{Computing output for $2^{nd}$ order} Each warp $i$ loads in one tile of $q_i$ into $2$ registers. We will use the $32$ threads in the warp to compute the $256$ outer product terms for each token computed by the Taylor $2^{nd}$ order term (for feature dimension $16$).

    Next, the threads multiply these $256$ terms with the running $A2_0$ and $A2_1$ slices. The results for the two slices are summed in register and then stored in SRAM ($A2[\mathrm{warpid}]$). Since $o_i$ is ultimately the sum of $q_i$ terms multiplied with \textit{all} slices of $A2$ (\Cref{eq:linar_attention_reshow}), we then sum the results from all the warps together (which hold the remaining slices of $A2$) and store the result in $y[\mathrm{block}]$. We can think of $y[\mathrm{block}]$ as holding the result up until the $(8\times\mathrm{cur_{block}} + \mathrm{block})$ tile of tokens (note $8$ is because in each increment of $\mathrm{cur_{block}}$, the $8$ warps handle $8$ different tiles of the sequence).

    \item \textbf{Updating the $KV$ state:} For $\mathrm{block}=i$, we load in $k[i], v[i]$ tiles of size $16 \times 16$ and $16 \times 64$ respectively to registers $k_a, k_b, \mathrm{v_{frag}}$. We compute the $256$ outer product terms on $k[i]$ using the $32$ threads, multiply with $\mathrm{v_{frag}}$, and store the result in the $A2_0, A2_1$ running state. 
\end{itemize}
The final result in $y$ is summed into the output to complete the $2^{nd}$ order computation.

\clearpage

\subsection{Next Token Prediction}

During next token prediction in generation, we contribute IO-aware algorithms for the expensive KV-state update in Taylor linear attention and for the sliding window attention computation. 

\subsubsection{Taylor linear attention recurrent update} 
During next token prediction, an important consideration is how to efficiently update the recurrent state $KV_{t} \in \mathbb{R}^{BHDd}$ at timestep $t$. The expensive operation during next token prediction is computing the outer product between projected hidden states $k_{t+1} \in \mathbb{R}^{BHD}$ and $v_{t+1} \in \mathbb{R}^{BHd}$. The outer product requires $\mathcal{O}(BHDd)$ computation and space, and the result is summed with $KV_{t}$ to produce $KV_{t+1}$. We provide an IO-aware algorithm for the state updates in \Cref{alg:lin_gen}. This algorithm incurs $\mathcal{O}(BHD + BHd)$ bytes of HBM to SRAM data movement (to load the $q, k, v$ projections).

The KV update in PyTorch is provided in the following listing. In \Cref{fig:la_inf} we benchmark the speed of the PyTorch implementation against our kernel. 

\begin{figure}[h]
    \centering
    \includegraphics[width=0.6\linewidth]{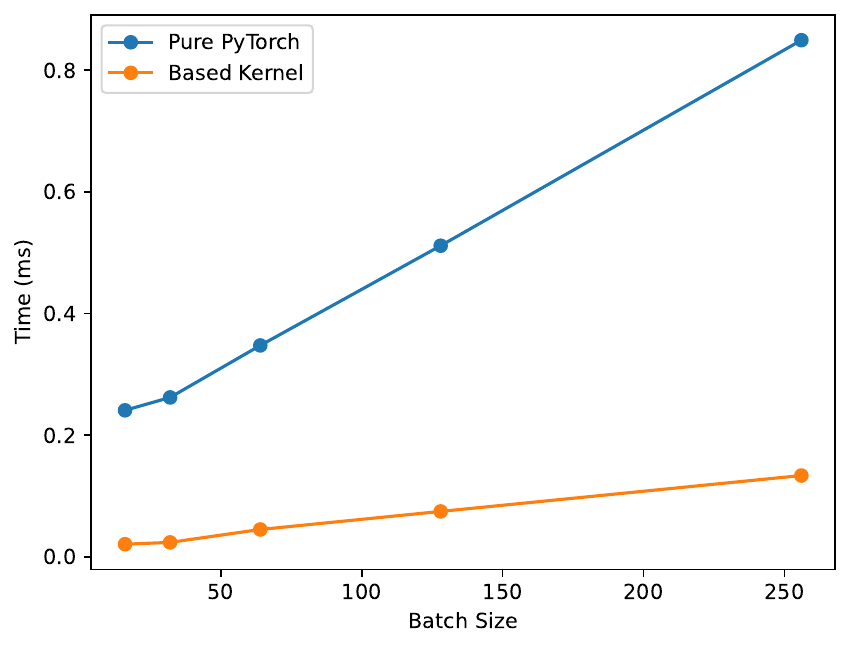}
    \caption{Time (ms) for computing the Taylor linear attention recurrent update using Pure PyTorch (shown in the Listing and introduced in \cite{hedgehog2023}) vs. our \sysname  kernel (\Cref{alg:lin_gen}). Benchmarking uses $16$ feature dimension, $16$ heads, $64$ head dimension, and focuses on the \textit{numerator} of the linear attention. Each point represents the median across $10$ iterations is measured on a single NVIDIA H100 GPU.}
    \label{fig:la_inf}
    \vspace{-2mm}
\end{figure}

\begin{lstlisting}[language=Python,frame=single, linewidth=0.8\textwidth,style=mystyle,caption={PyTorch implementation of Taylor linear attention KV update}]
from einops import rearrange
import torch
from torch import nn

def step(self, kv_state: torch.Tensor, k_state: torch.Tensor, q: torch.Tensor, k: torch.Tensor, v: torch.Tensor):
        """
        Compute linear attention with recurrent view
        -> Assume q.shape is (b, h, 1, D); k and v.shape are (b, h, l, d), where D is the dimension after applying the feature map and d is the head dimension. 
        """
        b, h, l, d = q.shape
        assert l == 1, f'q.shape is {q.shape} but should be ({b}, {h}, 1, {d})'
        # Expand dims for broadcasting to compute linear attention
        q, k, v = q.unsqueeze(-2), k.unsqueeze(-2), v.unsqueeze(-1)

        kv_state += k[:, :, -1:] * v[:, :, -1:]
        k_state  += k[:, :, -1:]

        # Compute linear attention
        num = (q * kv_state).sum(dim=-1)
        y = num / ((q * k_state).sum(dim=-1) + self.eps)

        y = rearrange(y, 'b h l d -> b l (h d)').to(q.dtype)
        return self.dropout(self.out_proj(y))

\end{lstlisting}

\begin{algorithm*}
  \caption{\label{alg:lin_gen} Computing $KV$ State Updates}
  \small
  \begin{algorithmic}
    \Require{$KV_{t-1}$ state $\in \mathbb{R}^{Hd'^2d}$, at time $t$. Featurized $q, k \in \mathbb{R}^{B \times H \times 1 \times D}$ and $V \in \mathbb{R}^{B \times H \times 1 \times d}$, for $d$ as the head dimension (e.g. $64$) and $D$ as the expanded feature map dimension (e.g. $273 = 1+16+16^2$ for feature dim $16$). To be hardware-friendly, we let $D=320$ (s.t. $320 \mod 64 = 0$) via padding.}
    \Ensure{Updated $KV_{t}$ state.}
    
    \State \textbf{Parallelize} 
    into $\mathrm{batch} \times \mathrm{heads}$ parallel computations, with $\mathrm{n_{warps}}=8$ warps per block.

    \State \textbf{Within a block:}
    \State Define $\mathrm{n_{threads}}= \mathrm{n_{warps}} \times 32$ 
    \Comment{Assuming $32$ threads per warp}
    \State Define $\mathrm{buffer_{size}} = \mathrm{n_{warps}} \times 8 \times d$ 
    \State Define $\mathrm{total_{batches}} = \frac{D}{\mathrm{n_{warps}\times 8}}$ \Comment{E.g. $\mathrm{total_{batches}} = 5$ if $D=320$; For $k$, $\frac{320}{5} = 64$ values per batch}
    \State Define $\mathrm{tic} = 0, \mathrm{toc} = 1$
    \newline

    \State Create SRAM buffer $B_q$ (Size $D$) for $q$
    \State Create SRAM buffer $B_k$ (Size $D$) for $k$
    \State Create SRAM buffer $B_v$ (Size $d$) for $V$ 
    \State Create SRAM buffer $B_{kvs}$ (Size $2 \times \mathrm{buffer_{size}} $) for storing blocks of $\mathrm{kv_{state}}$
    \State Create SRAM buffer $\mathrm{o}$ (Size $d$) for output.
    \State Create SRAM buffer $\mathrm{A}$ (Size $\mathrm{n_{warps}} \times d$) for intermediate computation
    \newline

    \State Create register buffer $\mathrm{v_{reg}}$ (Size $2$) to store $V$ data
    \State Create register $\mathrm{A_{reg}}$ (Size $2$) for intermediate computation 
    \newline

    \State Warps load $B_q \leftarrow q$ \Comment{HBM to SRAM; Load all $D=320$ elements of $q$}
    \State Warps load $B_k \leftarrow k$
    \State  Warps load $B_v \leftarrow V$
    \State  Warps load chunk $B_{kvs}[\mathrm{tic}]  \leftarrow \mathrm{kv_{state}}$  \Comment{Load $(1 \times 64) \times 64$ of the $(\mathrm{total_{batches}} \times 64) \times 64$ elements in $KV_{t-1}$}
    \newline

    \State Initialize $m=0$
    \For{Threads $j \in [0 .. 31]$; $j < d$; $j += 32, m += 1$} 
    \Comment{Each thread holds $2$ values ($d = 64$; $32$ threads)}
    \State Load $\mathrm{v_{reg}}[m] \leftarrow {v[j]}$ 
    \Comment{SRAM to Register; Now $v[j]$ is stored in thread $j \mod 32$} 
    \EndFor{}
    \newline

    \For{$i \in [0.. \mathrm{total_{batches}}]$;  $i = i + 1$, $\mathrm{tic} \oplus 1$, $\mathrm{toc} \oplus 1$}
    \State Loads $B_{kvs}[\mathrm{toc}] \leftarrow$ next batch of $\mathrm{kv_{state}}$ 
    \Comment{Asynchronous loads of next batch}
    \newline 
    
    \For{$j = \mathrm{warpid}$; $j < d$; $j += \mathrm{n_{warps}}$} 
    \Comment{Each of the $8$ warps loads $8$ of the $64$ \textit{rows} of $k$, $q$ in the batch}
    \State $\mathrm{k_{val}} \leftarrow B_{k}[i*d + j]$ 
    \Comment{Grab single rows $q[i]$ and $k[i]$, Broadcast to all threads}
    \State $\mathrm{q_{val}} \leftarrow B_{q}[i*d + j]$ 
    \newline
    
    \State $p = B_{kvs}[\mathrm{tic}] + j*d$ 
    \Comment{Point to output rows of $KV_t$; We write $d \times \frac{D}{\mathrm{total_{batches}}}$ sub-matrix for this batch}
    \State Initialize $m=0$
    
    \For{Thread $k \in [0..31]$; $k < d$; $k += 32, m += 1$} 
    \State $p[k] += \mathrm{k_{val}}*\mathrm{v_{reg}}[m]$ 
    \Comment{Update running state by multiplying broadcasted $\mathrm{k_{val}}$ with the full $\mathrm{v_{reg}}$}
    \State \Comment{This updates a $1 \times d$ strip of the $d \times D$ full $KV_t$ outer product}
    \State $\mathrm{A_{reg}}[m] += \mathrm{q_{val}}*p[k]$ 
    \Comment{Multiply $\mathrm{q_{val}}$ with the running state, updating all values in the $1 \times d$ output}
    \EndFor{} 
    
    \EndFor{}\newline
    
    \State Write out new $KV_{t}$ state for this batch: $B_{kvs}[\mathrm{tic}][k]$ 
    \Comment{SRAM to HBM}
    \EndFor{} \newline

    \State Initialize $m=0$
    \For{Threads $j \in [0 .. 31]$; $j < d$; $j += 32, m += 1$} 
    \Comment{Each thread holds info for $2$ of the $64$ output values}
    \State Store $A[\mathrm{warpid}][j] \leftarrow \mathrm{A_{reg}}[m]$
    \Comment{Register to SRAM}
    \EndFor{} \newline
    \For{Thread $j$; $j < d$; $j += \mathrm{n_{threads}}$} 
    \Comment{$d=64$ threads put values from first warp in $n_j$}
    \State $n_j = A[0][j]$ 
    \Comment{Each warp had only computed output values for a subset of (e.g. $8$) rows of $k$ and $q$}
    \For{$w \in [0..\mathrm{n_{warps}}]$}
    \State Sum the $n_j += A[w][j]$ across  \Comment{Need to combine results across warps}
    \EndFor{}
    \State Store $o[j] \leftarrow n_j$
    \EndFor{}
    \State Write output $\mathrm{o}$ \Comment{SRAM to HBM}
  \end{algorithmic}
\end{algorithm*}
\clearpage

\begin{figure}[h]
    \centering
    \includegraphics[width=0.65\linewidth]{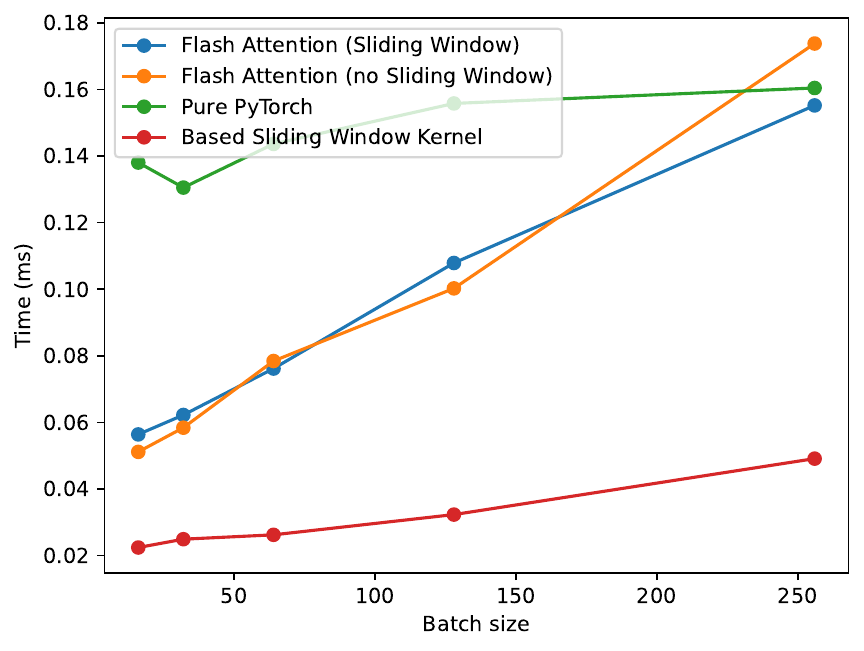}
    \caption{Time (ms) for different ways of computing sliding window attention next token prediction --- using PyTorch, Flash Attention (which supports a sliding window function), or our inference kernel. Each point represents the median across query tokens at different token positions in the generation $\in \{100, 250, 500, 750\}$.}
    \label{fig:sw_micro}
    \vspace{-2mm}
\end{figure}

\subsubsection{Sliding window attention} 

\label{sec:eff-slide}
Next we motivate the choice of window size for ${\window}$. In contrast to sliding-window style models such as the popular Mistral models, which use large window sizes $w=4096$ \citep{mistral7b}, ${\sysname}$ chooses a window size based on hardware specifications. GPU tensor cores operate on $16 \times 16$ tiles. Large GEMMs are compute bound (for e.g. in long-context attention). But, we need sufficient occupancy to hide the latency of the tensor core units. \Cref{fig:based-design} (Right) shows $64 \times 64$ dimension matrix multiplications are approximately the same latency as $16 \times 16$. ${\sysname}$ sets $w$ to use $64 \times 64$ tiles (Figure \ref{fig:based-design}). To distinguish from prior sliding windows, we refer to this approach as ${\window}$. We use the Flash Attention sliding window implementation during training \citep{dao2023flashattention2} and in
 \Cref{app:implementation} \Cref{alg:slide_gen}, we provide an IO-aware algorithm of ${\window}$ for next token prediction. 
 The naïve sliding window approach reads and writes $\mathcal{O}(BHwd)$ bytes between SRAM and HBM between each step of the attention computation. Our approach fuses computation in thread registers to improve upon the baselines.

\paragraph{Baselines} During training / prefill, we use the Flash Attention sliding window implementation \citep{dao2023flashattention2}. 

Our IO-aware implementation focuses on next token prediction. In the listing below, we include a Torch reference. Our IO-aware sliding window attention algorithm is provided in \ref{alg:slide_gen}. The key insight is to fuse operations \textit{in thread registers} to minimize slower SRAM to register data movement.

\paragraph{Micro Benchmark} We benchmark key baselines (Torch, Flash Attention-2 \cite{dao2023flashattention2}, and the {\sysname} kernel on an NVIDIA H100 GPU in \Cref{fig:sw_micro}. The benchmark uses window size $64$, head dimension $64$, and number of heads $16$. We vary the batch size on the $x$ axis and repeat the median timing across iterations on the $y$ axis. Note that these timings include only the attention computation and not the time for updating the KV-cache. These timings also do not include any processing for Rotary encodings (as shown below).

\begin{lstlisting}[language=Python,frame=single,style=mystyle,caption={PyTorch implementation of Sliding Window}]
import torch
from torch import nn

"""
b: batch size
h: number of heads
n: sequence length
d: head dimension

w: window size

qw: b x h x 1 x d
kw: b x h x w x d
vw: b x h x w x d
"""

w = torch.einsum("bhod, bhnd-> bhn",qw, kw)
a = torch.nn.functional.softmax(w, dim=-1)
result = torch.einsum("bhn,bhnd->bhd", a, vw)
\end{lstlisting}
\clearpage
\begin{algorithm*}
  \caption{\label{alg:slide_gen} Sliding window generation}
  \small
  \begin{algorithmic}
    \Require{$KV_{t-1}$ state $\in \mathbb{R}^{Hwd}$, at time $t$ and projected hidden states  $q, k, v \in \mathbb{R}^{B \times H \times 1 \times d}$, for $H$ heads, head dimension $d$, sliding window size $w$, and batch size $B$.}
    \Ensure{Updated $KV_{t}$ state.}
    
    \State \textbf{Parallelize} into $\mathrm{batch} \times \mathrm{heads}$ parallel computations, with $\mathrm{n_{warps}} = 4$ warps per block.

    \State \textbf{Within a block:}
    \State Define tile size $T$ \Comment{$T=16$ in {\sysname}} 
    \State Define $\mathrm{n_{threads}} = \mathrm{n_{warps}} \times 32$ \Comment{Assuming $32$ threads per warp}
    \newline 
    
    \State Create SRAM buffers $B_k$ and $B_v$ (Each of size $4T \times 4T$) to hold $k, v$. 
    \Comment{Assumes $4T=64$ is the $w$, $d$}
    \State Create SRAM vector $B_q$ (Size $1 \times 4T$) to hold $q$ during the kernel execution.
    \Comment{Single query, assume $d=64$}
    \State Create SRAM vector $B_w$ (Size $1 \times 4T$) of type float for intermediate attention computation.
    \State Create SRAM vector $B_o$ (Size $1 \times 4T$) to hold the output.
    \Comment{Single output, assume $d=64$}
    \State Create SRAM buffers $\mathrm{max}$ and $\mathrm{sum}$ (Each of $\mathrm{workers}$ by float size). \newline

    \State Create register fragments $\mathrm{q_{reg}}$, $\mathrm{k_{reg}}$, $\mathrm{v_{reg}}$ to hold data during fused computation in-register. 
    \State Create register fragments $\mathrm{w_{reg}}$ (size $1 \times 4T$) and $\mathrm{{wv}_{reg}}$ (size $4T \times 1$) to store intermediate computation in-register.
    \State Create register fragment $\mathrm{o_{reg}}$ (size $4T \times 1$) to store output in-register.\newline

    \State \textbf{Loads} $B_k \leftarrow k$ using $\mathrm{n_{threads}}$; $B_v \leftarrow v$ using $\mathrm{n_{threads}}$; $B_q \leftarrow q$ using one warp.
    \Comment{HBM to SRAM}

    \State \textbf{Loads} $\mathrm{q_{reg}} \leftarrow B_q$. $q$ gets broadcasted to all warps. 
    \Comment{SRAM to Register}
    \State \textbf{Loads} $\mathrm{k_{reg}} \leftarrow B_k[\mathrm{warpid}]$. Each warp gets $T \times 4T$ of the $4T \times 4T$ in $B_k$ (\textit{i.e.} a column).
    \State \textbf{Loads} $\mathrm{v_{reg}} \leftarrow B_v[\mathrm{warpid}]$. Each warp gets $T \times 4T$ of the $4T \times 4T$ in $B_v$  (\textit{i.e.} a column). \newline

    \State Initialize $\mathrm{w}_{reg}$ to zero
    \State $\mathrm{w_{reg}} \leftarrow \mathrm{q_{reg}} \mathrm{k_{reg}}$ \Comment{Matrix-vector (GEMV) multiplication} \newline

    \State Initialize float $m =- \infty$ for the max \Comment{Obtain the max across tiles for Softmax}
    \State Update $m \leftarrow \max(\mathrm{w_{reg}})$ with the max from the local data
    \State $\textrm{max[warpid]} \leftarrow m$  for all warps to access
    \State Iterate over $\mathrm{n_{warps}}$ entries in $\textrm{max}$ buffer to compute the global max of $\mathrm{w_{reg}}$
    \State Put global max back into each warp's $m$ float \newline

    \State Initialize float $s=0$ for the sum \Comment{Obtain the sum across tiles for Softmax}
    \State Update $s \leftarrow \mathrm{sum}(\mathrm{w_{reg}})$ with the sum from the local data
    \State $\textrm{sum}[\mathrm{warpid}] \leftarrow s$  for all warps to access
    \State Iterate over $\mathrm{n_{warps}}$ entries in $\textrm{sum}$ buffer to compute the global sum of $\mathrm{w_{reg}}$
    \State Put global sum back into each warp's $s$ float \newline

    \State $\mathrm{w_{reg}} \leftarrow \mathrm{w_{reg}} - m$ \Comment{Start attention computation in register}
    \State $\mathrm{w_{reg}} \leftarrow \exp(\mathrm{w_{reg}})$
    \State $\mathrm{w_{reg}} \leftarrow \frac{\mathrm{w_{reg}}}{s}$ 
    \State $B_w[\mathrm{warpid}] \leftarrow \mathrm{w_{reg}}$ \Comment{Register to SRAM; storing for the slice of $k$} 
    \State $\mathrm{wv_{reg}} \leftarrow B_w$ \Comment{SRAM to Register. Warp loads entirety of $B_w$; all slices}

    \State Initialize $\mathrm{o_{reg}}$ to zero.
    \State $\mathrm{o_{reg}} \leftarrow \mathrm{wv_{reg}} \mathrm{v_{reg}}$ \Comment{Matrix-vector (GEMV) multiplication}   \newline
    \State Write $\mathrm{o_{reg}}$ to global memory \Comment{Register to SRAM, SRAM to HBM}
  \end{algorithmic}
\end{algorithm*}

\clearpage

\section{Extended Architecture Details}
\label{app:extended-arch}
In this section, we describe two additional architectural details for {\sysname}
that can enable small improvements in language model perplexity. 
We emphasize, however, that the combination of Taylor linear attention and \window layers alone is sufficient to come within $0.1$ perplexity points of our best models using these additional components (\Cref{table:ablations}).  

\paragraph{Convolution.} We find that replacing some of the linear attention and \window layers with gated convolution layers enables small improvements in language modeling performance. A gated convolution layer uses a combination of gating (Hadamard product, elementwise product) and convolution operations. In {\sysname}, we use \BaseConv\ layers \citep{arora2023zoology} with short convolutions and a SilU non-linearity~\citep{hendrycks2023gaussian}. 
By keeping the convolutions short (\textit{e.g.} width 3), we keep the recurrent state size for these layers low and improve throughput. The projections expand the dimensionality by a factor $c=4$. 
\begin{equation}
        \label{eq: coyote-recursion}
        \begin{aligned}
            \bm{y}
            &:= 
            (\underbrace{({\bm{u} \cdot \bm{W}_1+\bm{b}_1})}_{\mathclap{\textbf{Linear Projection}}}
            \odot 
            \sigma\underbrace{({\bm{h} \ast \bm{u} \cdot \bm{W}_2 +\bm{b}_2})}_{\mathclap{\textbf{Convolution}}}) \cdot \bm{W}_3 + \bm{b}_3
            \quad 
        \end{aligned}
    \end{equation}
where $\bm{u} \in \R^{N \times d}$ is a projected input, $h\in\R^{N\times c d}$ is a learned filter, $\odot$ is the Hadamard product, and $\bm{W}_1, \bm{W}_2 \in\R^{d\times c d}$, $\bm{W}_3 \in\R^{c d\times d}$, $\bm{b}_1,\bm{b}_2 \in \R^{c d}$, and $\bm{b}_3,  \in \R^{d}$ define weights and biases of three linear projections.

\paragraph{Decay.} Recent recurrent architectures include the use of \textit{decay} terms, implemented in a variety of ways \citep{gu2021efficiently, sun2023retentive, gu2023mamba, yang2023gated}. As intuition, decay terms control how much a token should attend to ``recent'' tokens  vs. ``early" tokens in the sequence. Prior work falls in two categories: using input-independent \citep[inter alia.]{gu2021efficiently, sun2023retentive} or input-dependent \citep{gu2023mamba, yang2023gated} decay rates. The latter offers improved quality, but requires the use of a parallel scan during sequence processing \cite{gu2023mamba}. 

Instead, we explore a coarser input-dependent decay technique for the linear attention layer, avoiding the parallel scan. We first use a unique decay rate \textit{per head}, fixed across all inputs. We introduce a linear projection that takes in the inputs $\in \mathbb{R}^{N x d}$ and projects to $\mathbb{R}^{N \times h}$, where $N$ is the sequence length, $d$ is the model dimension, and $H$ is the number of heads. We use the result of this projection to scale the attention combination across heads. 

In our main experiments \Cref{table:main-quality}, we use no decay when training the models to $50$b and $30$b tokens. We observe that decay can help small in our \Cref{table:ablations} ablations, but removing the decay does not affect the overall trends for {\sysname} relative to other architectures.


\clearpage

\section{Extended Results}
\label{sec:appendix-downstream-evaluations}

\subsection{Extended empirical study of memory-recall tradeoff}
In \Cref{fig:extended-tradeoff}, we provide additional experimental results using the setup described in \Cref{sec:motivation_empirical}. The results in \Cref{fig:extended-tradeoff} include additional efficient architectures beyond those in \Cref{fig:feature-map-tradeoff} and \Cref{fig:tradeoff}. Specifically we include NystromFormer~\cite{xiong2021nystromformer}, BigBird~\cite{zaheer2020bigbird}, and ScatterBrain~\cite{chen2021scatterbrain}.

\begin{figure}[h]
    \begin{center}
    \includegraphics[width=0.6\linewidth]{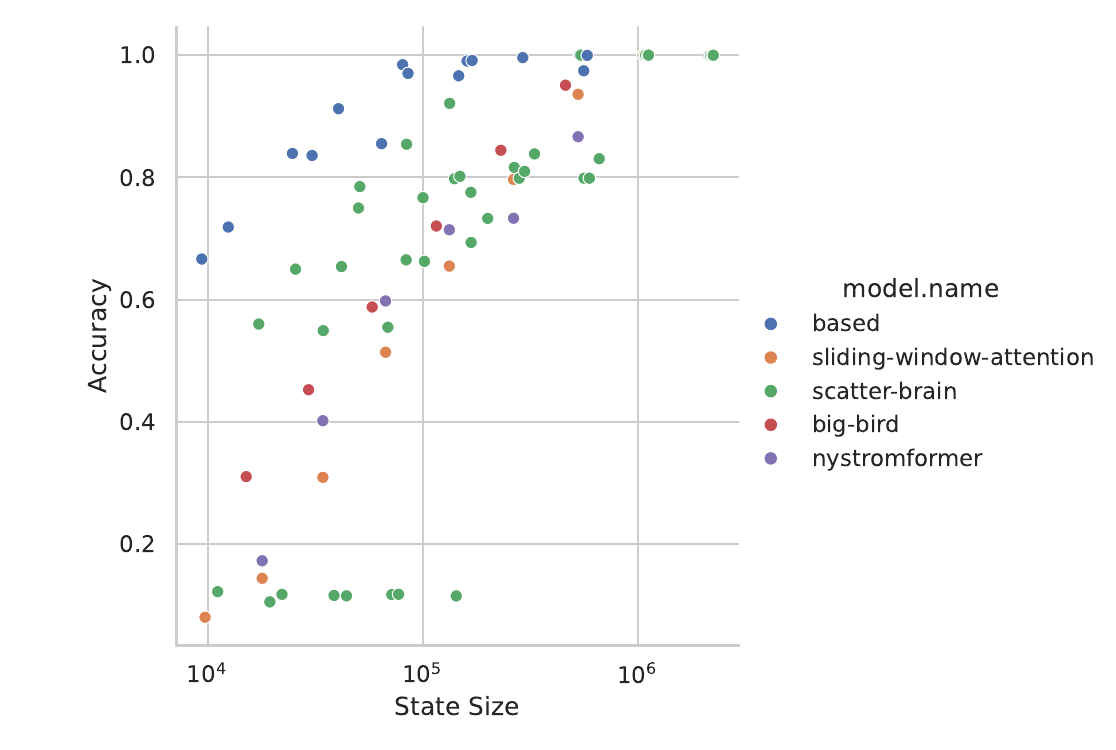}
    \vspace{-2em}
    \caption{
        \textbf{Extended Throughput (memory) - recall tradeoff.} $x$-axis shows state size (bytes) during generation; $y$-axis shows accuracy on the \syntheticname recall task \citep{arora2023zoology}.
        For each architecture, we train several models varying hyperparameters that affect the recurrent state size (\textit{e.g.} model dimension). 
        The plot shows a fundamental tradeoff between the recurrent state size and recall capacity that applies to broad class of models.
    }
    \label{fig:extended-tradeoff}
    \end{center}
\end{figure}

\subsection{Downstream Language Results}
To further evaluate \sysname's performance in language modeling, we evaluate the PILE-pretrained models on several downstream tasks that test general natural language understanding.

\begin{table*}[h]
\centering
\scriptsize 
\begin{tabular}{lccccccccc}
\toprule
\multirow{3}{*}{Architecture} &
  \multirow{3}{*}{Params/Tokens} & \\
   &
   &
  \multicolumn{2}{l}{LAMBADA} &
  HellaSwag &
  PIQA &
  Arc-E &
  Arc-C &
  WinoGrande &
  Average
   \\

   &
   &
  Ppl. $\downarrow$ &
  Acc. $\uparrow$ &
  Acc. Norm. $\uparrow$ &
  Acc $\uparrow$ &
  Acc $\uparrow$ &
  Acc. Norm. $\uparrow$ &
  Acc. $\uparrow$ &
  Acc. $\uparrow$ \\ \hline \hline
Transformer++ (LLaMa) & 1.33b/10b     & 11.12    &49.10  &39.29  &66.16 &51.68 &26.19  &53.43  &47.64      \\
\sysname     & 1.35b/10b               & 12.35    &46.96  &39.11  &66.32 &50.72 &26.54  &50.43  &46.68      \\
Mamba     & 1.32b/10b               & 13.11    &46.13  &39.41  &66.38 &52.36 &25.94  &50.83  &46.84      \\ \hline
Transformer++ (LLaMa) & 1.33b/50b  
& 7.38 &  57.50 & 49.62  & 70.46 & 57.58 & 27.99 & 56.83 & 53.33 \\
\sysname     & 1.35b/50b  
& 6.96 &  57.85 & 50.79 &  71.65 & 58.84 & 28.75 & 55.80 & 53.81 \\
Mamba     & 1.32b/50b 
& 7.19 & 57.56 & 50.94 & 71.87 & 59.39 & 28.41 & 53.83 & 53.50 \\ \hline

Transformer++ (LLaMa) & 360m/10b & 18.39    &42.52  &33.48  &63.98 &46.04 &24.49  &53.99  &44.08       \\
Transformer (Pythia)& 356m/10b & 25.17    &37.16  &31.32  &63.76 &44.82 &23.8   &51.54  &42.08       \\
\sysname      & 363m/10b   & 21.80    &38.66  &33.43  &64.42 &45.79 &24.66  &51.22  &43.03       \\
Mamba         & 358m/10b   & 20.23    &39.65  &33.63  &65.02 &47.01 &25.00  &50.75  &43.51      \\
H3            & 362m/10b   & 57.59    &23.58  &30.62  &63.11 &45.20 &23.29  &50.28  &39.35      \\ \hline
Transformer++ (LLaMa) & 360m/30b  
& 15.79 & 44.44 & 36.90 & 66.05 & 48.27 & 20.56 & 52.25 & 44.75 \\
\sysname     & 363m/30b  
& 14.43 & 45.20 & 37.41 & 67.46 & 49.45 & 21.42 & 51.22 & 45.36 \\
Mamba     & 358m/30b 
& 14.27 & 45.06 & 38.02  & 66.38 & 50.55 & 20.01 & 51.70 & 45.62 \\ \hline
\hline
\end{tabular}
\caption{\textbf{Downstream evaluation of pre-trained language models.} The same set of models as in \cref{table:main-quality}, all were trained on the same data drawn from the Pile~\citep{pile}, evaluated zero-shot using the default LM-Eval Harness settings from EleutherAI \cite{eval-harness}. These averages are computed across the $6$ tasks, excluding LAMBADA perplexity. These averages are included in \Cref{table:main-quality}.}
\label{table:lm-eval-harness}
\vspace{-3mm}
\end{table*}
\begin{table*}[]
\centering
\scriptsize 
\begin{tabular}{lcccccccccccc}
\toprule
\textbf{Model}                                                               & \multicolumn{1}{l}{\textbf{Shots}} & \multicolumn{1}{l}{\textbf{BoolQ}}  & \multicolumn{2}{c}{\textbf{CB}}                                         & \multicolumn{1}{l}{\textbf{COPA}}   & \multicolumn{1}{l}{\textbf{MultiRC}} & \multicolumn{2}{l}{\textbf{ReCoRD}}                                   & \multicolumn{1}{l}{\textbf{RTE}}    & \multicolumn{1}{l}{\textbf{WiC}}    & \multicolumn{1}{l}{\textbf{WSC}}    & \multicolumn{1}{l}{\textbf{Avg}} \\
                                                                                    & \multicolumn{1}{l}{}               & \multicolumn{1}{l}{Acc. $\uparrow$} & \multicolumn{1}{l}{Acc. $\uparrow$} & \multicolumn{1}{l}{F1 $\uparrow$} & \multicolumn{1}{l}{Acc. $\uparrow$} & \multicolumn{1}{l}{Acc. $\uparrow$}  & \multicolumn{1}{l}{F1 $\uparrow$} & \multicolumn{1}{l}{EM $\uparrow$} & \multicolumn{1}{l}{Acc. $\uparrow$} & \multicolumn{1}{l}{Acc. $\uparrow$} & \multicolumn{1}{l}{Acc. $\uparrow$} & \multicolumn{1}{l}{}                 \\ \hline\hline
\multirow{3}{*}{\begin{tabular}[c]{@{}l@{}}Based\\ (363m/10b)\end{tabular}}               & 0                                  & 59.0                                & 41.1                                & 19.4                              & 69.0                                & 54.9                                 & 14.5                              & 14.0                              & 52.0                                & 50.0                                & 36.5                                & 45.7                                 \\
                                                                                    & 1                                  & 57.5                                & 37.5                                & 26.8                              & 68.0                                & 52.5                                 & 19.9                              & 19.2                              & 47.7                                & 50.9                                & 49.0                                & 47.2                                 \\
                                                                                    & 5                                  & 56.6                                & 44.6                                & 28.9                              & 73.0                                & 53.6                                 & 24.9                              & 24.1                              & 48.7                                & 51.1                                & 39.4                                & 48.0                                 \\ \hline
\multirow{3}{*}{\begin{tabular}[c]{@{}l@{}}Transformer++ \\ (360m/10b)\end{tabular}} & 0                                  & 57.3                                & 41.1                                & 21.3                              & 67.0                                & 57.0                                 & 16.6                              & 16.1                              & 53.8                                & 50.0                                & 37.5                                & 46.3                                 \\
                                                                                    & 1                                  & 54.2                                & 39.3                                & 25.3                              & 69.0                                & 51.5                                 & 22.2                              & 21.6                              & 50.9                                & 47.0                                & 55.8                                & 47.8                                 \\
                                                                                    & 5                                  & 50.7                                & 58.9                                & 49.9                              & 64.0                                & 46.9                                 & 24.2                              & 23.6                              & 47.3                                & 52.2                                & 51.9                                & 48.9                                 \\ \hline
\multirow{3}{*}{\begin{tabular}[c]{@{}l@{}}Mamba\\ (358m/10b)\end{tabular}}               & 0                                  & 57.5                                & 35.7                                & 24.4                              & 71.0                                & 57.2                                 & 18.8                              & 18.3                              & 52.4                                & 50.0                                & 36.5                                & 46.6                                 \\
                                                                                    & 1                                  & 51.1                                & 39.3                                & 27.4                              & 71.0                                & 52.9                                 & 21.6                              & 21.0                              & 46.6                                & 46.2                                & 52.9                                & 46.9                                 \\
                                                                                    & 5                                  & 41.1                                & 37.5                                & 23.6                              & 69.0                                & 49.2                                 & 20.4                              & 19.9                              & 48.4                                & 51.7                                & 51.9                                & 45.2                          
                                                                                   \\
                                                                                   \bottomrule
\end{tabular}
\caption{\textbf{Few-shot downstream evaluation on SuperGLUE of pre-trained language models.} The same set of models as in \cref{table:main-quality}, all were trained on the same 10 billion tokens drawn from the Pile~\citep{pile}, evaluated on the SuperGLUE benchmark \cite{wang2019superglue} using the LM eval harness by EleutherAI \cite{eval-harness}. When computing the average, we first average the metrics by task and then average across tasks.}
\label{tab:super-glue}
\end{table*}

\paragraph{LM-Eval Harness Standard Tasks} We use the same protocol as \citep{gu2023mamba, yang2023gated}, utilizing the LM evaluation harness by EleutherAI \citep{eval-harness}. In particular, we use the following set of metrics and tasks: 
\begin{itemize}
    \item LAMBADA (perplexity and accuracy) \cite{paperno2016lambada}
    \item HellaSwag (normalized accuracy) \citep{zellers2019hellaswag}
    \item PIQA (accuracy) \citep{bisk2019piqa}
    \item ARC-challenge (normalized accuracy) and, separately, the easy subset ARC-easy (accuracy) \citep{clark2018think}
    \item WinoGrande (accuracy) \citep{sakaguchi2019winogrande}
\end{itemize} 
Normalized accuracy refers to accuracy normalized by sequence length and is used to maintain the equivalent setting to \citep{gu2023mamba}. We report results in Table \ref{table:lm-eval-harness}. For both 360 million and 1.3 billion parameter models, \sysname performs competitively with recent and state-of-the art architectures, including Mamba and Transformer++ (LLaMa).

\paragraph{SuperGLUE Fewshot Results} In order to understand in-context-learning performance, we next perform few-shot  evaluations on the SuperGLUE benchmark \cite{wang2019superglue} for \sysname, Mamba and Transformer++ in \Cref{tab:super-glue}. Each model was evaluated on all tasks using under 0 shot (\textit{i.e.}, number of in-context examples), 1 shot and 5 shot prompting, respectively. Transformer++ and \sysname both see monotonic improvement from increasing the number of shots. For Mamba, however, albeit getting a slight improvement from 0-shot to 1-shot, it performs worse on 5-shot than even on 0-shot. This result suggests that the limited recall ability observed in Mamba could also impact few-shot abilities.

\subsection{DNA Modeling}
Towards understanding the capability of ${\sysname}$ beyond natural English language, we next evaluate each architecture on its ability to model DNA sequences. 

\paragraph{Pretraining} In Table \ref{table:dna-quality},  we evaluate architectures on the HG38 (human genome) benchmark at $1k$, $4k$, and $8k$ sequence lengths used in prior architecture evaluations \citep{nguyen2023hyenadna, gu2023mamba}. The DNA tasks uses a byte-level tokenizer wherein the vocabulary consists of characters corresponding to the nucleotide bases. 
We find ${\sysname}$ is competitive with state-of-the-art architectures across evaluated sequence lengths.
\begin{table}[]
\centering
\scriptsize
\begin{tabular}{lcccc}
\toprule
\multirow{3}{*}{Model} &
  \multirow{3}{*}{Params} &
  \multicolumn{3}{c}{\textbf{HG38  PPL $\downarrow$}} \\
              &        & $N=$1024 &  $N=$4096 & $N=$8192  \\ \hline  \hline 
Transformer++ & 46.2   & 2.52 &  2.50 & 2.51  \\
Mamba         & 46.1   & \textbf{2.51} &  \textbf{2.49} & \textbf{2.49}  \\
Based         & 48.8   & \textbf{2.51} &  2.50 & \textbf{2.49}  \\
\hline
\end{tabular}
\caption{\textbf{DNA modeling performance on the HG38 dataset.} All models are pretrained from scratch for 10Bn tokens at $N=$ 1k, 4k, and 8k sequence lengths respectively. We report results after hyperparameter sweeping the learning rate for each architecture.}
\label{table:dna-quality}
\vspace{-3mm}
\end{table}

\paragraph{Downstream DNA Classification} We further evaluate how different architectures compare for DNA modeling. We take the pretrained models described and evaluate them on DNA sequence classification using a popular benchmark (GenomicBenchmarks)~\citep{Gresova2022.06.08.495248} in Table~\ref{table:genomics_benchmarks}. We find similar performance across tasks, indicating that prior matching in quality during pretraining transfers to downstream classification. For reference, we also include results from \cite{nguyen2023hyenadna}. Although not directly comparable to due differences in tokenization, the evaluations suggest \sysname can perform strongly on different modalities, and that recent sequence modeling architectures are also able to outperform or compete with prior state-of-the-art on evaluated DNA tasks.
\begin{table*}
\centering
\scriptsize 
\begin{tabular}{@{}lccccc@{}}
\toprule
Dataset       & Enhancer Cohn & Enhancer Ens & Human Reg. & Non-TATA Promoters & Human OCR Ens. \\ \midrule
CNN           & 69.5                 & 68.9                     & 93.3             & 84.6                     & 68.0               \\
DNABERT       & 74.0                 & 85.7                     & 88.1             & 85.6                     & 75.1               \\
GPT           & 70.5                 & 83.5                     & 91.5             & 87.7                     & 73.0               \\
HyenaDNA      & 74.2                 & 89.2                     & \textbf{93.8}    & 96.6                     & \textbf{80.9}      \\ \midrule
Transformer++ & 73.4                 & \textbf{89.5}            & 89.9             & 94.4                     & 79.5               \\
Mamba         & 73.0                 & -                        & -                & 96.6                     & -                  \\
Based         & \textbf{74.6}        & \textbf{89.5}            & 89.5             & \textbf{96.8}            & 79.0               \\ \bottomrule
\end{tabular}
\caption{\textbf{Downstream evaluation of pre-trained DNA models on GenomicsBenchmarks~\citep{Gresova2022.06.08.495248}.} We report top-1 classification accuracy ($\%$) with pretrained models (Transformer++, Mamba, \sysname) along with prior reported results in \cite{nguyen2023hyenadna}. We find the similar quality-matching in pretraining transfers to downstream tasks. Modern architectures are also able to achieve state-of-the-art results on the classification tasks.}
\label{table:genomics_benchmarks}
\vspace{-3mm}
\end{table*}

\subsection{{\sysname} Quality Ablations}
Our objective with \sysname is to measure the throughput and recall of the simplest possible linear attention model that achieves strong performance. Therefore, we ablate the key design decisions --- choice of feature map, feature dimension for the Taylor map, use of sliding window and convolutions --- to understand their contributions to the quality of \sysname. We ablate using the Pile dataset \cite{pile} with the same number of tokens and data ordering as the prior experiments.

In \textbf{feature map ablations}, we consider the CosFormer \cite{qin2022cosformer} and Performers \cite{choromanski2020rethinking} feature maps, which have been demonstrated as strong choices in prior work \cite{hedgehog2023}. We also include a baseline that expands the state size using learned projections and applies CosFormer towards comparing to the larger state size of the Taylor map. For these baselines, we keep the rest of the {\sysname} architecture the same (\textit{i.e.} in the number of linear attention layers and hybridization with sliding window and gated convolution layers). We observe that with the larger state size, CosFormer quality is increasingly competitive with the Taylor map. We note that expanding the state size requires increasing the model's overall parameter count (due to the learned projections) for CosFormer, in contrast to the Taylor map.

Next, we ablate the \textbf{feature dimension}, holding the feature map fixed to the Taylor map. We find larger feature dimension improves quality, with diminishing returns going from $24$ to $32$ dimension. Note that feature dimension $\sqrt{1024} = 32$, where $1024$ is the attention model dimension at the $360$ parameter scale in our experiments.

Next, the ablations show that eliminating the \textbf{convolutions} and/or the \textbf{sliding window attention} degrades quality. We observe that adding \textit{either} convolutions or sliding window helps on the associative recall slice relative to \textit{neither} (e.g. $2.29$ AR Ppl. on the Pile with \textit{neither} vs. $2.09$ or $2.11$ with sliding window \textit{or} convolutions.). Increasing the window size from $0$ to $64$ vs. $64$ to $128$ (also an efficient design point in \Cref{fig:based-design}, left) continues to help quality, but with marginal improvements. 

Finally, we ablate the use of the input-dependent decay strategy introduced in \Cref{app:extended-arch}. In our main results \Cref{table:main-quality}, we use \textbf{no input-dependent decay} whatsoever when training to $30$b and $50$b tokens for the $360$m and $1.3$b parameter models respectively. At $10$b tokens, we use the decay strategy and provide ablations without decay in \Cref{table:ablations}. We find that the decay can provide a small boost in quality, but removing the decay does not affect the overall trends.

\begin{table*}[]
\centering
\scriptsize
\begin{tabular}{lcccccccccc}
\toprule
  \multicolumn{5}{c}{\textbf{Hyperparameters}} &
  \multicolumn{3}{c}{\textbf{Language Modeling (Pile)}} &
  \multicolumn{2}{c}{\textbf{Info. Extraction}} &
  \multicolumn{1}{c}{\textbf{QA}} \\

  \multicolumn{1}{c}{\multirow{2}{*}{Feat. Map}} &
  \multirow{2}{*}{Feat. Dim.} &
  \multirow{2}{*}{Sliding} &
  \multirow{2}{*}{Convs.} &
  \multirow{2}{*}{Decay} &
  All &
  AR &
  Other &
  SWDE &
  FDA &
  \multicolumn{1}{c}{SQUAD}
   \\
\multicolumn{1}{c}{} &
   &
   &
   &
   &
  Ppl. $\downarrow$ &
  Ppl. $\downarrow$ &
  Ppl. $\downarrow$ &
  Acc. $\uparrow$ &
    Acc. $\uparrow$ &

  Acc. $\uparrow$ \\ \hline


Taylor Exp. (\nth{2}) & 16 (153) & \cmark (64)  & \cmark  &\cmark  & \textbf{ 8.65 }  &\textbf{2.07 }  &\textbf{9.64 }& \textbf{29.16}    &\textbf{11.71}   &25.07                \\

Performer             & 16 (16)  & \cmark (64)   & \cmark   &\cmark  &9.08   &8.53    &11.62  &  8.10&0.36    &7.47                  \\
CosFormer             & 16 (32)  & \cmark (64)  & \cmark   &\cmark  &9.03    &2.42   &9.98    &  19.35&7.71    &24.63            \\ 

CosFormer             & 64 (128)  & \cmark (64) & \cmark  &\cmark  &8.82    &2.18   &9.80     & 25.47 &9.07        &\textbf{27.85 }          \\ 
\hline
Taylor Exp. (\nth{2}) & 32 (561) & \cmark (64) & \cmark  &\cmark  &\textbf{8.56 }   &\textbf{2.00}    &\textbf{9.57  } &  \textbf{37.62}  &12.89   &\textbf{26.74}                 \\

Taylor Exp. (\nth{2}) & 24 (325) & \cmark (64) & \cmark  &\cmark  &8.58    &2.02    &9.58     &34.38 &\textbf{20.87}   &24.77                 \\  
Taylor Exp. (\nth{2}) & 16 (153) & \cmark (64) & \cmark  &\cmark  & 8.65   &2.07    &9.64     &{29.16}  &11.71   &25.07               \\

Taylor Exp. (\nth{2}) & 8  (45)  & \cmark (64) & \cmark  &\cmark  &8.77    &2.18    &9.75     &23.40  &12.79   &22.35                \\   
\hline

Taylor Exp. (\nth{2}) & 16 (153) & \cmark (64)  & \cmark  &\cmark  &\textbf{8.65 }   &2.07    &\textbf{9.64} &\textbf{29.16}   &11.71  &25.07             \\
Taylor Exp. (\nth{2}) & 16 (153) & \cmark(64)  & \cmark  &\xmark  &\textbf{8.65 }   &\textbf{2.04  } &9.66    & 22.95 &\textbf{12.34}    &\textbf{27.45}                 \\ 

Taylor Exp. (\nth{2}) & 16 (153) & \xmark   & \cmark  &\cmark  &8.91    &2.11    &9.94      &28.62 &10.16   &24.5                  \\
Taylor Exp. (\nth{2}) & 16 (153) & \cmark (64)  & \xmark  &\cmark  &8.74    &2.09    &9.74   &24.66   &2.36    &18.87                 \\
Taylor Exp. (\nth{2}) & 24 (325) & \xmark  & \xmark  &\cmark  &9.49    &2.29    &10.58       &19.62 &8.71    &11.33                 \\

\hline
Taylor Exp. (\nth{2}) & 16 (153) & \cmark (128)  & \cmark  &\cmark  &\textbf{8.61}    &\textbf{2.06 }   &\textbf{9.60 } & \textbf{32.13}  &\textbf{14.39}   &\textbf{31.84}            \\
Taylor Exp. (\nth{2}) & 16 (153) & \cmark (64)  & \cmark  &\cmark  &8.65    &2.07    &9.64    &29.16  &11.71   &25.07               \\

\hline

\end{tabular}
\caption{\textbf{Ablations.} All models are 362M param variants of the \sysname architecture described in \Cref{sec:architecture}, trained to 10 billion tokens on the Pile. We ablate the hyperparameters central to the design of \sysname: (1) the choice of feature map $\phi$ (see \Cref{sec:taylor_series_linear_attention}), (2) the size of the feature dim $d'$ (we show the effective size of the feature after applying the feature map in parantheses, see \Cref{sec:taylor_series_linear_attention}), (3) the use of local sequence mixers (sliding window attention and short convolutions), and (4) the data-dependent decay defined in \Cref{sec:architecture}.  
}

\label{table:ablations}
\vspace{-3mm}
\end{table*}

\clearpage

\section{Experimental Details}
\label{app:exp}

\subsection{Language Model Pretraining}
\label{app:exp-pretraining}

We use A100 80GB Nvidia GPUs to run all experiments. We use training infrastructure closely adapted from the FlashAttention code base: \url{https://github.com/Dao-AILab/flash-attention/tree/main} for all pretraining runs~\cite{dao2023flashattention2}. The Pile data is tokenized using the GPT2BPETokenizer and all models see the data in the same order.
\label{sec:appendix_experiment_details}
\label{app:exp-architectures}
Here we provide details on the hyperaparamters and configurations used for training each architecture. We also provide details on the FLOPs computation. 

\begin{itemize}
    \item \textbf{\sysname} We train using the specifications in \Cref{tab:based-training-details}. Our implementation is provided here: \url{https://github.com/HazyResearch/based}. The initial models were trained and evaluated using the Fast Transformer CUDA kernels discussed in \Cref{app:implementation} \citep{vyas_et_al_2020, katharopoulos-et-al-2020}. We use no input-dependent decay whatsoever when training the models to $30$b and $50$b tokens at $360$m and $1.3$b parameters respectively.

    \item \textbf{Transformer++ \citep{touvron2023llama}} We refer to the modern Llama architecture with Rotary encodings, RMSNorm and SwiGLU as Transformer++, following prior work \cite{gu2023mamba, yang2023gated}. We train using the the specifications in Table \ref{tab:attn-training-details} using the Flash Attention training code provided here: \url{https://github.com/Dao-AILab/flash-attention/tree/main}~\cite{dao2023flashattention2}. 

    \item \textbf{Mamba \citep{gu2023mamba} } We train using the specifications in \Cref{tab:mamba-training-details}, where the parameters are sourced from the Appendix of \cite{gu2023mamba}. The implementation is sourced from the provided reference at \url{https://github.com/state-spaces/mamba}.

    \item \textbf{Hyena \citep{poli2023hyena}} We train using the specifications in \Cref{tab:hyena-training-details}, where the parameters are sourced from the Appendix of \cite{poli2023hyena}. The implementation is sourced from the provided reference at \url{https://github.com/HazyResearch/safari}.
    
    \item \textbf{H3 \citep{dao2022hungry}} We train using the specifications in \Cref{tab:h3-training-details}. The implementation is sourced from the provided reference at \url{https://github.com/HazyResearch/safari}.

    \item \textbf{RWKV \citep{peng2023rwkv}} We train using the specifications in \Cref{tab:rwkv-training-details} and use the reference implementation at \url{https://github.com/BlinkDL/RWKV-LM}. We specifically evaluate RWKV-V5. 

    \item \textbf{Gated Linear Attention (GLA)} We train using the specifications in \Cref{tab:gla-training-details}. We train following the reference implementation at \url{https://github.com/berlino/gated_linear_attention}.
     
\end{itemize}

\noindent We give all models the improved Transformer++ recipe (e.g., SwiGLU) as relevant. 

\subsection{Computing Recurrent State Size}
\label{app:exp-recurrent-state-size} 

In this section, we provide details on how we compute the size of the recurrent hidden state for the results described in \Cref{sec:motivation_empirical}. We train and evaluate six sequence mixers on a synthetic associative recall task: attention~\cite{vaswani2018attention}, sliding window attention~\cite{beltagy2020longformer}, Mamba~\cite{gu2023mamba}, H3~\cite{dao2022hungry}, Hyena~\cite{poli2023hyena}, and \sysname.  
For each, we vary hyperparameters that affect the memory consumption during inference. We compare how \syntheticname accuracy varies with the size of the recurrent hidden state.

\paragraph{\sysname.} The recurrent state size in \sysname is determined by the model dimension $d$ and the size of the hidden dimension after applying the feature map $\tilde{d}$.  The $+ 1$ accounts for the K-state required for computing the denominator. For more details on the recurrent view of \sysname, see \ref{sec:architecture}.

\begin{equation}
    \text{sizeof}(\bm{s}_i) = (d + 1) \times \tilde{d}
\end{equation}

In Based, we use the Taylor Exponential feature map after projecting $d$ down to a smaller dimension $d'$. With this approach, recurrent state size is given by:
\begin{equation}
    \text{sizeof}(\bm{s}_i) = (d + 1) \times (1 + \frac{3d'}{2} + \frac{d'^2}{2})
\end{equation}

In our synthetic experiments, we run \sysname with $d \in \{48, 64, 128\}$ and $d' \in \{8, 16, 24\}$.

\paragraph{Attention.} The recurrent state size (\textit{i.e.} KV-cache size) in attention depends on two parameters: the model dimension $d$ and the sequence length $N$. The $2$ in the expression below accounts for the separate storage for keys and values in the KV-cache.

\begin{equation}
    \text{sizeof}(\bm{s}_i) = 2 \times d \times N 
\end{equation}

In our synthetic experiments we run attention with $d \in \{64, 128\}$. The sequence length $N$ is determined by the task, not the model architecture.

\paragraph{Sliding window attention.} The recurrent state size in sliding window attention is given by the model dimension $d$ and the width of the sliding window $k_\text{sliding}$. The $2$ in the expression below accounts for the separate storage for keys and values in the KV-cache.
\begin{equation}
    \text{sizeof}(\bm{s}_i) = 2 \times d \times \min(N, k_\text{sliding}) 
\end{equation}
In our synthetic experiment we run sliding window attention with $d \in \{128\}$ and $k_\text{sliding} \in \{8, 16, 32, 64, 128, 256, 512, 1024\}$.

\paragraph{Mamba.} The recurrent state size in Mamba is determined by the model dimension $d$ and the number of heads $h$. The $2$ in the expression below accounts for the expansion in the Mamba block.

\begin{equation}
    \text{sizeof}(\bm{s}_i) = 2 \times d \times d_\text{state} 
\end{equation}

In our synthetic experiments, we run Mamba with $d \in \{64, 128, 256\}$ and $d_\text{state} \in \{8, 16, 24\}$.

\paragraph{H3.} The recurrent state size in H3 is determined by the model dimension $d$ and the number of heads $d_\text{state}$. 

\begin{equation}
    \text{sizeof}(\bm{s}_i) =  d \times d_\text{state}
\end{equation}

In our synthetic experiments, we run H3 with $d \in \{64, 128, 256\}$ and $d_\text{state} = \frac{d}{4}$.

\paragraph{Hyena.} The recurrent state size in Hyena is determined by the model dimension $d$ and the number of heads $h$. The $2$ in the expression below accounts for the separate storage for keys and values in the KV-cache.

\begin{equation}
    \text{sizeof}(\bm{s}_i) = d \times N
\end{equation}

In our synthetic experiments, we run Hyena with $d \in \{64, 128, 256\}$.

\subsection{Language Model Evaluation}
\label{app:exp-downstream}
In this section, we provide details on each of the evaluations (columns) reported in \cref{table:main-quality,table:ablations}.  

\paragraph{Pile} \textit{(Language Modeling).} First, we report overall perplexity on the Pile test set~\cite{pile}. Then , to understand how much of the perplexity gap is due to recall capacity, we also evaluate perplexity on two slices (\textit{i.e.} subsets) of the test set:
\begin{enumerate}
    \item \textit{Associative recall(AR) tokens.} Tokens in the final position of a bigram which previously occured in context, but $\leq 1250$ times in the training data. 
    \item \textit{Other tokens.}  All other tokens. 
\end{enumerate}
To construct these slices, we exactly follow the protocol in \citet{arora2023zoology} and refer the reader to that work for more details.
We compute these slices on the first 16 million tokens in the test set.

\paragraph{SWDE} \textit{(Information Extraction).} 
The task in the SWDE benchmark is to extract semi-structured relations from raw HTML websites. 
For example, given an IMBD page for a movie (\textit{e.g.} \textit{Harry Potter and the Sorcerer's Stone}) and a relation key (\textit{e.g.} release date), the model must extract the correct relation value (\textit{e.g.} 2001). 
The SWDE benchmark was originally curated by ~\citet{lockard-etal-2019-openceres} for the task of open information extraction from the semi-structured web. Because we are evaluating the zero-shot capabilities of relatively small language models, we adapt the task to make it slightly easier. Our task setup is similar after to that used in \citet{arora2023evaporate}.

\paragraph{FDA} \textit{(Information Extraction).} 
The task is to extract key-value pairs from a set of PDFs scraped from the FDA website. We use the dataset and labels collected in \cite{arora2023evaporate}.
We break apart the documents into chunks of 1,920 tokens. For every key-value pair that appears in the chunk, we create a zero-shot prompt using the simple prompt template:
{
\centering \\
\texttt{\{chunk\} \textbackslash n \{key\}:} \\
}
We allow the model to generate a fixed number of tokens after the prompt and check (with case insensitivity) if the value is contained within the generation. We report \textbf{accuracy}, the fraction of prompts for which the generation contains the value. \newline 

\noindent Below we include one example of a zero-shot prompt for the key-value pair ``\textit{Type of Test: Quantitative, colorometric, pyranose oxidase (PROD)}". The actual chunk is substantially longer in the dataset (note the ellipsis).   
\begin{mdframed}[style=example]
510(k) SUBSTANTIAL EQUIVALENCE DETERMINATION DECISION SUMMARY ASSAY ONLY TEMPLATE A. 510(k) Number: k180209 B. Purpose for Submission: New Device C. Measurand: 1,5-Anhydroglucitol (1,5-AG) D. \textbf{Type of Test: Quantitative, colorometric, pyranose oxidase (PROD)} E. Applicant: Diazyme Laboratories Inc. F. Proprietary and Established Names: Diazyme 1,5-AG Assay G. Regulatory Information: 1. Regulation section: 21 CFR 864.7470; Glycosylated hemoglobin assay 2. Classification: Class II ... \textit{[1,920 tokens of context from the PDF]} ...
Diazyme’s 1,5-AG assay uses the enzyme pyranose oxidase (PROD) to oxidize the 2nd position hydroxyl group of 1,5-AG and to detect the generated hydrogen peroxide by colorimetry using peroxidase (POD). 
\textbf{Type of Test:}
\end{mdframed}  


\paragraph{SQUAD} \textit{(Question Answering).} 
The Stanford Question Answering Dataset (SQUAD) can be used to evaluate the reading comprehension of language models. The model is given a passage of text and a question whose answer is contained in the passage. \newline 

\noindent  Because the models trained in this work are relatively small-scale (up to 1.3 billion parameters trained on 10 billion tokens) and not instruction fine-tuned, they struggle to answer questions when asked directly. To make the task more amenable to these raw language models, we first use GPT-4 to reformat the questions to more closely resemble the next-token-prediction task the models were trained on:

{
\centering 
\texttt{
Can you rewrite this question and answer as a statement. Ensure that the answer is the last part of the statement. \textbackslash n \textbackslash n Question: \{question\} \textbackslash n\textbackslash n Answer: \{answer\} \textbackslash n\textbackslash n Rewrite:
}
}

For example, the question and answer ``\textit{Question: Which NFL team represented the AFC at Super Bowl 50? Answer: Denver Broncos}"  
was rewritten by GPT-4 as
``\textit{The NFL team that represented the AFC at Super Bowl 50 was the Denver Broncos.}"
We verify that the rewritten sentence does indeed end with the answer, discarding any sentences where it does not (40\% of questions). 

We run the reformatting on 5,000 squad questions from the validation set, yielding a final dataset of \textbf{2,984 questions} formatted as next token predictions. \newline 

\noindent Below we include one example of a zero-shot prompt. The reformatted question is in bold. 

\begin{mdframed}[style=example]
For the third straight season, the number one seeds from both conferences met in the Super Bowl. The Carolina Panthers became one of only ten teams to have completed a regular season with only one loss, and one of only six teams to have acquired a 15–1 record, while the Denver Broncos became one of four teams to have made eight appearances in the Super Bowl. The Broncos made their second Super Bowl appearance in three years, having reached Super Bowl XLVIII, while the Panthers made their second Super Bowl appearance in franchise history, their other appearance being Super Bowl XXXVIII. Coincidentally, both teams were coached by John Fox in their last Super Bowl appearance prior to Super Bowl 50.  \textbf{The team in Super Bowl 50 that had a 15-1 record was the}
\end{mdframed}

\clearpage

\section{Theoretical Results}
\label{sec: theory}

\subsection{Introduction}
Our focus in this section will be on the theoretical results of the paper. Specifically, we will show the equivalence of models $\Based$ and $\Mamba$~\citep{gu2023mamba} with $\BaseConv$, a minimal gated-convolution operator~\citep[Definition 4.1]{arora2023zoology}, and prove lower bounds for the MQAR problem~\citep[Section H.7.1]{arora2023zoology} in various settings. We begin by setting notation and introducing the theoretical formulations of the models.

\paragraph{Notation.} We will be denoting the all \(1\) row vector of size $k$, given by \(\begin{bmatrix}1&1&\ldots&1&1\end{bmatrix}\),  and the all \(0\) row vector of size $k$, given by \(\begin{bmatrix}0&0&\ldots&0& 0\end{bmatrix}\), as \(\bm{1}^{k}\) and \(\bm{0}^{k}\), respectively. We will also construe the standard basis vector $\mathbf{e}_i$ as a column vector in these notes, and adhere to the following matrix indexing convention: $\tbf{M}[i,j]$ is the entry in the $i$th row and the $j$th column,  $\tbf{M}[i,:] \in \F^{1 \times n}$ denotes the $i$th row, and $\tbf{M}[:,j] \in \F^{m \times 1}$ denotes the $j$th column of $\tbf{M} \in \F^{m \times n},$ where $\F$ is a field and the reader can substitute $\F$ for $\R$ for convenience. 
For a matrix $\tbf{M} \in \R^{n \times m}$, we define the pair-wise Hadamard product of columns of $\tbf{M}$ as $\squareMat{\mM} \in \R^{n \times m^2}$
, where
\begin{equation}
\label{eq: sqauareMat}
\begin{aligned}
(\squareMat{\mM})[:, i] := \mathbf{M}[:, j] \odot \mathbf{M}[:, k] \quad \text{for} \quad i \in [m^2],\\
j = \left\lfloor \frac{i - 1}{m} \right\rfloor + 1, \quad k = (i - 1) \mod m + 1.
\end{aligned}
\end{equation}
Moreover, we define the element-wise exponentiation of a matrix $\tbf{M}$ as $\exp[\tbf{M}]$ where $\exp[\tbf{M}]_{ij} = \exp(\tbf{M}_{ij})$.
Next, we denote the {\em Hadamard product} of vectors $\tbf{u}, \tbf{v} \in \F^n$ as $\tbf{u}\odot \tbf{v}$; the operation can be extended to matrices accordingly, and for vectors $\tbf{u}, \tbf{v} \in \F^n$, we denote their {\em linear (or acyclic) convolution} as $\tbf{u} \ast \tbf{v}$

\paragraph{Arithmetic Circuit Notation.}  We briefly introduce the notation of arithmetic circuits~\citep{burgisser2013algebraic}.
An {\em arithmetic circuit} $\calC$ with variables $X\triangleq \{x_1, x_2, \ldots, x_n\}$ over a field $\F$ is interpreted as a directed acyclic graph, where the input nodes are labelled by either the variables from $X$ or constants from $\F$ and the internal nodes are labelled by $+$ or $\times$ with the output being the polynomial computed at the output node. 
    
We shall also refer to the {\em size} of the circuit as the number of nodes, the {\em depth} of the circuit as the length of the longest path between an input node and the output node, and the {\em width} of the circuit as the number of parallel operations in the circuit, or `wires' which will be intersected by a horizontal `cut' through the circuit. Moreover, the {\em degree} of a circuit is defined as the degree of the polynomial computed by the circuit. We summarize this with the following definition:

\begin{definition}
\label{def: circuit-tuple}
   An arithmetic circuit $\calC$ is an {\em $(n,s,\circuitDegree,w)$-circuit} if $\calC$ is an $n$-variate arithmetic circuit of size $s$ and of depth at most $\circuitDegree$, and width $w$.
   \end{definition}
\subsection{The Models}
\label{sec: models}
We now introduce the definitions of the models Based and Mamba for the reader's convenience. Note that we have redefined these models to ensure consistency with the notation presented above.
\subsubsection{Based}
The Based model combines two layer types: BaseConv and LinAtt defined below.
\begin{definition}[\BaseConv~\citep{arora2023zoology}]
Given an input sequence $\vu \in \R^{\inputLength \times \modelDim},$ where $\inputLength$ is the sequence length and $\modelDim$ is the model dimension, a learned weight matrix $\mW^B \in \R^{\modelDim \times \modelDim}$ and biases $\mB^B, \mB^K \in \R^{\inputLength \times \modelDim}$ and a matrix of convolution filters $\mK \in \R^{\inputLength \times \modelDim}$, a \BaseConv layer computes the following:
\begin{equation}
\label{eq: baseconv}
    \bm{z}^{BaseConv} := (\vu{\mW}^B+\mB^B) \odot ({{\mK} \ast \vu + \mB^K}) \in \R^{\inputLength \times \modelDim},
\end{equation}
where the convolutions are applied across the input length $\inputLength$.
\end{definition}
\begin{definition}[\LinAtt~\citep{katharopoulos-et-al-2020}]
Given an input sequence $\vu \in \R^{\inputLength \times \modelDim},$ where $\inputLength$ is the sequence length and $\modelDim$ is the model dimension, a set of linear projections\footnote{By linear projections of a matrix $\vu \in \R^{m \times n},$ we mean $\vu\mW+\mB$ for some weight matrix $\mW \in \R^{n \times n}$ and bias $\mB \in \R^{m \times n}$.} $\ProjectionN_q, \ProjectionN_k \in \R^{\modelDim \times \featureDim}, \ProjectionN_v \in \R^{\modelDim \times \modelDim}$, where $\featureDim$ is the feature dimension, the \LinAtt\ layer computes the following:
\begin{equation}
\label{eq: linatt}
    \bm{z}^{\LinAtt} := ({\concatMat{\mQ}\ \concatMat{\mK}^\top}) \mV \in \R^{\inputLength \times \modelDim},
\end{equation}
where $\mQ := \ProjectionN_q(\vu), \mK := \ProjectionN_k(\vu), \mV := \ProjectionN_v(\vu)$, and we have
\[
    \begin{aligned}
        \concatMat{\mQ} &= [\bm{1}, \mQ, \squareMat{\mQ}] \in \R^{\inputLength \times (1+\featureDim+\featureDim^2)}, \\
        \concatMat{\mK} &= [\bm{1}, \mQ, \squareMat{\mK}] \in \R^{\inputLength \times (1+\featureDim+\featureDim^2)}. \\
    \end{aligned}
\]
\end{definition}

\subsubsection{Mamba}
We now introduce the \Mamba\ model from \citep{gu2023mamba}.
\begin{definition}[\Mamba~\citep{gu2023mamba}]
\label{def: mamba}
    Given an input sequence $\vu \in \R^{\inputLength \times \modelDim},$ where $\inputLength$ is the sequence length and $\modelDim$ is the model dimension, the \Mamba\ layer computes the following:
\begin{equation}
    \label{eq: mamba-ssm}
    \bm{z}^{\Mamba} := 
    \SSM(\overline{\mA}, \overline{\mB}, \mC)(\vu) \in \R^{\inputLength \times \modelDim},
\end{equation}
with the parameters, $\overline{\mA} \in \R^{\stateDim \times \stateDim}, \overline{\mB} \in \R^{\stateDim}$, defined as 
\begin{equation}
\label{eq: bar-param-mamba}
    \begin{aligned}
        \overline{\mA} &:= \exp({\Delta\mA})
        , \\
        \overline{\mB} &:= ({\Delta\mA})^{-1}(\exp({\Delta\mA}) - \mI)\cdot \Delta\mB ,\\
        &= \mA^{-1}(\exp({\Delta\mA}) - \mI)\cdot \mB,
    \end{aligned}
\end{equation}
where $\stateDim$, the state dimension, and $\mA \in \R^{\stateDim \times \stateDim}$ are parameters of the model and do not depend on the input $\vu$, along with the following {\em input-dependent} parameters $\mB, \mC \in \R^{\inputLength \times \stateDim}, \Delta \in \R^{\inputLength \times \modelDim}$ defined as
\begin{equation}
\label{eq: dep-param-mamba}
    \begin{aligned}
        \mB &:= \Linear{\inputLength \times \stateDim}(\vu) \in \R^{\stateDim},\\
        \mC &:= \Linear{\inputLength \times \stateDim}(\vu) \in \R^{\stateDim},\\
        \Delta &:= \Linear{\inputLength \times \modelDim}(\vu) \in \R
    \end{aligned}
\end{equation}
for $i \in [N]$. It is important to note here that the parameters $\overline{\mB}, \mC, \Delta$ are causal\footnote{That is, $\mB[i,:], C[i,:]$ and $\Delta[i,:]$ depend only on $\vu[0 \cdots i-1]$.} and we denote the dependence on upto the $i$th row of the input $\vu$ for $i \in [N]$ by adding a subscript $i$ where the dependence for  $\overline{\mA}_{i} \in \R^{\stateDim \times \stateDim}$ is inherited from $\Delta_{i}$ in \eqref{eq: bar-param-mamba} and we denote $\overline{\mB}[i,:] =: \mB_i, \overline{\mC}[i,:] =: \mC_i$.

Finally, the $\SSM$ in \eqref{eq: mamba-ssm} is realized as a linear recurrence. That is, for every $(i,j) \in [N] \times [\modelDim]$, we have
\begin{equation}
\label{eq: ssm-rec}
    \begin{aligned}
        \bm{h}[i,j] &= \overline{\mA}_{i}\bm{h}[i-1,j] + \overline{\mB}_i\vu[i,j] \\
        \bm{z}[i,j] &= \mC_i^{\top}\bm{h}[i,j]
    \end{aligned}
\end{equation}
where $\bm{h}[i,j] \in \R^{\stateDim}, \bm{z}[i,j] \in \R$ denote the latent state and the output of the $\SSM$ in \cref{eq: mamba-ssm}, respectively.
\end{definition}
\subsection{Equivalency to \BaseConv} 
For a polynomial with variables $X$ over a field $\F$, there exists a corresponding arithmetic circuit $\calC$ over $X$ that computes the output of the polynomial at its terminating node when interpreted as a directed acyclic graph. For any such arithmetic circuit $\calC$ of size $s$ and depth $\Delta$, \citep[Theorem 4.2]{arora2023zoology} showed the existence of an equivalent $\BaseConv$ operator that uses $\tilde{\calO}(s\Delta)$ parameters and $\tilde{\calO}(\Delta)$ layers.
In the sequel, we use this result by expressing the model outputs computed in \eqref{eq: linatt} and \eqref{eq: mamba-ssm} as polynomials in $\vu$ and $\exp({\vu})$ to show the equivalency between these disparate models. We would now like to recall \citep[Theorem 4.2]{arora2023zoology}. Before doing so, we first establish the following definitions from \citep{arora2023zoology}.

\begin{definition}
\label{def: gated-conv}
   An $\modelTuple{\inputLength}{\depth}{\headDim}{\innerN}{\innerD}{\text{Gated Convolution Model}}$ is a stacked sequence to sequence model with $L$ layers such that:
   \begin{enumerate}
     \item  input and output are $\inputLength\times \headDim$ matrices, 
    \item  each layer's operations consist of element-wise gating, convolution, linear projection,  and
    \item all the individual gated convolution layers take in $\innerN \times \innerD$ matrices and output $\innerN \times \innerD$ matrices. We refer to the tuple $\innerDim$ as the \emph{inner dimension} of the model.  
    \end{enumerate}
\end{definition}
We also assume that the input $\modelInput \in \R^{\inputDim}$ is embedded into $\modelInput' \in \R^{\innerN \times \innerD}$  such that

\[
    \modelInput'[n,t] = \begin{cases}
    \modelInput[n,t] \ \ \text{ if } n < \inputLength, \ t < \headDim \ \\
    0 \ \ \text{ otherwise. }
    \end{cases}
\]
The output from the last layer $\vz \in \R^{\innerN \times \innerD}$ is transformed into output $\bm{y} \in R^{\inputDim}$ by extracting the top left $\inputDim$ entries in $\vz$.

\begin{theorem}[\cite{arora2023zoology}, Theorem 4.2]
\label{thm: baseconv-ac}
    For any {\em $(n\headDim,s,\circuitDegree,w)$-arithmetic circuit} $\calC$, there exists an equivalent $\coyoteTuple{\inputLength}{\circuitDegree'}{\headDim}{\innerN}{\innerD}$ with $\inputLength = n, \circuitDegree'=\calO(\Delta\log{w})$, $\innerN=\calO(w), \innerD= \headDim$ that simulates $\calC$.
    \label{thm: gen-ac}
\end{theorem}
\begin{remark}
For notational simplicity, we will use $\vu_{i,j}$ as the symbol for the variable in the polynomial in $\vu$ representing the entry $\vu[i,j]$.
\end{remark}
We now present the results showing equivalency between the models in \cref{sec: models} and the $\BaseConv$ layer in \eqref{eq: baseconv} using \cref{thm: baseconv-ac}.
\begin{proposition}
\label{prop: based-sim}
    Given an input $\vu \in \R^{\inputLength \times \modelDim}$, there exists an equivalent $\coyoteTuple{\inputLength}{O(\log^2(Nd))}{\modelDim}{O(\inputLength(\modelDim + \featureDim^2)}{O(\max(d, \featureDim^2))}$ that computes the output of the $\LinAtt$ layer with feature dimension $\featureDim$, cf. \cref{eq: linatt}. 
\end{proposition}
\begin{proof}
    For the matrices $\mQ, \mK \in \R^{\inputLength \times \featureDim},\mV \in \R^{\inputLength \times \modelDim}$ with the corresponding projection matrices $\mW^Q, \mW^k \in \R^{\modelDim \times \featureDim}, \mW^V \in \R^{\modelDim \times \modelDim}$, a single $\BaseConv$ layer that computes each of these matrices by simply taking identical projection and $\vh^s, \vh^l, \mB^s \equiv 0$ and $\mB^{\ell} \equiv \bbone^{\inputLength \times \modelDim}$, the all $1$ matrix. Using the remembering primitive~\citep[Proposition H.10]{arora2023zoology}, we can compute each of these in turn while remembering others using $O(1)$ layers and $Nd$ parameters. 
    
    Next, we derive an expression for each entry $(i,j) \in [\inputLength] \times [\featureDim^2]$ of $\squareMat{\mQ}, \squareMat{\mK}\in \R^{\inputLength \times \featureDim^2}$. From \eqref{eq: sqauareMat}, observe that each entry of $\squareMat{\mM}$ can be written as the product of entries from $\mM$. Hence we have
    \begin{equation}
        \label{eq: kk-qq-poly}
        \begin{aligned}
            (\squareMat{\mQ})[i, j] &\equiv \mQ[i,k] \cdot \mQ[i,\ell] \\(\squareMat{\mK})[i, j] &\equiv \mK[i,k] \cdot \mK[i,\ell] \\
        \end{aligned}
    \end{equation}
    for $k = \left\lfloor \frac{j - 1}{\featureDim} \right\rfloor + 1, \quad \ell = (j - 1) \mod \featureDim + 1.$ Note, however, that we can simulate the above by first increasing the inner dimension and copying over columns of $\mQ$ to get $\mQ_1, \mQ_2 \in \R^{N \times d}$ defined as 
    \(
        \mQ_1[i,j] := \mQ[i,k]\text{ and } \mQ_2[i,j] := \mQ[i,\ell]
    \)
    for $k = \left\lfloor \frac{j - 1}{\featureDim} \right\rfloor + 1, \quad \ell = (j - 1) \mod \featureDim + 1$ so that $(\squareMat{\mQ}) = \mQ_1 \odot \mQ_2$, which, {\em mutatis mutandis}, also applies to $(\squareMat{\mK})$ We can achieve the copying of the columns by simply using the projection matrix $\mW^B$ and another permutation matrix $\mP$. Apart from the multiplication by $\mP$, we only need to use $O(1)$ layers, and moreover, since the circuit that computes $\mP\vu$ simply rearranges the input, there exists a single \BaseConv\ layer that computes $\mP \vu$~\citep[Corollary H.20]{arora2023zoology}. By the stacking lemma \citep[Lemma H.11]{arora2023zoology}, we can stack these layers to get a composition of the outputs so far to get a $\coyoteTuple{\inputLength}{O(1)}{\modelDim}{O(\inputLength(\modelDim + \featureDim^2)}{O(\max(d, \featureDim^2))}$ model. Moreover, the concatenated matrices $\overline{\mQ}, \overline{\mK}$ $ \in \R^{\inputLength \times (1+\featureDim+\featureDim^2)}$  then take the addition of the computed components so far which again takes $O(1)$ layers of \BaseConv.

    Finally, we can express each entry $(i,j)\in [\inputLength] \times [\modelDim]$ of the output of $\LinAtt$ as a polynomial as follows:
    \begin{align}
        \vz_{i,j}(\vu) &\equiv \sum_{m \in [1+\featureDim+\featureDim^2], n \in [\inputLength]} \overline{\mQ}[i,m] \cdot \overline{\mK}[n,m] \cdot \mV[n,j].
        \label{eq: linatt-poly}
    \end{align}
    Thus, we can derive the arithmetic circuit that computes $\vz_{i,j}(\vu)$ by taking in the outputs of the \BaseConv\ layers so far as input and compute each of the terms inside the sum by multiplying the outputs from all three and compute the sum using additional $\log{\ceil{\inputLength\modelDim}}$ depth. 
    Each term inside the sum requires two multiplication gates with depth $2$, each of which serve as inputs to the circuit with size $\inputLength\modelDim$ computing the sum. Moreover, there are $\inputLength \cdot \modelDim$ such output gates each of which is computed in parallel resulting in a circuit of size $O(\inputLength\cdot\modelDim)$, depth $O(\log(\inputLength\modelDim))$ and width $O(\inputLength\modelDim)$. O Overall, applying \cref{thm: baseconv-ac} then results in an equivalent $\coyoteTuple{\inputLength}{O(\log^2(Nd))}{\modelDim}{O(\inputLength(\modelDim + \featureDim^2)}{O(\max(d, \featureDim^2))}$ model that computes $\vz$.
\end{proof}

\subsection{The Lower Bounds}
In the sequel, we consider the {\em multiple-query associative recall} problem ($\Task$) as defined in \citep[Section H.7.1]{arora2023zoology}. We briefly recall the definition here. 
\begin{displayquote}
    Suppose we are given an input sequence 
$
\bm{u}[0 \cdots 3\inputLength-1] \triangleq \{\paren{\bm{k}_0, \bm{v}_0, \bm{q}_0}, \ldots, \paren{\bm{k}_{\tupleLength-1}, \bm{v}_{\tupleLength-1}, \bm{q}_{\tupleLength-1}}\}
$
with each $\bm{k}_i, \bm{v}_i, \bm{q}_i \in C$ is a token drawn from a vocabulary of size $c = |C|$.
Our goal is then to check, for each $1 \le i \le \tupleLength-1$, whether there exists $0 \le j < i$ such that $\bm{q}_i \equiv \bm{k}_j$, and if so, output $\bm{v}_{j}$. 
\end{displayquote}

\subsubsection{The Space Complexity of AR}
We will start by providing a lower bound on the space complexity of solving the standard associative recall (AR) problem. As AR is a subclass of MQAR, this naturally provides a lower bound on the space complexity of MQAR as well. Here, we formally recall the associative recall problem.
\begin{displayquote}
     The AR problem takes key-value pairs $\{\bm{k}_i, \bm{v}_i\}_{i = 0}^{n-1}$ along with a query $\bm{q}$ appended at the end as input and the goal is to output $\bm{v}_i$ if $\bm{q} = \bm{k}_i$ for some $i \in [0, N-1]$.
\end{displayquote} 
 We now require a randomized communication complexity lower bound result for the {\em index problem}:
\begin{displayquote}
    The index problem has two agents, Alice and Bob, where Alice has a string $\bm{x} \in \{0,1\}^n$ and Bob has an index $i \in [n]$, and the goal for the players is to output the $i$-th entry $\bm{x}_i$. Moreover, we also require the communication to be {\em one-way}: only Alice is allowed to send a single message to Bob and Bob needs to output the answer.
\end{displayquote} 
We will make use of the following lower-bound result.
\begin{theorem}[\cite{jayram2008one}]
\label{thm: space-index}
    The one-way randomized communication complexity\footnote{The randomized communication complexity of function $f$ is defined as $\min_{\pi} \norm{\pi}$, where $\pi$ ranges over all randomized protocols that can solve $f$ with probability of success at least $2/3$.} of the index problem for sending an $n$-length bit string is $\Omega(n)$.
\end{theorem}

\subsubsection{Lower Bound for Recurrent Models}
We now use \cref{thm: space-index} to first provide a lower bound on the number of bits required by the following class of models to solve AR.
\begin{definition}[Recurrent Models]
\label{def: reg-model}
A model $\calM$ taking an input $\vu \in \R^{\inputLength \times \modelDim}$, where $\inputLength$ is the input length and $d$ is the model dimension, is termed a \emph{recurrent model} if its $i$-th state, representing the output at location $i$, $\mZ_{\calM}^i \in \R^{\innerD}$, with $\innerD$ denoting the state size, is determined exclusively by the preceding elements of the input $\vu[0 \ldots i-1]$. The state $\mZ_{\mathcal{M}}^i$ represents the accumulated information of the model depending on the inputs up to the $i$-th element, and is distinct from learned parameters that are static with respect to the input sequence.

Specifically, $\mZ_{\calM}^i(\vu) = \phi(\vu[0 \ldots i-1])$, indicating that the state is a function of the input history but not of the entire input sequence simultaneously. Moreover, we can express this as:
\begin{equation}
\label{eq: reg-dep}
\mZ_{\calM}^i(\vu) = f_{\calM}^i(\mZ_{\calM}^{i-1}, \vu[i]),
\end{equation}
for a sequence of functions $\{f_{\calM}^i\}_{i \in [\inputLength]}$, where each function is tailored to evolve the state based on the immediate past state and the current input. 
\end{definition}
\begin{remark}
    Note that \cref{def: reg-model} excludes models that inherently require the entire input sequence for computation at any state, such as those based on non-causal convolutional operations over the full input.
\end{remark}
\begin{theorem}
\label{thm: space-reg}
    Any recurrent model $\calM$ (\cref{def: reg-model}) that solves AR requires $\max_i \abs{\mZ_{\calM}^i}$ to be at least $\Omega(\inputLength)$-bits.
\end{theorem}
\begin{proof}
    Consider an instance $(\bm{x}, i)$ of the index problem with $\bm{x} \in \{0,1\}^N$. We now describe the corresponding instance of the AR problem:
    \begin{equation}
        \label{eq: ar-index}
        \{j, \bm{x}_j\}_{j = 0}^{N-1}, i.
    \end{equation}
    Next, consider the following one-way protocol for solving the index problem using the regressive model $\calM$. Alice with their access of $\bm{x} \in \{0,1\}^N$ generate an input for AR (without the query) as in \eqref{eq: ar-index}. Alice then runs the model $\calM$ on $\{i, \bm{x}_j\}_{j = 0}^{N-1}$ and sends the memory content of running the model $\calM$ to Bob. This should include the state $\mZ^{N-1}_{\calM}$ of size $\innerD$ as we can reasonably assume that both have access to the set of functions $\{f_{\calM}^j\}_{j \in [N]}$. Since we assume that this model solves AR, the output ${\tt Out}[N, :] = \bm{x}_i$ should contain the associated value of $i$. Here, Bob can compute ${\tt Out}[N, :]$ by using the memory content sent by Alice and applying the function $f^N$ as follows.
    \[
    \bm{x}_i = {\tt Out}[N, :] = f^N(\mZ^{N-1}, \bm{u}[N]).
    \]
    That is, the total number of bits that are communicated in this protocol is $\abs{\mZ_{\calM}^{N-1}}$. Now, if $\max_j \abs{\mZ_{\calM}^j}$ is $o(N)$ bits, we have shown that a one-way communication protocol exists for solving the index problem exists that uses $o(N)$ communication complexity. This contradicts \cref{thm: space-index} and hence, we conclude that the model $\calM$ solving AR also needs $\Omega(N)$ bits.
\end{proof}

\begin{corollary}
    \label{cor: mamba-ar}
    Given an input $\vu \in \R^{\inputLength \times \modelDim}$ to the AR problem, a causal \Mamba\ model with all entries in its computation taking $O(1)$ bits needs $\modelDim + \stateDim \ge \Omega(N)$ to solve AR. 
\end{corollary}
\begin{proof}
    We will first show that causal \Mamba\ is a recurrent model. To see this, first observe \eqref{eq: ssm-rec} and note the fact that the input-dependent parameters $\overline{\mA}, \overline{\mB}, \mC, \Delta$ are causal as mentioned in \cref{def: mamba}.

    Next, due to \eqref{eq: ssm-rec}, in order to compute $\vz_{\inputLength,:} \in \R^\modelDim$, we need $\mC_{N} \in \R^{\stateDim}, \overline{\mB}_{N} \in \R^{\stateDim}$ and $\Delta_{N} \in \R^{\modelDim}$ along with $\vh[N-1,:] \in \R^{\stateDim}$. Here, we have the $(N-1)$-st state $\mZ_{\Mamba}^{N-1} \in \R^{3\stateDim + \modelDim}$ given by
    \[
     \mZ_{\Mamba}^{N-1} := \{\vh[i-1,:], \Delta^1_N, \overline{\mB}^1_N, \mC^1_N\},
    \]
    where $\Delta^1_N, \overline{\mB}^1_N, \mC^1_N$ are all linear functions of $\vu[0 \cdots N-1]$ that we receive from the $(N-1)$-st state and we compute $\Delta^2_N, \overline{\mB}^2_N, \mC^2_N$ as linear functions of $\vu[N]$ so that we have $\Delta_N = \Delta^1_N + \Delta^1_N, \overline{\mB}_N =\overline{\mB}^1_N + \overline{\mB}^2_N, \mC_N = \mC^1_N+ \mC^2_N$. We can then define the function $f^N$ as follows:
    \[
    \begin{aligned}
        \mZ_{\Mamba}^{N}[j] 
        &= \exp(\Delta_N[j]\mA)\bm{h}[N-1,j] + \overline{\mB}_N\vu[N,j]\\
        &= \overline{\mA}_N\bm{h}[N-1,j] + \overline{\mB}_N\vu[N,j],\\
        {\tt Out}[N, j] &= f^{N}(\mZ_{\Mamba}^{N-1})[j] = \mC_N^{\top}\mZ_{\Mamba}^{N}[j].
    \end{aligned}
    \]
    Thus, due to \cref{thm: space-reg}, we can conclude that $\abs{\mZ_{\Mamba}^{N-1}}$ does require $\Omega(\inputLength)$-bits to solve AR. Finally, assuming each entry of $\mZ_{\Mamba}^{N-1}$ needs $O(1)$ bits to represent, the overall state $\mZ_{\Mamba}^{N-1}$ needs $O(\modelDim + \stateDim)$ to represent, which completes the proof of the claim.
\end{proof}

\subsubsection{Lower Bound on the Number of Layers for AR}
Next, we will again use \cref{thm: space-index} to provide a better bound on the number of layers required to solve AR. (Note that since AR is a special case of $\Task$, the result below immediately implies \Cref{thm: poly-layer-main}.)

\begin{theorem}
\label{thm: poly-layer}
    Given an input $\vu \in \{0,1\}^{\inputLength \times \modelDim}$ to the AR problem with any encoding such that $\log{c} \le d \le 2^{(\log{N})^{1-\epsilon}}$ for $\epsilon > 0$,
    and $c$ possible tokens from the vocabulary with $c \le N$,
    a data-independent \BaseConv\ model with model parameters taking $O(\log{N})$ bits needs $\Omega(\epsilon\log\log{N})$ layers to solve AR.
\end{theorem}
\begin{proof}
    For a \BaseConv\ model that solves AR using $L$ layers, by definition, there exists a polynomial $P(\vu)$ of degree at most $2^L$ that solves AR for any $\vu \in \{0,1\}^{\inputLength \times \modelDim}$\footnote{Since \BaseConv\ is data independent, note that the polynomial $P(\cdot)$ is defined once we fix $N$ and $d$.}. This is because for the output of the $i$th layer of \BaseConv, given by $\mZ_{\BaseConv}^i$,  we have
    \[
    \mZ_{\BaseConv}^i(\mY_{\calM}^{i-1}) \equiv P^i(\mZ_{\BaseConv}^{i-1}), \quad \deg(P^i) \le 2,
    \]
    for some polynomial $P^i$ of degree $2$ which simply takes the inner products allowing the model to solve AR, where $\mZ_{\BaseConv}^{0} := \vu$. Further, for such a model with $L$ layers, by composition, the output of the $i$-th layer for $i \in [L]$ is also a polynomial over the input $\vu$ and has degree at most $2^i$. At the end, we have a polynomial $P(\vu)$ of degree $\le 2^L$ for $\vu \in \{0,1\}^{\inputLength \times \modelDim}.$
    As in the proof of \cref{thm: space-reg}, again take the instance instance $(\bm{x}, i)$ of the index problem with $\bm{x} \in \{0,1\}^N$ and the corresponding instance of the AR problem as before
    \begin{equation}
        \label{eq: ar-index-dup}
        \vu := \{j, \bm{x}_j\}_{j = 0}^{N-1}, i.
    \end{equation}
    Next, we build the following one-way protocol for solving the index problem using the \BaseConv\ model from the hypothesis that it solves AR. Alice with their access of $\bm{x} \in \{0,1\}^N$ will again generate an input  $\bm{u}$ for AR (without the query) as in \eqref{eq: ar-index-dup}. 
    
    Alice first takes the values $\va:= \vu[0: N-2, :] \in \{0,1\}^{(N-1)\times d}$ and substitutes these known $(N-1)d$ values to define the following polynomial:
    \begin{equation}
        \label{eq: q-poly}\
        Q(\vu_{N-1,0},\ldots,\vu_{N-1,d-1})=P(\va,\vu_{N-1,0},\ldots,\vu_{N-1,d-1}).
    \end{equation}
    Here, note that $Q$ is a polynomial in $d$ variables that correspond to the values $\vu[N-1, :]$ that Bob has and trivially has degree $D \le 2^L$. 
    Now, Alice can run the model $\calM$, retrieve the coefficients of $Q$, and send it to Bob. Since we assume that $P$ solves AR, Bob can take the coefficients of $Q$ and substitute $\vu[N-1, :]$ to $Q$ to compute $P(\vu)$ which is the associated value of $i$.

    Here, the polynomial $Q$ that Alice sends has at most $d^{2^L}$ coefficients as each term in $Q$ can have degree at most $2^L$.  If each such coefficient has $B$ bits, then using \cref{thm: space-index}, the total number of bits being communicated must satisfy $B \cdot d^{2^L} \ge \Omega(N)$. This follows from the fact that if $B\cdot d^{2^L} \le o(N)$, then since the associated value of $i$ in \eqref{eq: ar-index-dup} is the answer to the indexing problem, we have shown that a one-way communication protocol for solving the index problem uses $o(N)$ communication complexity, which then contradicts \cref{thm: space-index}. Thus, we must have
    \[
        B\cdot d^{2^L} \ge \Omega(N) 
    \implies 2^L\log(d) \ge \log\paren{\frac{N}{B}} - O(1).
    \]
    Taking logarithm of both sides then yields
    \begin{align}
    L &\ge {\log\paren{\frac{\log\paren{\frac{N}{B}}}{\log\paren{d}}}} - O(1) \ge {\log\paren{\frac{\log{N} - \log{B}}{\log\paren{d}}}} - O(1) \nonumber\\
    &\ge {\log\paren{\frac{\log{N} - \log{B}}{(\log{N})^{1-\epsilon}}}}\label{eq: L-N-B},
    \end{align}  
    where we use the fact that $d \le 2^{(\log{N})^{1-\epsilon}}$ for any $\epsilon > 0$ in \eqref{eq: L-N-B}.
    
    Moreover, as the model parameters are assumed to be $O(\log{N})$ bits, any coefficient in $Q$ should have absolute value at most $\paren{2^{O(\log{N})}\cdot Nd}^{2^L}$ as each coefficient can be a product of at most $Nd$ variables. That is, for some $\alpha>0$, we have the following bound on each coefficient:
    \[
       2^B \le (N^{\alpha+1}d)^{2^L} \le (N^{(\alpha+2)})^{2^L}
    \]
    where the last equality uses the fact that $d \le N$. We thus have 
    \begin{equation}
        \label{eq: q-coefficient}
    \log(B) \le \log(\alpha+2) + L + \log\log{N}.
    \end{equation}
 Substituting \eqref{eq: q-coefficient} to \eqref{eq: L-N-B}, we get
    \begin{align}
        L \ge \log\paren{\frac{\log{N} - \log(\alpha+2) - L - \log\log{N}}{(\log{N})^{1-\epsilon}}} \label{eq: L-lower}
    \end{align}
    Now, if $L > \log\log{N}$, we are done. Otherwise, if $L \le \log\log{N}$, then we can substitute this to \eqref{eq: L-lower} to get
    \begin{align}
        L &\ge  \log\paren{\frac{\log{N} - \log(\alpha+2) - 2\log\log{N}}{(\log{N})^{1-\epsilon}}} \nonumber\\
        &= \log\paren{\log{N} - \log(\alpha+2) - 2\log\log{N}} - (1-\epsilon)\log\log{N}\label{eq: L-explicit}
    \end{align}
    We now claim that first term in \eqref{eq: L-explicit} satisfies the following:
    \begin{equation}
        \label{eq: claim}
        \log\paren{{\log{N} - \log(\alpha+2) - 2\log\log{N}}} \ge (1-\frac{\epsilon}{2}) \log\log{N}.
    \end{equation}
    To see this, note that, for sufficiently large enough $N$, the following holds:
    \begin{align*}
        \frac{\log{N}}{2} \ge \log(\alpha + 2) + 2\log\log{N},
    \end{align*}
    hence, we get
    \[
    \log\paren{{\log{N} - \log(\alpha+2) - 2\log\log{N}}} \ge \log\paren{\frac{\log{N}}{2}} \ge \log\log{N} -1 \ge (1-\frac{\epsilon}{2}) \log\log{N}.
    \]
    This proves the claim in \eqref{eq: claim}. Finally, using \eqref{eq: claim}, \eqref{eq: L-explicit} leads to the following:
    \[
    L \ge (1-\frac{\epsilon}{2}) \log\log{N} - (1-\epsilon)\log\log{N} = \frac{\epsilon}{2}\log\log{N},
    \]
    which still provides the lower bound $L = \Omega(\epsilon\log\log{N})$, as desired. 
\end{proof}

\begin{remark}
We remark that it is possible to extend \cref{thm: poly-layer} to any model whose output from each layer is a polynomial of some degree $\Delta \ge 2$ ot get a lower bound of $\Omega(\epsilon\log\log{N}/\log{\Delta})$.
\end{remark}
\subsubsection{Lower Bound on the Number of Layers for $\Task$ with $d=\log_2{c}$}
\label{sec: d_log}
\paragraph{Setup.} We take $d = \log_2{c}$ to encode all $c$ possible tokens from $C$. That is, all the $2^d$ possible $d$-bit vectors can appear as a token in the input for $\Task$. We will show that data-independent \BaseConv\ needs $\Omega(\log{d})$ = $\Omega(\log{\log{c}})$-layers to solve this setting of $\Task$, while Attention (+ReLU) can solve this in $O(1)$ layers.

We first provide the trivial solution using Attention (+ReLU).
\begin{proposition}
    Attention (with linear biases and ReLU) followed by two layers of MLPs can solve MQAR for an input sequence $\bm{u} \in \{0,1\}^{3N \times d}$ such that $d = \log_2(c)$ in $O(1)$ layers.
\end{proposition}
\begin{proof}
Given a row $\vu[i,:] \in \{0,1\}^d$, we express each row as $\vw[i,:] \in \{-1,1\}^d$ by applying the projection $\vu\mW + \mB$, where $\mW:= \diag(2, \ldots, 2) \in \R^{d \times d}$ and the bias matrix $\mB$ is the matrix of all $-1$'s so that $\vw[i,j] = 2\vu[i,j]-1$.
Then, we can specify the query and key projection matrices $\tbf{Q}, \tbf{K}, \tbf{V} \in \R^{3N \times d}$ as follows:
\[
\begin{aligned}
    \tbf{K}[i,:]&\equiv \begin{cases}
        \vw[i,:] = \vk_{\floor{i/3}} & \text{if } i \equiv 0\mod 3\\
        \bm{0} &\text{otherwise}
    \end{cases}\\
    \tbf{Q}[i,:] &\equiv \begin{cases}
        \vw[i,:] = \vq_{\floor{i/3}} & \text{if } i\equiv 2 \mod 3\\
        \bm{0} &\text{otherwise}
    \end{cases}\\
    \tbf{V}[i,:]&\equiv \begin{cases}
        \vw[i+1,:] = \vv_{\floor{i/3}} & \text{if } i \equiv 0\mod 3\\
        \bm{0} &\text{otherwise}
    \end{cases}
\end{aligned},
\]
where the values are shifted to the corresponding key index.
Computing the pair-wise inner products then yields
\[
\tbf{QK}^{\top}[i,j] \equiv \begin{cases}
    \angles{\bm{q}_{\floor{i/3}}, \bm{k}_{\floor{j/3}}} & \text{if } i\equiv 2 \mod 3\text{ and }j \equiv 0\mod 3\\
    \bm{0} &\text{otherwise}
\end{cases}
\]
However, since both $\bm{q}_{\floor{i/3}}, \bm{k}_{\floor{j/3}} \in \{-1, 1\}^d$, we have $\angles{\bm{q}_{\floor{i/3}}, \bm{k}_{\floor{j/3}}} \le d$ with equality iff $\bm{q}_{\floor{i/3}} \equiv \bm{k}_{\floor{j/3}}$. We then subtract off $d-1$ from each of the $3N \times 3N$ entries by taking the bias $\tbf{B} \in \R^{3N \times 3N}$ as the matrix with each entry $-d+1$. Let $\tbf{Z} := \textsc{ReLU}(\tbf{QK}^{\top}+ \tbf{B})$ so that we have
\[
\tbf{Z}[i,j] = \bbone\{\bm{q}_{\floor{i/3}} \equiv \bm{k}_{\floor{j/3}}\}.
\]
Next, as we may have multiple matches and we only need to return $1$, we modify $\mZ$ by multiplying with the matrices $\mW_1, \mW_2 \in \R^{d \times d}$ and adding the bias $\mB \in \R^{d \times d}$ defined as follows:
\[
\mW_1[k, j]:= \begin{cases}
    1 &\text{if } k \ge j \\
    0 &\text{otherwise}
\end{cases},\quad 
\mW_2[\ell, k]:= \begin{cases}
    -1 &\text{if } k = 0 \\
    1 &\text{if } k = \ell, \ell \neq 0\\
    0 &\text{otherwise}
\end{cases},\quad
\mB[i,j] = 1.
\]
For $\mZ_1:= \mZ\mW_1$ and $\mZ_2:= \mZ\mW_1\mW_2$, we have:
\begin{align*}
    \mZ_1[i,j] &= \sum_{k} \mZ[i, k] \mW_1[k, j] = \sum_{k \ge j} \mZ[i,k],\\
    \mZ_2[i,j] &= \sum_{k} \mZ_1[i, k] \mW_2[k, j] = \mZ_1[i,j] - \mZ_1[i,0].
\end{align*}
That is, each entry in $\mZ_1$ sums the entries in the row that are at the same or higher column index while each column in $\mZ_2$ subtracts the first entry—the sum of all entries in the row—from each entry in the row. Semantically, for each row in $\mZ_1$, the entries from $0$ to the index of the first match must have the same value, and thus, are the only non-negative entries in $\mZ_2$. Next, we add the bias and activate under $\textsc{ReLU}$ to get $\mZ' \in \R^{3N \times d}$:
\[
\mZ'[i,k] := \textsc{ReLU}(\mZ_2 + \mB)[i,k] = \begin{cases}
    1 &\text{if } k \le \min\{j \lvert\ \vq_{\floor{i/3}} \equiv \vk_{\floor{j/3}}\}\\
    0 &\text{otherwise.}
\end{cases}
\]
Now, we multiply by the weight matrix $\mW_3 \in \R^{3N \times d}$ defined as
\[\mW_3[k,j] := \begin{cases}
   -1 &\text{if } k = j+1 \\
    1 &\text{if } k = j\\
    0 &\text{otherwise}
\end{cases}
\] 
This yields the retriever $\overline{\mZ} = \mZ'\mW_3 \in \R^{3N \times d}$ given by
\[
\overline{\mZ}[i,k] := \sum_{\ell} \mZ'[i,\ell]\mW_3[\ell,k] = \mZ'[i,k] - \mZ'[i,k+1] = \bbone\{k = \min\{j \lvert\ \vq_{\floor{i/3}} \equiv \vk_{\floor{j/3}}\}\}.
\]
Finally, we multiply with the values $\tbf{V}$ to get
\[
(\overline{\mZ}\tbf{V})[i,:] \equiv \overline{\mZ}[i,:]\tbf{V} \equiv \overline{\mZ}[i,j^*]\cdot \tbf{V}[j^*,:] \equiv \begin{cases}
    \bm{v}_{j^*} &\text{if }\vq_{\floor{i/3}} \equiv \vk_{\floor{j^*/3}}, j^* =  \min\{j \lvert\ \vq_{\floor{i/3}} \equiv \vk_{\floor{j/3}}\}. \\
    \bm{0} &\text{if no such $j^*$ exists.}
\end{cases}
\]
That is, the row corresponding to the query returns the value associated to the first matching key. Thus, the model with Attention (computing $\mZ$) followed by two MLPs computing $\mZ'$ and $\overline{\mZ}$, respectively, solves the MQAR problem.
\end{proof}

Next, we relate the output of $L$ layers of $\BaseConv$ to the degree of the polynomial that it computes.
\begin{lemma}
\label{lem: based-degree}
    For any input sequence $\bm{u}$, there exists a multilinear polynomial equivalent (over Boolean inputs) to the polynomial computed by $L$ layers of $\BaseConv$ with degree at most $2^L$.
\end{lemma}
\begin{proof} 
Let $P(\bm{u})$ be the polynomial computed by $L$ layers of \BaseConv. Since the output of a single layer of $\BaseConv$ is equivalent to a polynomial over the input variables with degree at most $2$, composing $L$ such layers yields a polynomial of degree at most $2^L$. However, $P(\bm{u})$ need not be multi linear, but the polynomial defined as
\[
{Q}(\bm{u}):= (\cdots ((P(\bm{u}) \mod(u_1^2-u_1))\mod(u_2^2-u_2))\cdots ) \mod(u_{3Nd}^2-u_{3Nd})
\] is equivalent to $P(\bm{u})$ as $(u_i^2-u_i)$ evaluates to $0$ for each input var $u_i \in \{0,1\}$. However, $\deg({Q}(\bm{u})) \le \deg(P(\bm{u}))$, and thus, the claim holds.
\end{proof}
We now relate the MQAR (in the above setting) to the degree of the polynomial that it computes.
\begin{lemma}
\label{lem: unqiue-mqar}
    The MQAR problem with $d = \log_2(c)$ is represented by a multi-linear polynomial of degree $2d+1$.
\end{lemma}
\begin{proof}
We will start by specifying the obvious Boolean circuit that solves MQAR. First, we take the XNOR of keys and queries bitwise as follows.
\begin{equation}
\label{eq: xnor}
        \bm{x}^{ij} = {\bm{q}_i\ {\tt xnor}\ \bm{k}_j} := \paren{\bm{q}_i \wedge \bm{k}_j} \vee \paren{\neg \bm{q}_i \wedge \neg \bm{k}_j} \text{ for } i > j,
\end{equation}
where, for $\vx, \vy \in \{0,1\}^d$, we have
\[
[\vx \ {\tt xnor}\  \vy][k]:= \begin{cases}
    1 &\text{if }\vx[k] = \vy[k] \\
    0 &\text{othwerise}
\end{cases}
\]
That is, each bit from $\bm{x}^{ij}$ is set to $1$ iff the corresponding bits from $\bm{q}_i$ and $\bm{k}_j$ match. Next, we take the AND of the $d$-bits to get 
\begin{equation}
\label{eq: and}
        \bm{y}^{ij} := \bigwedge_{k \in [d]} \bm{x}^{ij}_k, i > j.
\end{equation}
Thus, $\bm{y}^{ij}$ is set to $1$ iff the query $\bm{q}_i$ matches with the key $\vk_j$. Finally, we AND with each bit of the values to get the output $\bm{z}^{ij}$ with the $k$th bit for $k \in [d]$ given by
\begin{equation}
\label{eq: val}
        \bm{z}^{ij}_k := \bm{y}_{ij} \wedge [\bm{v}_j]_k.
\end{equation}
Thus, the output of the circuit can be represented as
\[
\bm{z}^{ij} = 
\begin{cases}
    \bm{v}_i & \text{if }\bm{q}_i \equiv \bm{k}_j, i > j\\
    \bm{0} & \text{otherwise.}
\end{cases}
\]
We can now directly translate the above circuit into a multi-linear polynomial. With slight abuse of notation, we have the following correspondence for \eqref{eq: and}, where $\bm{u}_i \equiv \bm{q}_i, \bm{u}_j \equiv \bm{k}_j, i > j$ and we use $\bm{u}_{ij}$ to represent the variable corresponding to the entry $\bm{u}[i,j]$.
\[
        \bm{x}^{ij}_k(\bm{u}) := \bm{u}_{ik}\bm{u}_{jk} + (1-\bm{u}_{ik})(1-\bm{u}_{jk})\quad \text{for each }k \in [d], i > j.
\]
Next, we translate \eqref{eq: and} as follows.
\[
        \bm{y}^{ij}(\bm{u}) := \prod_{k \in [d]} \paren{\bm{u}_{ik}\bm{u}_{jk} + (1-\bm{u}_{ik})(1-\bm{u}_{jk})}.
\]
Finally, we can write the polynomial that computes MQAR as follows.
\begin{equation}
    \label{eq: out-poly}
        \bm{z}^{ij}(\bm{u}) := \paren{\prod_{k \in [d]} \bm{u}_{ik}\bm{u}_{jk} + (1-\bm{u}_{ik})(1-\bm{u}_{jk})}\bm{u}_{(i+1)k} \quad \text{for each }k \in [d], i > j,
\end{equation}
where $\bm{u}[i+1,:] \equiv \bm{v}_j$. It is then easy to observe that \eqref{eq: out-poly} is multi-linear and has degree $2d+1$.
\end{proof}

We are now ready to provide the lower bound.
\begin{theorem}
\label{thm: mqar-1hot}
    A data-independent \BaseConv\ model needs $\log(2d)$-layers to solve $\Task$ for an input sequence $\bm{u}\in \{0,1\}^{3N\times d}$ with $d = \log_2(c)$.
\end{theorem}
\begin{proof}
Due to \Cref{lem: unqiue-mqar}, we know there exists a multi-linear polynomial that solves MQAR, and due to \citep[Lecture 3, Proposition 4]{kopparty2020notes}, it is unique. Specifically we cannot solve MQAR with a multi-linear polynomial of degree $\le 2d$. Now, assume that there is a \BaseConv\ model with $L$ layers that exactly solves MQAR. Then, due to \Cref{lem: based-degree}, this yields a multilinear polynomial $P(\bm{u})$ of degree at most $2^L$. Here, if $L \le \log(2d)$, then the resulting \BaseConv\ with $L$ layers results in a multilinear polynomial of degree $\le 2d$. This contradicts the above claim that we cannot have a multi linear polynomial of degree $< 2d+1$ that exactly represents MQAR. Consequently, a data-independent \BaseConv\ model needs $\ge \log(2d)$-layers to solve $\Task$.
\end{proof}

\subsection{Lower Bound on the Number of Layers for $d \ge \log_2{c}$ with Specific Encodings}
\label{sec: root-hot}
\subsubsection{The Equality Problem}
For an input pair $\bm{u}_1, \bm{u}_2$ where each $\bm{u}_i$ is a token drawn from a vocabulary of size $c = |C|$ and embedded in $\{0,1\}^d$, we define the {\em equality problem} (EQ) as checking whether the two encodings are equal: $\bm{u}_1 \equiv \bm{u}_2$. 

We first note that any model that solves $\Task$ also solves EQ via the following proposition.
\begin{proposition}
    \label{prop: eq_mqar}
    Any model $M_{\Task}$ that solves MQAR also solves EQ using the same number of layers.
\end{proposition}
\begin{proof}
    If there exists a model $\textsc{M}_{\Task}$ that solves MQAR using $L$ layers, then for an arbitrary input instance for EQ given by $\bm{u}_1, \bm{u}_2 \in \R^{2 \times d}$,
    we can produce the following input instance for MQAR: $\bm{u}:= \{(\bm{u}_1, \bbone, \bm{u}_1), (\bm{u}_2, \bbone, \bm{u}_2)\}$ and solve EQ using $L$ layers with $\textsc{M}_{\Task}$ returning $\bbone$ iff there is a match.
\end{proof}

Due to \Cref{prop: eq_mqar}, we obtain the following corollary.
\begin{corollary}
\label{cor: eq-mqar-lower}
    Any lower bound $\overline{L}$ on the number of layers $L$ of \BaseConv\ to solving EQ is also a lower bound on the number of layers required for solving $\Task$. 
\end{corollary}

We now try to prove a lower bound for the case of $d \ge \log_2{c}$. First, note that there are embeddings here where the lower bound from \ref{thm: mqar-1hot} holds: consider the embedding where the first $\log_2{c}$ has the compact binary embedding as before but the last $d-\log_2{c}$ bits are the same for all the tokens. We will instead prove a lower bound for a more interesting set of embeddings.

\subsubsection{The $p$-Hot Encoding for $p \ge 1$}
\begin{definition}[(Almost) $p$-Hot Encoding]
\label{def: phot}
    We define the {\em $p$-hot encoding} to be the collection of embeddings for a token $\bm{x}_t$ with $0\le t <c$ such that we express $t$ in base $\sqrt[p]{c}:(t_0,..,t_{p-1}) \in [0,\sqrt[p]{c})^p$ and represent each $t_i$ as one hot encoding in $\{0,1\}^{\sqrt[p]{c}}$. That is, we take $d = p\cdot\sqrt[p]{c}$.

    Moreover, we define the {\em almost $p$-hot encoding} to be the collection of embeddings where each $t_i$ is mapped in $\{0,1\}^{\sqrt[p]{c} - 1}$ obtained by dropping the last bit of its one-hot encoding in $\{0,1\}^{\sqrt[p]{c}}$.

    Note that both of the encodings have $p$-many blocks derived from each of the one-hot encodings. 
\end{definition}

\begin{definition}[Block-Exclusive]
\label{def: block-exlusive}
    We say that a polynomial $P$ with variables in $\bm{u}:= \paren{\bm{u}_0, \ldots, \bm{u}_{p-1}}$ is {\em block-exclusive} if each non-zero monomial in $P$ given by the product 
    \[
    \prod_{{i \in [p],\ j \in [\sqrt[p]{c}]}} \bm{u}_{i,j}
    \]
    does not contain any product of the form $\bm{u}_{i, j}\bm{u}_{i, j'}$ for $i \in [p], j,j' \in [\sqrt[p]{c}]$.
\end{definition}

\begin{remark}
    The condition specified in \Cref{def: block-exlusive} ensures that a block-exclusive polynomial is necessarily multilinear, as it disallows the term $\bm{u}_{i, j}\bm{u}_{i, j'}$ for $j = j'$ in any non-zero monomial.
\end{remark}

\begin{lemma}
\label{lem: block-ex}
    For any Boolean function $f: \{0,1\} \to \{0,1\}$ with inputs from the almost $p$-hot encoding {\em or} the $p$-hot encoding setting, there exists a block-exclusive polynomial equivalent to $f$.
\end{lemma}
\begin{proof}
    Given an input $\bm{u}$ to $f$ from the almost $p$-hot encoding or the $p$-hot encoding such that $\bm{u}:= \paren{\bm{u}_0, \ldots, \bm{u}_{p-1}}$, we first observe that the polynomial $P(\bm{u})$ representing $f(\vu)$ cannot have a non-zero monomial with variables from the same block. Specifically, for $0 \le j < p$, any non-zero monomial in $P$ cannot have a product of the form $\vu_{j, k}\vu_{j,k'}$ for $k \neq k'$. To see this, assume that there exists a non-zero monomial in $P$ with at least two terms $\vu_{j,k}\vu_{j,k'}$ from the same $j$th block in $\vu$, then monomial always evaluates to $0$ as the $j$th block is derived from the one-hot encoding in $\{0,1\}^{\sqrt[p]{c}}$ or the almost one-hot encoding in $\{0,1\}^{\sqrt[p]{c}-1}$, and hence, cannot have more than one bit set to $1$.

    Next, if a non-zero monomial in $P$ does contain a product of the form  $\vu_{j, k}\vu_{j,k'}$ for $k, k' \in [\sqrt[p]{c}]$, we can define the polynomial 
    \[
    {Q}(\bm{u}):= (\cdots ((P(\bm{u}) \mod(u_{0,0}^2-u_{0,0}))\mod(u_{0,1}^2-u_{0,1}))\cdots ) \mod(u_{p-1, \sqrt[p]{c}-1}^2-u_{p-1, \sqrt[p]{c}-1}).
    \]
    Since each entry is Boolean, $Q$ is equivalent to $P$ over Boolean inputs, and thus, $Q$ is the block-exclusive polynomial equivalent to $f$.
\end{proof}
\begin{proposition}
\label{prop: unique-almost}
    Any Boolean function $f: \{0,1\} \to \{0,1\}$ with inputs from the almost $p$-hot encoding setting has a unique representation as a block-exclusive polynomial.
\end{proposition}
\begin{proof}
    Due to \citep[Proposition 4]{kopparty2020notes}, we know that every Boolean function $f$ is represented by a multilinear polynomial. Moreover, from \Cref{lem: block-ex}, we know that the polynomial $P(\bm{u})$ representing $f(\vu)$ is block-exclusive for $\vu$ with the almost $p$-hot encoding.

    To show uniqueness, we replicate the argument from \citep[Lecture 3, Proposition 4]{kopparty2020notes}: Given two block-exclusive polynomials $P$ and $P'$ equivalent to $f$ with inputs from the almost $p$-hot encoding, we have $(P-P')(\bm{u}) \equiv 0$. Now, assume, for the sake of contradiction, that $P - P' \not\equiv 0$. Here, note that as $P-P'$ is not identically zero and we have a non-zero monomial, and since the inputs are from the almost $p$-hot encoding, we know that this monomial cannot contain any product of the form $\bm{u}_{j,k}\bm{u}_{j, k'}$. Let $S \subseteq [p] \times [\sqrt[p]{c}-1]$ be a minimal set of indices such that the monomial $\prod_{(j,k) \in S} \bm{u}_{j,k}$ appears in $P-P'$ with non-zero coefficient. Note that $\chi_S$ forms a valid input to $f$ as each block in $S$ can be assigned at most one non-zero entry. Then, since $(P-P')(\chi_S) \neq 0$ as every other monomial will get at least one variable that is assigned to $0$ for $\chi_S$ , we achieve a contradiction, and thus, $P - P'$ must be identically zero on inputs from the almost $p$-hot encoding.
\end{proof}

\begin{lemma}
\label{lem: eq-almost}
    The EQ problem in the almost $p$-hot encoding setting is represented by a block-exclusive polynomial of degree $2p$.
\end{lemma}
\begin{proof}
Each input pair $\bm{u}^1, \bm{u}^2$ to the EQ problem can be represented as $\bm{u}^i:= (\bm{u}^i_0, \ldots, \bm{u}^i_{p-1})$ for $i \in \{1,2\}$, where for each $0 < j < p$ such that we have
\[
\bm{u}^i_{j} := (\bm{u}^i_{j,0}, \ldots, \bm{u}^i_{j, \sqrt[p]{c}-2}) \in \{0,1\}^{\sqrt[p]{c}-1}.
\]
The following polynomial takes the inner product of each of these one-hot encodings:
\[
P^j(\bm{u}) := \sum_{k= 0}^{\sqrt[p]{c}-2} \bm{u}^1_{j,k}\cdot \bm{u}^2_{j,k} + (1 - \sum_{k=0}^{\sqrt[p]{c}-2} \bm{u}^1_{j,k})(1 - \sum_{k=0}^{\sqrt[p]{c}-2} \bm{u}^2_{j,k})
\]
for $0 < j < p$. Here, note that there can be only be at most $1$ in both $\bm{u}^1_j$ and $\bm{u}^2_j$, and thus, $P^j(\bm{u}) = 1$ iff the $j$th block agree.

Next, the following polynomial is equivalent to the Boolean function that solves the EQ problem:
\[
P(\bm{u}) := \prod_{j=0}^{p-1} P^j(\bm{u}),
\]
and we have $P(\bm{u}) = \bbone\{\bm{u}^1 \equiv \bm{u}^2\}$. Here, note that $P$ is multi-linear and has degree $2p$ as each $P^j$ is a degree-2 polynomial. Moreover, $P$ is block-exclusive as each $P^j$ is block-exclusive and we only multiply monomials from different blocks in $P$.

\end{proof}

\begin{proposition}
\label{prop: deg-almost}
    Let $P$ be the block-exclusive polynomial that solves the EQ problem in the $p$-hot encoding. Then, $\deg(P) \ge 2p$. 
\end{proposition}
\begin{proof}
    For the sake of contradiction, assume that there exists a block-exclusive polynomial $P$ that solves EQ in the $p$-hot encoding setting with degree $\le 2p-1$. Then, given an input $\vu:= (\vu_0, \ldots, \vu_{p-1})$ from the almost $p$-hot encoding, where each block $\vu_i$ corresponds to the truncated bit string from the one-hot encoding in $\{0,1\}^{\sqrt[p]{c}-1}$, we can convert this input to the $p$-hot encoding $\vv:= (\vv_0, \ldots, \vv_{p-1})$ as follows:
    \[
    \vv_i := \paren{\vu_{i, 0}, \ldots, \vu_{i, \sqrt[p]{c}-2}, 1 - \sum_{j=0}^{\sqrt[p]{c}-2}\vu_{i,j}}
    \]
    Then, the block-wise multilinear polynomial $Q(\vu) = P(\vv)$ solves the EQ problem in the almost one-hot encoding setting and has $\deg(Q) \le \deg(P) \le 2p-1$ which contradicts the combination of \Cref{prop: unique-almost} and \Cref{lem: eq-almost}.
\end{proof}

\begin{theorem}
\label{thm: mqar-phot}
     A data-independent \BaseConv\ model needs at least $\floor{\log(2p)}$-layers to solve $\Task$ for an input sequence $\bm{u}\in \{0,1\}^{3N\times d}$ in the {$p$-hot encoding} setting, where $d = p\cdot\sqrt[p]{c}$.
\end{theorem}
\begin{proof}
We know from \Cref{cor: eq-mqar-lower} that it suffices to show a lower bound for the EQ problem. Moreover, we know from \Cref{prop: deg-almost} that we cannot solve the EQ problem in the $p$-hot encoding setting with a block-exclusive polynomial of degree $\le 2p - 1$. Now, assume that there is a \BaseConv\ model with $L$ layers that exactly solves EQ in the $p$-hot encoding setting. Then, due to \Cref{lem: based-degree} and \Cref{prop: unique-almost}, this yields a block-exclusive polynomial $P(\bm{u})$ of degree at most $2^L$. Here, if $L < \floor{\log(2p)}$ which , then the resulting \BaseConv\ with $L$ layers results in a block-exclusive polynomial of degree $\le 2p-1$. This contradicts the above claim that we cannot have a block-exclusive polynomial of degree $< 2p$ that exactly represents EQ. Consequently, a data-independent \BaseConv\ model needs $\ge \floor{\log(2p)}$-layers to solve EQ.
\end{proof}
\newcommand{\bin}{\text{bin}}
\newcommand{\iMap}{\tbf{B}}
\newcommand{\hot}{\tbf{e}}
\newcommand{\bN}{b}
\newcommand{\ip}{\textsc{Input: }}
\newcommand{\op}{\textsc{Output: }}
\newcommand{\id}{\textsc{Internal Dimension: }}
\newcommand{\blockWidth}{B}
\newcommand{\vSize}{r}
\newcommand{\basedModelDim}{\modelDim}
\newcommand{\basedn}{n}
\newcommand{\baseconv}{\textsc{BaseConv}}
\newcommand{\primLength}{N'}
\newcommand{\primDim}{d'}
\newcommand{\qkv}{\tbf{Y}}
\newcommand{\aab}[1]{
   \mathrel{\overset{\uparrow}{\underset{\downarrow}{#1}}}
}

\newcommand{\bb}{\overline{b}}
\newcommand{\splitIdx}{s}
\newcommand{\hyenaInput}{\bm{y}}
\newcommand{\baseOutput}{\bm{z}}
\newcommand{\mb}{\bm{b}}
\newcommand{\combDim}{\overline{N}}
\newcommand{\posSize}{z}
\newcommand{\arbVal}{y}
\newcommand{\remStart}{r}
\newcommand{\remEnd}{t}
\newcommand{\primRepCol}{\texttt{repeat$\_$columns}}
\newcommand{\primRepMat}{\texttt{repeat$\_$matrix}}
\newcommand{\primCumSum}{\texttt{cumulative$\_$sum}}
\newcommand{\primSumAllCol}{\texttt{sum$\_$all$\_$columns}}
\newcommand{\primSumColBlo}{\texttt{sum$\_$column$\_$blocks}}
\newcommand{\primHot}{\texttt{one$\_$hot$\_$encoding}}
\newcommand{\primRem}{\texttt{remember}}
\newcommand{\shiftDown}{\texttt{shift-down}}
\newcommand{\shiftUp}{\texttt{shift-up}}
\newcommand{\arbitrary}{?}
\newcommand{\linxPad}{
\begin{pmatrix}
\bm{xW}\\
\bm{0}^{s\times\primDim}
\end{pmatrix}
}
\newcommand{\linvPad}{
\begin{pmatrix}
\bm{vW}\\
\bm{0}^t
\end{pmatrix}
}
\newcommand{\bcOut}{(\bm{xW}+\bm{b}_1)\odot(\bm{h}\ast\bm{x}+\bm{b}_2)}


\subsection{Upperbound on MQAR with sub-logarithmically many BaseConv layers}
\paragraph{Setup: }
For an input $\mQ, \mK, \mV \in \R^{\inputLength \times \modelDim}$, the MQAR problem is computing $(\mC\odot (\mQ\mK^{\top})) \times \mV$ where $\mC \in \R^{\inputLength \times \inputLength}$ is a lower triangular matrix with $1$s in all possible non-zero positions:
\[
\mC[i,j] \equiv
\begin{cases}
    1 &\text{if }j \leq i\\
    0 &\text{otherwise}.
\end{cases}
\]
Further, we define the following notation:
\[\bN = \lceil\log(\inputLength+1)\rceil,\]
\[\bb = \lceil\log(\modelDim+1)\rceil,\]
\[\combDim = \max(\inputLength,\modelDim).\] 

Finally, we define $\bin(i)\in \{0,1\}^{\bN}$ to be the binary representation of $1\leq i\leq N$ and $\bin(j)\in \{0,1\}^{\bb}$ to be the binary representation of $1\leq j\leq \modelDim$. All vectors are assumed to be in column form and all row and column indices will start from 1. 

We assume that:
\begin{enumerate}[label=(\roman*)] 

\item  Each row of $\mQ$ and $\mK$ use 1-hot encoding and 
\item Each query matches with at most one key. 
\end{enumerate}
We  show that BaseConv can compute the MQAR problem with $O(\log\log(\combDim))$ layers:
\begin{theorem}\label{MQAR:math}
    The MQAR problem with $1$-hot encoded tokens, at most one key match per query, and $\combDim\geq8$ can be solved with $\BaseConv(\inputLength,O(\log\log(\overline{\inputLength})),\modelDim,O(\overline{\inputLength}\log(\overline{\inputLength})),O(\overline{\inputLength}\log(\overline{\inputLength})))$.
\end{theorem}

\subsubsection{BaseConv Primitives}\label{prims}
In this section we show some basic primitives that will be helpful in proving Theorem \ref{MQAR:math}. 

We define $\bm{e}_i^{(j)}$ to be the the $i$th standard basis vector with a dimension of $j$ (remember, both are one based indexing).

We first define the primitives and then show how to implement them using BaseConv.

Note that if a convolution matrix, denoted with $\bm{h}$ are given as a single column, all columns of this matrix are identical as as defined. Specifically, if given $\bm{h}\in\R^\inputLength$ and $\bm{x}\in\R^{\inputLength\times\modelDim}$we have $\bm{h}\ast\bm{h}$ to denote $\bm{k}\ast\bm{x}$ where $\bm{k}[:,j]=\bm{h} \forall j\in[\modelDim]$.

    \begin{definition}
        $\primRepCol(\hyenaInput, r)$ \\$\ip \hyenaInput\in\R^{\primLength\times\primDim r}, r\in\Z^{+}$.\\ $\op \baseOutput\in\R^{\primLength\times\primDim r},$ where $\bm{z}$ has each of the first $\primDim$ columns of $\bm{y}$ repeated $r$ times. In other words,

        \[
            \bm{y} \equiv
            \begin{pmatrix}
                \aab{\bm{y}^{(1)}},\hdots,\aab{\bm{y}^{(\primDim)}}, \aab{\arbitrary},\hdots,\aab{\arbitrary}\\
            \end{pmatrix},
            \bm{z} \equiv
            \begin{pmatrix}
                \underbrace{\aab{\bm{y}^{(1)}},\aab{\hdots\bm{y}^{(1)}}}_{\text{r times}},\hdots,\underbrace{\aab{\bm{y}^{(\primDim)}}, \hdots, \aab{\bm{y}^{(\primDim)}}}_{\text{r times}}\\
            \end{pmatrix}.
        \]
    \end{definition}

    \begin{definition}
        \primRepMat($\hyenaInput, r)$ \\$\ip \hyenaInput\in\R^{\primLength r\times\primDim}, r\in\Z^{+}$ such that $\bm{y}[\primLength+1:,:]=\bm{0}^{\primLength(r-1)\times\primDim}$.\\ $\op \baseOutput\in\R^{\primLength r\times\primDim},$ where $\baseOutput$ is the first $\primLength$ rows of $\hyenaInput$ repeated $r$ times. In other words, 

        \[
                    \bm{y} \equiv
            \begin{pmatrix}
                \leftarrow\bm{y}[1:\primLength,:]\rightarrow\\
                \hline\\
                \bm{0}^{\primLength\times\primDim}\\
                \hline\\
                \vdots\\
                \hline\\
                \bm{0}^{\primLength\times\primDim}\\
            \end{pmatrix},
            \bm{z} \equiv
            \left.
            \begin{pmatrix}
                \leftarrow\bm{y}[1:\primLength,:]\rightarrow\\
                \hline\\
                \leftarrow\bm{y}[1:\primLength,:]\rightarrow\\
                \hline\\
                \vdots\\
                \hline\\
                \leftarrow\bm{y}[1:\primLength,:]\rightarrow\\
            \end{pmatrix}\right\}\text{$r$ times}
        \]
    
    \end{definition}
    
    \begin{definition}
        \primCumSum($\hyenaInput$)\\ $\ip \hyenaInput\in\R^{\primLength\times\primDim}$.\\ $\op \baseOutput\in\R^{\primLength\times\primDim}, $ where each row of $\baseOutput$ is the sum of all rows of $\hyenaInput$ with a smaller or equal index. In other words, 
    
            \[
            \bm{y} \equiv
            \begin{pmatrix}
                \leftarrow\bm{y}_1\rightarrow\\
                \hline\\
                \vdots\\
                \hline\\
                \leftarrow\bm{y}_i\rightarrow\\
                \hline\\
                \vdots\\
                \hline\\
                \leftarrow\bm{y}_{\primLength}\rightarrow\\
            \end{pmatrix},
            \bm{z} \equiv
            \begin{pmatrix}
                \leftarrow\bm{y}_1\rightarrow\\
                \hline\\
                \vdots\\
                \hline\\
                \leftarrow\sum_{j=1}^{i}\bm{y}_j\rightarrow\\
                \hline\\
                \vdots\\
                \hline\\
                \leftarrow\sum_{j=1}^{\primLength}\bm{y}_j\rightarrow\\
            \end{pmatrix}.
        \]
    \end{definition} 

    \begin{definition}
        \primSumAllCol($\hyenaInput$) \\$\ip \hyenaInput\in\R^{\primLength\times\primDim}$.\\ $\op \baseOutput\in\R^{\primLength\times\primDim}, \\$where the first column of $\bm{z}$ has the sum of all columns of $\bm{y}$ and the rest are all zeros. In other words, 

                \[
            \bm{y} \equiv
            \begin{pmatrix}
                \aab{\bm{y}^{(1)}}, \aab{\bm{y}^{(2)}},\hdots,\aab{\bm{y}^{(\primDim)}}\\
            \end{pmatrix},
            \bm{z} \equiv
            \begin{pmatrix}
                \aab{\sum_{j=1}^{\primDim}\bm{y}^{(j)}}, \aab{\bm{0}}, \hdots, \aab{\bm{0}}\\
            \end{pmatrix}.
        \]
    \end{definition}
    
    \begin{definition}
        \primSumColBlo($\hyenaInput, \blockWidth$) 
    \\$\ip \hyenaInput\in\R^{\primLength\times\primDim}, \blockWidth\in\Z$ such that $\blockWidth$ divides $\primDim$.\\ $\op \baseOutput\in\R^{\primLength\times\primDim},$ where of column blocks the first column block of $\baseOutput$ is the sum of all column blocks and the rest are zero. In other words,
    \[
    \baseOutput[:,1:\blockWidth] \equiv
        \sum_{j=0}^{\frac{\primDim}{\blockWidth}-1}\bm{y}[:,j\blockWidth+1:(j+1)\blockWidth]
    \]

    \[\baseOutput[:,j]=\bm{0}^{\primLength} \text{ for all }j>\blockWidth.\]
    \end{definition}
    Note that $\primSumAllCol(\bm{y})$ = $\primSumColBlo(\bm{y}, 1).$
    
    \begin{definition}
        \primHot($\bm{y}, \modelDim$)\\ $\ip \hyenaInput\in\R^{\primLength\times\primLength}$.\\ $\op \baseOutput\in\R^{\primLength\times\primLength},$ where the first $\lceil\log(\primLength)\rceil$ columns of each row of $\bm{y}$  represent a binary encoding $\overline{y}_i\in[1,\primLength]$ which is converted to a 1-hot encoding. In other words, 

    \[  \bm{y} \equiv
        \begin{pmatrix}
            \bin(\overline{\arbVal}_1)^{\top}, \leftarrow\arbitrary\rightarrow \\ 
            \hline \\ 
            \vdots \\
            \hline\\
            \bin(\overline{\arbVal}_i)^{\top}, \leftarrow\arbitrary\rightarrow\\
            \hline\\
            \vdots\\
            \hline\\
            \bin(\overline{\arbVal}_{\primLength})^{\top}, \leftarrow\arbitrary\rightarrow
        \end{pmatrix} 
        \qquad \qquad \bm{z} \equiv
        \begin{pmatrix}    
            \hot_{\overline{\arbVal}_{1}}^{\top} \\ 
            \hline \\ 
            \vdots \\
            \hline\\
            \hot_{\overline{\arbVal}_{i}}^{\top} \\
            \hline\\
            \vdots\\
            \hline\\
            \hot_{\overline{\arbVal}_{\primLength}}^{\top} 
        \end{pmatrix} 
    \] 
    \end{definition} 

    \begin{definition}
        $\primRem(\bm{y}, \remStart, \remEnd, f)$ \\$\ip \hyenaInput\in\R^{\primLength\times\primDim},r\in\Z,t\in\Z, f:\R^{t-r}\rightarrow\R^{t-r+s},\bm{v}_1\in\R^{r}, \bm{x}\in\R^{t-r}$, where $\bm{y}$ is defined as below.\\ $\op \baseOutput\in\R^{\primLength\times\primDim}$, which is defined as follows:
    \[
            \bm{y} \equiv
                    \begin{pmatrix}
                    \leftarrow\bm{v}_1\rightarrow\\
                    \hline\\
                     \leftarrow\bm{x}\rightarrow \\
                        \hline\\
                        \bm{0}^{s\times\primDim}\\
                        \hline\\
                        \leftarrow\bm{v}_2\rightarrow\\
                        \hline\\
                        \bm{0}\\
                        \hline\\
                        \vdots\\
                        \hline\\
                        \bm{0}
                    \end{pmatrix} 
                    \qquad\qquad \bm{z}  \equiv
                    \begin{pmatrix}    
                        \leftarrow\bm{v_1}\rightarrow\\
                        \hline\\ \\
                        \leftarrow f(\bm{x})\rightarrow \\
                        \\ 
                        \hline\\
                        \leftarrow\bm{v}_2\rightarrow\\
                        \hline\\
                        \bm{0}\\
                        \hline\\
                        \vdots\\
                        \hline\\
                        \bm{0}
                    \end{pmatrix} 
            \] 
    \end{definition} 

Recall, that $\shiftDown({\bm{y}}, {s})$ and $\shiftUp({\bm{y}}, {s})$ will shift the matrix $\bm{y}$ down or up by s rows, respectively. 
\begin{proposition}[~\citep{arora2023zoology}]
\label{prop: prim-shift}
For any $\bm{y} \in \R^{N \times d}$, there exist $\coyoteTuple{\inputLength}{1}{\headDim}{\inputLength}{\headDim}$ and $\coyoteTuple{\inputLength}{3}{\headDim}{\inputLength}{\headDim}$ that computes $\texttt{shift\_down}({\bm{y}}, {s})$ and $\texttt{shift\_up}({\bm{y}}, {s})$ for any $s \le N$.
\end{proposition} 
Now we will show how each primitive is implemented in terms of $\baseconv$ layers.

\begin{proposition}[The Repeat Columns primitive]\label{prim:repeatcols} For any $\bm{y}\in\R^{\primLength\times\primDim r}$ and any $r\in\Z^{+}$ there exists a $\coyoteTuple{\primLength}{1}{\primDim r}{\primLength}{\primDim r}$ that computes \primRepCol$(\bm{y}, r)$.
\end{proposition}

\begin{proof}
Define
\[\bm{z}\gets \baseconv (\bm{y}, \bm{W}, \bm{0}^{\primLength\times\primDim}, \bm{0}^{\primLength\times\primDim}, \bm{1}^{\primLength\times\primDim}),\]

where $\bm{W}\in\R^{\primDim\times\primDim}$ is defined as:
    \[
    \mW[i,j] =
    \begin{cases}
        1 &\text{if }\left\lceil\dfrac{i}{r}\right\rceil=j\\
        0 &\text{otherwise}.
    \end{cases}
    \]

Then note that the output of this layer is:

        \[
        \bm{z} =
        (\bm{yW}+\bm{0}^{\primLength\times\primDim})\odot(\bm{0}^{\primLength\times\primDim}\ast\bm{y}+\bm{1}^{\primLength\times\primDim})= \bm{yW}=
            \begin{pmatrix}
            \underbrace{\aab{\bm{y}^{(1)}},\aab{\hdots\bm{y}^{(1)}}}_{\text{r times}},\hdots,\underbrace{\aab{\bm{y}^{(\primDim)}}, \hdots, \aab{\bm{y}^{(\primDim)}}}_{\text{r times}}
            \end{pmatrix},
        \]
where the last equality follows from the definition of $\bm{W}$.
\end{proof}

\begin{proposition}[The Repeat Matrix primitive]\label{prim:repeatmat} For any $\bm{y}\in\R^{\primLength r\times\primDim}$ and any $r\in\Z^{+}$ there exists a $\coyoteTuple{\primLength r}{1}{\primDim}{\primLength r}{\primDim}$ that computes ${\tt repeat\_matrix}(\bm{y}, r)$.
\end{proposition}
\begin{proof}
Define
\[\bm{z}\gets \baseconv (\bm{y}, \bm{0}^{\primDim\times\primDim}, \bm{1}^{\primLength r\times\primDim}, \bm{h}, \bm{0}^{\primLength r\times\primDim}),\]

where $\bm{h}\in\R^{\primLength\times\primDim}$ is defined as:
        \[ \bm{h}(X) \equiv \sum_{j=0}^{r-1}\bm{x}^{\primLength j}.\]
The computation of the convolution will result in:
\begin{align*}
    \paren{\bm{h} \ast \bm{y}} &= \coeff((\bm{y}(X))\cdot(1+ X^{\primLength}+ \hdots+ X^{\primLength\times(r-1)})\\
    &= \coeff(\bm{y}(X)+\bm{y}(X)\cdot X^{\primLength}+ \hdots+\bm{y}(X)\cdot(X^{\primLength\times(r-1)}))\\
    &=\bm{y}+\shiftDown(\bm{y},\primLength)+\hdots+\shiftDown(\bm{y},r\primLength).
\end{align*}

In the above, the final equality follows from \Cref{prop: prim-shift}. The output of this layer will compute:
        \[
        \bm{z} = 
        (\bm{y}\cdot\bm{0}^{\primDim\times\primDim}+\bm{1}^{\primLength\times\primDim})\odot(\bm{h}\ast\bm{y}+\bm{0}^{\primLength\times\primDim})=\bm{h}\ast\bm{y}=
            \begin{pmatrix}
                \bm{y}[1:\primLength,:]\\
                \hline
                \vdots\\
                \hline\\
                \bm{y}[1:\primLength,:]\\
            \end{pmatrix}.
        \]

\end{proof}
In the above, the last equality follows from \Cref{prop: prim-shift} and the fact that $\bm{y}[\primLength+1:, :]=\bm{0}^{((\primLength-1)\times r)\times\primDim)}$

\begin{proposition}[The Cumulative Sum primitive]\label{prim:cumulativesum} For any $\bm{y}\in\R^{\primLength\times\primDim}$  there exists a $\coyoteTuple{\primLength}{1}{\primDim}{\primLength}{\primDim}$ that computes ${\tt cumulative\_sum}(\bm{y})$.
\end{proposition}
\begin{proof}
Define 
\[\bm{z}\gets \baseconv (\bm{y}, \bm{0}^{\primDim\times\primDim}, \bm{1}^{\primLength\times\primDim}, \bm{h}, \bm{0}^{\primLength\times\primDim}),\]

where $\bm{h}\in\R^{\primLength\times\primDim}$ is defined as:
        \[ \bm{h}(X) \equiv \sum_{j=0}^{\primLength-1}\bm{x}^{j}.\]
        
The computation of the convolution will result in:
\begin{align*}
    \paren{\bm{h} \ast \bm{y}} &= \coeff((\bm{y}(X))\cdot(1+ X+ \hdots+ X^{\primLength-1})\\
    &= \coeff(\bm{y}(X)+\bm{y}(X)\cdot X + \hdots+\bm{y}(X)\cdot X^{\primLength-1})\\
    &=\bm{y}+\shiftDown(\bm{y},1)+\hdots+\shiftDown(\bm{y},\primLength-1).
\end{align*}
     
The output of this layer is:
\[
    \begin{pmatrix}
        \leftarrow\bm{y}_1\rightarrow\\
        \hline\\
        \vdots\\
        \hline\\
        \leftarrow\sum_{j=1}^{i}\bm{y}_j\rightarrow\\
        \hline\\
        \vdots\\
        \hline\\
        \leftarrow\sum_{j=1}^{N}\bm{y}_j\rightarrow\\
    \end{pmatrix}.
\]
\end{proof}

\begin{proposition}[The Sum All Columns primitive]\label{prim:sumall} For any $\bm{y}\in\R^{\primLength\times\primDim}$ there exists a $\coyoteTuple{\primLength}{1}{\primDim}{\primLength}{\primDim}$ that computes ${\tt sum\_all\_columns}(\bm{y})$.
\end{proposition}
\begin{proof}
Define
\[\bm{z}\gets \baseconv (\bm{y}, \bm{W}, \bm{0}^{\primLength\times\primDim}, \bm{0}^{\primLength\times\primDim}, \bm{1}^{\primLength\times\primDim}),\]

where $\bm{W}\in\R^{\primDim\times\primDim}$ is defined as:
        \[
    \mW[i,j] \equiv
    \begin{cases}
        1 &\text{if }j=1\\
        0 &\text{otherwise}.
    \end{cases}
    \] 
The output of this layer will be:
\[            
(\bm{yW}+\bm{0}^{\primLength\times\primDim})\odot(\bm{0}^{\primLength\times\primDim}\ast\bm{y}+\bm{1}^{\primLength\times\primDim})= \bm{yW}=
    \begin{pmatrix}
    \aab{\sum_{j=1}^{\primDim}\bm{y}^{(j)}}, \aab{\bm{0}}, \hdots, \aab{\bm{0}}\\
    \end{pmatrix},
\]
where the last equality follows from the definition of $\bm{W}$. 
\end{proof}

\begin{proposition}[The Sum Block Columns primitive]\label{prim:sumblocks} For any $\bm{y}\in\R^{\primLength\times\primDim}$ and $B$ that divides $\primDim$ there exists a $\coyoteTuple{\primLength}{1}{\primDim}{\primLength}{\primDim}$ that computes $\primSumColBlo(\bm{y},B)$.
\end{proposition}

\begin{proof}
Define
\[\bm{z}\gets \baseconv (\bm{y}, \bm{W}, \bm{0}^{\primLength\times\primDim}, \bm{0}^{\primLength\times\primDim}, \bm{1}^{\primLength\times\primDim}),\]

where $\bm{W}\in\R^{\primDim\times\primDim}$ is defined as:
        \[
    \mW[i,j] \equiv
    \begin{cases}
        1 &\text{if }j\le B \text{ and } j-B\lfloor\frac{i}{B}\rfloor\\
        0 &\text{otherwise}.
    \end{cases}
    \]
The output of this layer will be:

\begin{align*}
(\bm{yW}+\bm{0}^{\primLength\times\primDim})\odot(\bm{0}^{\primLength\times\primDim}\ast\bm{y}+\bm{1}^{\primLength\times\primDim})&= \bm{yW}\\
 &=\begin{pmatrix}
            \aab{\sum_{j=0}^{\frac{\primDim}{\blockWidth}-1}\bm{y}[:,j\blockWidth+1]},\hdots, \aab{\sum_{j=0}^{\frac{\primDim}{\blockWidth}-1}\bm{y}[:,j\blockWidth+B-1]}, \underbrace{\bm{0},\hdots,\bm{0}}_{\primDim-\lfloor\frac{\primDim}{D}\rfloor \text{times}}
\end{pmatrix},
\end{align*}
where the last equality follows from definition of $\bm{W}$. 

\end{proof}

\begin{proposition}[The 1-hot primitive] For any $\bm{y}\in\R^{\primLength\times\primLength}$ there exists a $\coyoteTuple{\primLength}{\lceil\log\log(\primLength)\rceil+O(1)}{\primLength}{2\primLength\lceil\log\primLength\rceil}{\primLength}$ that computes $\primHot(\bm{y})$.
\end{proposition}
\begin{proof}
We first give a sketch of the proof. Each row has a binary representation of a number which we want to convert to it's 1-hot encoding. In order to do this, we need to know which position in the 1-hot encoded vector needs to be 1. We need to extract information from each bit, which details which subset of positions the 1-hot encoded vector could potentially have as 1. Concretely, if the least significant bit is 0, only the even position can be represented, if that same bit is 1, only the odd positions can be represented. This pattern continues for all bits in the binary number. Each bit in the binary representation gets its own row. Finally, we take the bit wise ANDs of each row of the same binary representation to get the resulting 1-hot encoded vector. Next, we present the details. 

First compute $\bm{z}_1\in\R^{2\primLength\lceil\log\primLength\rceil\times\primLength}$ defined as:
\[\bm{z}_1 \gets \primRepMat(\bm{y}, 2\lceil\log(\primLength)\rceil).\]

Define
\[\bm{z}_2 \gets \texttt{BaseConv}(\bm{z}_1, \bm{I}^{\primLength\times\primLength}, \bm{0}^{2\primLength\lceil\log\primLength\rceil\times\primLength}, \bm{0}^{2\primLength\lceil\log\primLength\rceil\times\primLength}, \bm{b}_2^2),\]
where $\bm{b}_2\in\R^{2\primLength\lceil\log\primLength\rceil\times\primLength}$ is defined as for $1\leq i \leq 2\primLength\lceil\log\primLength\rceil$ and $1\leq j \leq \primLength$:
\[\bm{b}_2^2[i,j] \equiv 
\begin{cases}
    1 &\text{if } \paren{\left\lfloor\frac{i}{\primLength}\right\rfloor \mod{\lceil\log(\primLength)\rceil}}+1 = j\\
    0 &\text{otherwise.}
\end{cases}\]

Note that $\bm{z}_1$ has $2\log(\primLength)$ copies of the binary representations in the first column block. When we zero them out to get $\bm{z}_2$, the $(i\mod\log(\primLength))$th matrix stores the value of the $i$th bit with all others being set to zero. \\
Compute $\bm{z}_3\in\R^{2\primLength\lceil\log\primLength\rceil\times\primLength}$ by storing the sum all columns in the first column. Being that each row only has a single non-zero entry, this is equivalent to moving every one of these non-zero entries to the first column. Define 
\[\bm{z}_3 \gets\primSumAllCol(\bm{z}_2).\]
Compute $\bm{z}_4\in\R^{2\primLength\lceil\log\primLength\rceil\times\primLength}$ by copying the first column to all other columns. Define
\[\bm{z}_4 \gets \primRepCol(\bm{z}_3, \primLength).\]
Then, we will take a Hadamard product of $\bm{z}_4$ with a binary representation matrix. Define
\[\bm{z}_5 \gets \texttt{BaseConv($\bm{z}_4, \bm{0}^{\primLength\times\primLength}, \bm{b}_{1}^{5}, 1, \bm{0}^{2\primLength\lceil\log\primLength\rceil\times\primLength}$)}.\]
In the above, $\bm{b}_{1}^{5}$ has the positions where a binary number with that has that bit set to a 1 could fall(and 0 if it in the bottom half of the matrix. We will define it in blocks where $1\leq i\leq2\lceil\log\primLength\rceil, 1\leq k\leq\primLength, 1\leq j\leq\primLength$

\[\bm{b}_{1}^{5}[(i,k),j]\equiv \begin{cases}
    1 & \text{if } j\mod2^{i}\geq2^{i-1} \text{ and } i\leq\lceil\log\primLength\rceil\\
    1 & \text{if } j\mod2^{i}<2^{i-1} \text{ and } i>\lceil\log\primLength\rceil\\
    
    0 & \text{otherwise}.
\end{cases}\]

Next, we will combine the two representations:
\[\bm{z}_6 \gets \texttt{BaseConv($\bm{z}_5, \bm{0}^{\primLength\times\primLength}, \bm{b}_1^6, \bm{h}^6, \bm{0}^{2\primLength\lceil\log\primLength\rceil\times\primLength}$)},\]

where $\bm{b}_1^6,\bm{h}^6$ are defined as:

\[\bm{b}_1^6=\begin{pmatrix}
    \bm{0}^{(\primLength\lceil\log\primLength\rceil)\times\primLength}\\\hline\\\bm{1}^{(\primLength\lceil\log\primLength\rceil)\times\primLength}
\end{pmatrix},\bm{h}^6=\begin{pmatrix}
    \bm{e}_1^{(\primLength\lceil\log\primLength\rceil)}\\
    \hline\\
    \bm{e}_1^{(\primLength\lceil\log\primLength\rceil)}
\end{pmatrix}.\]

We now specify the results of this kernel:

\begin{align*}
    \paren{\bm{h}^6 \ast \bm{z}_5} 
    &= \coeff\paren{\paren{1+X^{\primLength\lceil\log\primLength\rceil}}\cdot \bm{z_5}(x)}\\
    &= \coeff\paren{\bm{z_5}(x) + \bm{z_5}(x)\cdot X^{\primLength\lceil\log\primLength\rceil}}\\
    &= \bm{z_5} + \shiftDown(\bm{z_5}, X^{\primLength\lceil\log\primLength\rceil})
\end{align*}

By combining the possible positions based on a $0$ or $1$ being present in the binary number, $\bm{z}_6$ now stores the expanded binary representation in each row block in the bottom half of the matrix. We move this to the top half as shown below:
\[\bm{z}_7 \gets \shiftUp(\bm{z}_6, \primLength\lceil\log\primLength\rceil).\]

Finally we do a bit wise multiplication between corresponding rows. For $0\leq m<\lceil\log\log\primLength\rceil$ such that on the $m$'th iteration the following function is performed:

\[\bm{z}'_{8,m}\gets\baseconv(\bm{z}_{8,m-1}, \bm{I}^{2\primLength\lceil\log\primLength\rceil\times\primLength}, \bm{0}^{2\primLength\lceil\log\primLength\rceil\times\primLength}, \bm{h}_m,\bm{0}^{2\primLength\lceil\log\primLength\rceil\times\primLength})\] 
\[\bm{z}_{8,m}\gets\shiftUp(\bm{z}'_{8,m},\primLength2^{m})\]

Where $\bm{z}_{8,-1}\equiv\bm{z}_{7}$ and $\bm{h}_m$ is defined below:
\[\bm{h}_m = \begin{pmatrix}
    \bm{e}_1^{(2^m)}\\
    \hline\\
    \bm{e}_1^{(2^{m}-2\primLength\lceil\log\primLength\rceil)}
\end{pmatrix}.\]

The computation of this convolution will result in:
\begin{align*}
    \paren{\bm{h}_m \ast \bm{z}_{8,m}} 
    &= \coeff\paren{\paren{1+X^{(2^m)}}\cdot \bm{z}_{8,m}}(x)\\
    &= \coeff\paren{\bm{z}_{8,m}(x) + \bm{z}_{8,m}(x)\cdot X^{(2^m)}}\\
    &=\bm{z}_{8,m} + \shiftDown(\bm{z}_{8,m}, X^{(2^m)})
\end{align*}

$\bm{z}_7$ holds the possible positions of of the "$1$" in the one hot vector for each bit in the binary representation. This step takes a bit-wise and of the rows corresponding to the same binary representation so that we are left with the 1-hot encoding of the original binary representation. The idea is that for each bit in each binary representation of each value, there is a row in $\bm{z}_7$ that represents the possible positions, in the form of a bitmap, that this binary number can encode. When we bitwise AND all of these rows together, we are left with the position that satisfies all the constraints and is therefore the index that these binary numbers encoded. 

\end{proof}

\begin{proposition}[The Remembering Primitive]
\label{prop: prim-remember}
For any $\bm{x} \in \R^{\basedn \times \primDim}, \bm{v}_1 \in \R^{\vSize \times \primDim}, \bm{v}_2\in\R^{m-\vSize}$ where $n= \remEnd-\remStart$ contained in some $\bm{y} \in \R^{\primLength \times \primDim}$ such that $\bm{v}_1$ is in the first $\vSize$ rows, $\bm{x}$ is in the next $\basedn$ rows, 0s fill up the next $s$ rows, and $\bm{v}_2$ are in the next $m-\vSize$ rows, for some $3\basedn+3m+2s+2t \le \primLength$ so that for $\bm{h} \in \R^{n \times d}$ and $\bm{W} \in \R^{\primDim\times \primDim}$ with $\bm{x} \ast \bm{h} \in \R^{(\basedn+s) \times \primDim}$ and $\bm{v} \ast \bm{h} \in \R^{(m+t) \times \primDim}$, where $\bm{v}\in\R^{m\times\modelDim}$ is defined as $\bm{v}_2 +$\texttt{shift\_down($\bm{v}_1, m-\vSize$)}, there exists a $\coyoteTuple{\primLength}{8}{\primDim}{\primLength}{\primDim}$ that computes $\texttt{remember}(\bm{y}, \remStart, \remEnd, f)$, where $f$ can be implemented in 1 layer of $\baseconv$ through the parameters $\bm{W}\in\R^{\primDim\times\primDim}, \bm{h}\in\R^{\primLength\times\primDim}, \bm{b}_1\in\R^{\primLength\times\primDim}, \bm{b}_2\in\R^{\primLength\times\primDim}$ as defined below:
\[f(\bm{u})=\paren{\begin{pmatrix}\bm{uW}\\\bm{0}^{s\times\primDim}\end{pmatrix}+\begin{pmatrix}\bm{b}_1\\\bm{1}^{s\times\primDim}\end{pmatrix}} \odot \paren{\bm{u}\ast\bm{h}+\begin{pmatrix}\bm{b}_2\\\bm{0}^{s\times\primDim}\end{pmatrix}}\]
\end{proposition}
\begin{proof} First, we will convert the $\bm{y}$ so that it stores $\bm{v}$ in consecutive rows to get $\bm{z}_1$.
Recall 

\[            \bm{y} =
                    \begin{pmatrix}
                    \leftarrow\bm{v}_1\rightarrow\\
                    \hline\\
                     \leftarrow\bm{x}\rightarrow \\
                        \hline\\
                        \bm{0}^{s\times\primDim}\\
                        \hline\\
                        \leftarrow\bm{v}_2\rightarrow\\
                        \hline\\
                        \bm{0}^{(\primLength-m-s-n)\times\primDim}
                    \end{pmatrix}. \]

We compute $\bm{z}_1$ as:
\[\bm{z}_{1}\gets \baseconv (\bm{y}, \bm{0}^{\primDim\times\primDim}, \bm{b}_{1}^{1}, \bm{h}^{1}, \bm{0}^{\primLength\times\primDim}),\]

where the kernels $\bm{h}^{1}\in\R^{\primLength\times\primDim}$ and $\bm{b}_{1}^{1}\in\R^{\primLength\times\primDim}$ are given by:
\[
    \bm{h}^{1} 
    \gets
    \begin{pmatrix}
        \bm{e}_{1}^{(n)} \\
        \hline\\
        \bm{0}^{m}\\
        \hline\\
        \bm{0}^{s}\\
        \hline\\
        \bm{e}_{1}^{(n)}\\
        \hline\\
        \bm{0}^{n}\\
        \hline\\
        \cdots\\
        \hline\\
        \bm{0}^{n}\\
    \end{pmatrix},
    \bm{b}_{1}^{1} 
    \gets
    \begin{pmatrix}
        \bm{0}^{\vSize\times\primDim}\\
        \hline\\
        \bm{1}^{n\times\primDim}\\
        \hline\\
        \bm{0}^{s\times\primDim}\\
        \hline\\
        \bm{1}^{(m-\vSize)\times\primDim}\\
        \hline\\
        \bm{1}^{\vSize\times\primDim}\\
        \hline\\
        \bm{0}^{n\times\primDim}\\
        \hline\\
        \bm{0}^{s\times\primDim}\\
        \hline\\
        \bm{0}^{(m-\vSize)\times\primDim}\\
        \hline\\
        \hdots\\
        \hline\\
        \bm{0}^{n\times\primDim}\\   
    \end{pmatrix}
    .
    \]

We now specify the result of this kernel:
\begin{align*}
    \paren{\bm{h}^1 \ast \bm{y}} &= \mathrm{coeff}\paren{(1+X^{n+s+m}) \cdot \paren{\bm{v}_1(X) + \bm{x}(X) \cdot X^{\vSize} + \bm{v}_2(X) \cdot X^{n+s+\vSize}}}\\
    &= \mathrm{coeff}\paren{\bm{v}_1(X) + \bm{x}(X) \cdot X^{\vSize} + \bm{v}_2 \cdot X^{n+s+\vSize} + \bm{v}_1(X) \cdot X^{n+s+m} + \bm{x}(X) \cdot X^{n+s+m+\vSize}+ \bm{v}_2 \cdot X^{2n+2s+m+\vSize}}\\
    &= \bm{v}_1+\shiftDown(\bm{x},r)+\shiftDown(\bm{v}_2,n+s+r)+\shiftDown(\bm{v}_1,n+s+m)\\& +\shiftDown(\bm{x},n+s+m+r)+\shiftDown(\bm{v}_2,2n+2s+m+r).
\end{align*}
With this we have:
    \[
    \bm{z}_1 = \paren{\bm{y}\cdot\bm{0}^{\primDim\times\primDim}+\bm{b}_1^1} \odot \paren{\bm{h}^1 \ast \bm{y}+\bm{0}^{\primLength\times\primDim}} = \bm{b}_1^1 \odot \paren{\bm{h}^1 \ast \bm{y}}
     = 
    \begin{pmatrix}
        \bm{0}^{\vSize\times\primDim}\\
        \hline\\
        \bm{1}^{n\times\primDim}\\
        \hline\\
        \bm{0}^{s\times\primDim}\\
        \hline\\
        \bm{1}^{(m-\vSize)\times\primDim}\\
        \hline\\
        \bm{1}^{\vSize\times\primDim}\\
        \hline\\
        \bm{0}^{n\times\primDim}\\
        \hline\\
        \bm{0}^{s\times\primDim}\\
        \hline\\
        \bm{0}^{(m-\vSize)\times\primDim}\\
        \hline\\
        \hdots\\
        \hline\\
        \bm{0}^{n\times\primDim}\\
        
    \end{pmatrix} \odot 
    \begin{pmatrix}
        \bm{v}_1\\
        \hline\\
        \bm{x} \\
        \hline\\
        \bm{0}^{s\times\primDim}\\
        \hline\\
        \bm{v}_2\\
        \hline\\
        \bm{v}_1\\
        \hline\\
        \bm{x} \\
        \hline\\
        \bm{0}^{s\times\primDim}\\
        \hline\\
        \bm{v}_2\\
        \hline\\
        \hdots\\
        \hline\\
        \bm{0}^{n\times\primDim}\\
    \end{pmatrix} =  
    \begin{pmatrix}
        \bm{0}^{\vSize\times\primDim}\\
        \hline\\
        \bm{x} \\
        \hline\\
        \bm{0}^{s\times\primDim}\\
        \hline\\
        \bm{v}_2\\
        \hline\\
        \bm{v}_1\\
        \hline\\
        \bm{0}^{n\times\primDim}\\
        \hline\\
        \bm{0}^{s\times\primDim}\\
        \hline\\
        \bm{0}^{(m-\vSize)\times\primDim}\\
        \hline\\
        \hdots\\
        \hline\\
        \bm{0}^{n\times\primDim}\\
    \end{pmatrix}=  
    \begin{pmatrix}
        \bm{0}^{\vSize\times\primDim}\\
        \hline\\
        \bm{x} \\
        \hline\\
        \bm{0}^{s\times\primDim}\\
        \hline\\\\
        \bm{v}\\\\
        \hline\\
        \bm{0}^{n\times\primDim}\\
        \hline\\
        \bm{0}^{s\times\primDim}\\
        \hline\\
        \bm{0}^{(m-\vSize)\times\primDim}\\
        \hline\\
        \hdots\\
        \hline\\
        \bm{0}^{n\times\primDim}\\
    \end{pmatrix}.
    \]

Next compute $\bm{z}_2$ : 
\[\bm{z}_2 \gets \shiftUp(\bm{z}_1, \vSize),\] as seen in \cite{arora2023zoology} Proposition \ref{prop: prim-shift}.

At this point $\bm{z}_2$ looks like:
\[\begin{pmatrix}
    \bm{x}\\
    \hline\\
    \bm{0}^{s\times\primDim}\\
    \hline\\
    \bm{v}\\
    \hline\\
    \bm{0}^{n\times\primDim}\\
    \hline
    \vdots\\
    \hline\\
    \bm{0}^{n\times\primDim}
\end{pmatrix}
\]

Next, we will apply $f$ to $\bm{z}_2$ to get $f(\bm{x})$ but also retain a unchanged version of $\bm{v}$. Define 

\[\bm{z}_{3}\gets \baseconv (\bm{z}_2, \bm{W}, \bm{b}_{1}^{3}, \bm{h}^{3}, \bm{b}_{2}^{3}),\]

with the kernels $\bm{h}^{3}$, $\bm{b}_{1}^{3}$, and $\bm{b}_{2}^{3}$ for this layer are given by:

\[
    \bm{h}^{3} 
    \gets
    \begin{pmatrix}
        \bm{h}\\
        \hline\\
        \bm{0}^{s\times\primDim} \\
        \hline\\
        \bm{0}^{m\times\primDim}\\
        \hline\\
        \bm{0}^{t\times\primDim}\\
        \hline\\
        \bm{e}_1^{(n)}\\
        \hline\\
        \bm{0}^{s\times\primDim}\\
        \hline\\
        \bm{0}^{n\times\primDim}\\
    \end{pmatrix},
    \bm{b}_{1}^{3} 
    \gets
    \begin{pmatrix}
        \bm{b}_1\\
        \hline\\
        \bm{1}^{s\times\primDim}\\
        \hline\\
        \bm{0}^{m\times\primDim}\\
        \hline\\
        \bm{0}^{t\times\primDim}\\
        \hline\\
        \bm{0}^{n\times\primDim}\\
        \hline\\
        \bm{0}^{s\times\primDim}\\
        \hline\\
        \bm{1}^{m\times\primDim}\\   
    \end{pmatrix},
        \bm{b}_{2}^{3} 
    \gets
    \begin{pmatrix}
        \bm{b}_2\\
        \hline\\
        \bm{0}^{s\times\primDim}\\
        \hline\\
        \bm{0}^{m\times\primDim}\\
        \hline\\
        \bm{0}^{t\times\primDim}\\
        \hline\\
        \bm{0}^{n\times\primDim}\\
        \hline\\
        \bm{0}^{s\times\primDim}\\
        \hline\\
        \bm{0}^{m\times\primDim}\\   
    \end{pmatrix}.
    \]
Remember that $\bm{W}, \bm{h}, \bm{b}_1, \bm{b}_2$ come from the definition of $f$.\\
We specify the result of this kernel as: 

\begin{align*}
    \paren{\bm{h}^3 \ast \bm{z}_2} &= \mathrm{coeff}\paren{(\bm{h}(X)+X^{n+s+m+t}) \cdot \paren{\bm{x}(X) + \bm{v}(X) \cdot X^{n+s}}}\\
    &= \mathrm{coeff}\paren{\bm{h}\cdot\bm{x}(X)+\bm{h}\cdot\bm{v}(X)\cdot X^{n+s}
    + \bm{x}(X)\cdot X^{n+s+m+t} + \bm{v}(X) \cdot X^{2n+2s+m+t}}\\
    &= \bm{h}\ast\bm{x}+\shiftDown(\bm{h}\ast\bm{v},n+s)\\
    &+\shiftDown(\bm{x},n+s+m+t)+\shiftDown(\bm{v},2n+2s+m+t).
\end{align*}

With this we have
    \begin{align*}
    \bm{z}_3 = \paren{\bm{z}_2\bm{W}+\bm{b}_1^3} \odot \paren{\bm{h}^3 \ast \bm{z}_2 + \bm{b}_2^3}
     &= 
    \paren{\begin{pmatrix}
        \bm{xW}\\
        \hline\\
        \bm{0}^{s\times\primDim}\\
        \hline\\
        \bm{vW}\\
        \hline\\
        \bm{0}^{t\times\primDim}\\
        \hline\\
        \bm{0}^{n\times\primDim}\\
        \hline\\
        \bm{0}^{s\times\primDim}\\
        \hline\\
        \bm{0}^{m\times\primDim}\\
    \end{pmatrix} +
        \begin{pmatrix}
        \bm{b}_1\\
        \hline\\
        \bm{1}^{s\times\primDim}\\
        \hline\\
        \bm{0}^{m\times\primDim}\\
        \hline\\
        \bm{0}^{t\times\primDim}\\
        \hline\\
        \bm{0}^{n\times\primDim}\\
        \hline\\
        \bm{0}^{s\times\primDim}\\
        \hline\\
        \bm{1}^{m\times\primDim}\\   
    \end{pmatrix}} \odot
    \paren{
    \begin{pmatrix}
        \bm{h}\ast\bm{x}\\ \\ \\
        \hline\\
        \bm{h}\ast\bm{v}\\ \\ \\ 
        \hline\\
        \bm{x}\\
        \hline\\
        \bm{0}^{s\times\primDim}\\
        \hline\\
        \bm{v}\\
    \end{pmatrix} +  
     \begin{pmatrix}
        \bm{b}_2\\
        \hline\\
        \bm{0}^{s\times\primDim}\\
        \hline\\
        \bm{0}^{m\times\primDim}\\
        \hline\\
        \bm{0}^{t\times\primDim}\\
        \hline\\
        \bm{0}^{n\times\primDim}\\
        \hline\\
        \bm{0}^{s\times\primDim}\\
        \hline\\
        \bm{0}^{m\times\primDim}\\   
    \end{pmatrix}}\\ &=
    \begin{pmatrix}
        \paren{\linxPad+\begin{pmatrix}\bm{b}_1\\\bm{1}^{s\times\primDim}\end{pmatrix}}\odot\paren{\bm{h}\ast\bm{x}+\begin{pmatrix}
        \bm{b}_2\\\bm{0}^{s\times\primDim}
        \end{pmatrix}}\\ \\ \\ 
        \hline\\
        \linvPad\odot\bm{h}\ast\bm{v}\\ \\ \\ 
        \hline\\
        \bm{0}^{n\times\primDim}\\
        \hline\\
        \bm{0}^{s\times\primDim}\\
        \hline\\
        \bm{v}\\
    \end{pmatrix}.
    \end{align*}

Note that $\paren{\linxPad+\begin{pmatrix}
    \bm{b}_1\\\bm{0}^{s\times\primDim} \end{pmatrix}}\odot\paren{\bm{h}\ast\bm{x}+\begin{pmatrix}
        \bm{b}_2\\\bm{0}^{s\times\primDim}
        \end{pmatrix}}$ is $f(\bm{x})$ as defined.  
This next step will mask out duplicate and unnecessary $\bm{x}$ and $\bm{v}$ values. Define

\[\bm{z}_{4}\gets \baseconv (\bm{z}_3, \bm{I}^{\primLength\times\primDim}, \bm{0}^{\primLength\times\primDim}, \bm{0}^{\primLength\times\primDim}, \bm{b}_{2}^{4}),\]

where the kernel $\bm{b}_2^{4}\in\R^{\primLength\times\primDim}$ for this layer is given by:

\[
    \bm{b}_2^{4} 
    \gets
    \begin{pmatrix}
        \bm{1}^{(n+s)\times\primDim}\\
        \hline\\
        \bm{0}^{(m+t)\times\primDim}\\
        \hline\\
        \bm{0}^{(n+s)\times\primDim}\\
        \hline\\
        \bm{1}^{(m)\times\primDim}\\
    \end{pmatrix}.
    \]
We will specify the output of this layer:
    \[
    \bm{z}_4 = \bm{z}_3 \odot \bm{b}_2^4
     = 
    \begin{pmatrix}
        f(\bm{x})\\
        \hline\\
        \linvPad\odot\bm{h}\ast\bm{v}\\
        \hline\\
        \bm{0}^{(n+s)\times\primDim}\\
        \hline\\
        \bm{v}\\
    \end{pmatrix} \odot 
    \begin{pmatrix}\\
        \bm{1}^{(n+s)\times\primDim}\\
        \hline\\
        \bm{0}^{(m+t)\times\primDim}\\\\
        \hline\\
        \bm{0}^{(n+s)\times\primDim}\\
        \hline\\
        \bm{1}^{m\times\primDim}\\
    \end{pmatrix} =  
    \begin{pmatrix}
        f(\bm{x})\\
        \hline\\
        \bm{0}^{(m+t)\times\primDim}\\\\
        \hline\\
        \bm{0}^{(n+s)\times\primDim}\\
        \hline\\
        \bm{v}\\
    \end{pmatrix}.
    \]

In the next step we will reorder the information such that $f(\bm{x})$ and $\bm{v}$ are contained in contiguous rows by copying it with a convolution. Define 

\[\bm{z}_{5}\gets \baseconv (\bm{z}_4, \bm{0}^{\primDim\times\primDim}, \bm{b}_1^5, \bm{h}^{5}, \bm{0}^{\primLength\times\primDim}),\]

wehre the kernels $\bm{h}^{5}\in\R^{\primLength\times\primDim}$ and $\bm{b}_1^5\in\R^{\primLength\times\primDim}$ for this layer is given by:
\[
    \bm{h}^{5} 
    \gets
    \begin{pmatrix}
        \bm{e}_{1}^{(n)}\\
        \hline\\
        \bm{0}^{s\times\primDim} \\
        \hline\\
        \bm{0}^{m\times\primDim}\\
        \hline\\
        \bm{0}^{t\times\primDim}\\
        \hline\\
        \bm{e}_{1}^{(n)}\\
        \hline\\
        \bm{0}^{s\times\primDim}\\
        \hline\\
        \bm{0}^{n\times\primDim}\\
    \end{pmatrix}, 
    \bm{b}_1^5
    \gets
    \begin{pmatrix}\\
        \bm{0}^{(n+s)\times\primDim}\\
        \hline\\
        \bm{0}^{(m+t)\times\primDim}\\
        \hline\\
        \bm{1}^{(n+s)\times\primDim}\\\\
        \hline\\
        \bm{1}^{m\times\primDim}\\
        \hline\\
        \bm{0}^{(m+t)\times\primDim}\\
        \hline\\
        \bm{0}^{(n+s)\times\primDim}\\
        \hline\\
        \bm{0}^{m\times\primDim}\\
    \end{pmatrix}
    .
    \]

\begin{align*}
    \paren{\bm{h}^5\ast\bm{z}_4}&=\mathrm{coeff}\paren{(1+X^{n+s+m+t})\cdot \paren{f(\bm{x})(X) + \bm{v}(X)\cdot X^{2n+2s+m+t}}}\\
    &= \mathrm{coeff}(f(\bm{x})(X) + f(\bm{x})(X)\cdot X^{n+s+m+t} + \bm{v}(X)\cdot X^{2n+2s+m+t} + \bm{v}(X)\cdot X^{3n+3s+2m+2t}) \\
    &= f(\bm{x}) + \shiftDown(f(\bm{x}), n+s+m+t)+\shiftDown(\bm{v}, 2n+2s+m+t)\\
    &+\shiftDown(\bm{v}, 3n+3s+2m+2t).
\end{align*}

    \[
    \bm{z}_5 = \bm{b}_1^5 \odot \paren{\bm{h}^5 \ast \bm{z}_4}
     =
    \begin{pmatrix}\\
        \bm{0}^{(n+s)\times\primDim}\\
        \hline\\
        \bm{0}^{(m+t)\times\primDim}\\
        \hline\\
        \bm{1}^{(n+s)\times\primDim}\\\\
        \hline\\
        \bm{1}^{m\times\primDim}\\
        \hline\\
        \bm{0}^{(m+t)\times\primDim}\\
        \hline\\
        \bm{0}^{(n+s)\times\primDim}\\
        \hline\\
        \bm{0}^{m\times\primDim}\\
    \end{pmatrix} \odot
    \begin{pmatrix}
        f(\bm{x})\\
        \hline\\
        \bm{0}^{(m+t)\times\primDim}\\
        \hline\\
        f(\bm{x})\\
        \hline\\
        \bm{v}\\
        \hline\\
        \bm{0}^{(m+t)\times\primDim}\\
        \hline\\
        \bm{0}^{(n+s)\times\primDim}\\
        \hline\\
        \bm{v}\\
    \end{pmatrix} = 
    \begin{pmatrix}\\
        \bm{0}^{(n+s)\times\primDim}\\
        \hline\\
        \bm{0}^{(m+t)\times\primDim}\\
        \hline\\
        f(\bm{x})\\
        \hline\\
        \bm{v}\\
        \hline\\
        \bm{0}^{(m+t)\times\primDim}\\
        \hline\\
        \bm{0}^{(n+s)\times\primDim}\\
        \hline\\
        \bm{0}^{m\times\primDim}\\
    \end{pmatrix}.
    \]

The next step we will duplicate the entries so we can position $\bm{v}_1$ and $\bm{v}_2$ in the same position relative to the enacted upon portion of the matrix as in the input. Define
\[\bm{z}_{6}\gets \baseconv (\bm{z}_5, \bm{0}^{\primDim\times\primDim}, \bm{b}_1^6, \bm{h}^{6}, \bm{0}^{\primLength\times\primDim}),\] 

where the kernels $\bm{h}^{6}$ and $\bm{b}_1^6$ for this layer is given by:
\[
    \bm{h}^{6} 
    \gets
    \begin{pmatrix}
        \bm{e}_1^{(2n+2s+2m+t)}\\
        \hline\\
        \bm{e}_1^{(\primLength-2n-2s-2m-t)} \\
    \end{pmatrix}, 
    \bm{b}_1^6
    \gets
\begin{pmatrix}
        \bm{0}^{(n+s)\times\primDim}\\
        \hline\\
        \bm{0}^{(m+t)\times\primDim}\\
        \hline\\
        \bm{0}^{(n+s)\times\primDim}\\\\
        \hline\\
        \bm{0}^{(m-\vSize)\times\primDim}\\
        \hline\\
        \bm{1}^{r\times\primDim}\\
        \hline\\
        \bm{1}^{(n+s)\times\primDim}\\\\
        \hline\\
        \bm{1}^{(m-\vSize)\times\primDim}\\
        \hline\\
        \bm{0}^{\vSize\times\primDim}\\
    \end{pmatrix}
    .\]

Specifically the convolution does this:
\begin{align*}
    \paren{\bm{h}^6\ast\bm{z}_5}&=\mathrm{coeff}\paren{(1+X^{2n+2s+2m+t})\cdot \paren{((f(\bm{x}))(X))\cdot X^{n+m+s+t}+ \bm{v}(X)\cdot X^{2n+m+2s+t}}}\\
    &= \mathrm{coeff}((f(\bm{x}))(X))\cdot X^{n+m+s+t}+\mathrm{coeff}((f(\bm{x})(X))\cdot X^{3n+3m+3s+2t})\\
    &+\bm{v}(X)\cdot X^{2n+1m+2s+t}+\bm{v}(X)\cdot X^{4n+3m+4s+2t}\\
    &=\shiftDown((f(\bm{x})(X)),n+m+s+t)\\
    &+\shiftDown((f(\bm{x})(X)),3n+3m+3s+2t)\\
    &+\shiftDown(\bm{v},2n+m+2s+t)+\shiftDown(\bm{v}, 4n+3m+4s+2t).
\end{align*}

   \[
    \bm{z}_6 = \bm{b}_1^6 \odot \paren{\bm{h}^6 \ast \bm{z}_5}
     =
    \begin{pmatrix}
        \bm{0}^{(n+s)\times\primDim}\\
        \hline\\
        \bm{0}^{(m+t)\times\primDim}\\
        \hline\\
        \bm{0}^{(n+s)\times\primDim}\\\\
        \hline\\
        \bm{0}^{(m-\vSize)\times\primDim}\\
        \hline\\
        \bm{1}^{r\times\primDim}\\
        \hline\\
        \bm{1}^{(n+s))\times\primDim}\\\\
        \hline\\
        \bm{1}^{(m-\vSize)\times\primDim}\\
        \hline\\
        \bm{0}^{\vSize\times\primDim}\\
    \end{pmatrix} \odot
    \begin{pmatrix}
        \bm{0}^{(n+s)\times\primDim}\\
        \hline\\
        \bm{0}^{(m+t)\times\primDim}\\
        \hline\\
        f(\bm{x})\\
        \hline\\
        \bm{v}_2\\
        \hline\\
        \bm{v}_1\\
        \hline\\
        f(\bm{x})\\
        \hline\\
        \bm{v}_2\\
        \hline\\
        \bm{v}_1\\
    \end{pmatrix} = 
    \begin{pmatrix}
        \bm{0}^{(n+s)\times\primDim}\\
        \hline\\
        \bm{0}^{(m+t)\times\primDim}\\
        \hline\\
        \bm{0}^{(n+s)\times\primDim}\\\\
        \hline\\
        \bm{0}^{(m-\vSize)\times\primDim}\\
        \hline\\
        \bm{v}_1\\
        \hline\\
        f(x)\\
        \hline\\
        \bm{v}_2\\
        \hline\\
        \bm{0}^{r\times\primDim}\\
    \end{pmatrix}.
    \]

Finally we get, 
\[\bm{z}_7 \gets \shiftUp(\bm{z}_6, 2n+2m+2s+t-r)\] The final output of this layer is:

\[    \bm{z}_7 \gets
    \begin{pmatrix}
        \bm{v}_1\\
        \hline\\
        f(x)\\
        \hline\\
        \bm{v}_2\\
        \hline\\
        \bm{0}\\
        \hline
        \vdots\\
        \hline\\ 
        \bm{0}
    \end{pmatrix},
    \]
which is our final output.
\end{proof}

\begin{corollary}\label{remember:extended}
Let $\bm{y}$ be as in \Cref{prop: prim-remember} but now let $f$ be implemented with $\BaseConv(\primLength, L, \primDim, \primLength, \primDim)$. Then $\primRem(\bm{y}, r, t, f)$ where $t-r=n$ can be implemented with $\BaseConv(\primLength, O(L), \primDim, \primLength, \primDim)$.
\end{corollary} 
\begin{proof}
        The Remember primitive can be used to implement any number of BaseConv layers. As shown, the remember primitive can perform a BaseConv operation on a portion of a matrix while maintaining the values of the rest. This output matrix can then be fed through another remember primitive repeatedly such that any number of BaseConv layers can be performed through remember. 
\end{proof}

\begin{definition}
Being that primitives 1-7 can be implemented using BaseConv layers and that remember can apply BaseConv to a continuous subsection of a matrix, we can implement these primitives on subsections of any matrix "through" remember. This will be represented as \texttt{Remember($i,j,f$)} where $i$ and $j$ are the start and end rows that will be effected, respectively, and $f$ is the function which will be applied to them. 
\end{definition}

\subsubsection{Proof of \Cref{MQAR:math}}
We first mathematically state the major steps, after which we will show how to implement each step using BaseConv layers. 
\begin{enumerate}
    \item $\ip \mQ, \mK, \mV$. \\
    $\op \mQ, \mK', \mV$ where $\mK' \in \R^{\inputLength \times \modelDim\bN}$ is defined below. \\
    In steps 1.1 and 1.2, for $1\leq i\leq\inputLength$, replace each $1$ in $\mK[i,:]$ by $\bin(i)^{\top}$ to get $\overline{\mK}\in\R^{\inputLength\times\modelDim\bN}$. Then in step 1.3 then we compute ${\mK}'$ where every row is the sum of all previous rows and itself in $\overline{\mK}$. 
    \begin{enumerate}[label=1.\arabic*]
        \item $\ip \mQ, \mK, \mV$. \\
    $\op \mQ, \overline{\mK}', \mV$ where $\overline{\mK}' \in \R^{\inputLength \times \modelDim\bN}$ is defined below. \\
    \[\overline{\mK}':=\primRepCol(\mK,\bN).\]
    \item $\ip \mQ, \overline{\mK}', \mV$. \\
    $\op \mQ, \overline{\mK}, \mV$ where $\iMap\in\R^{\inputLength\times\modelDim\bN}, \overline{\mK}' \in \R^{\inputLength \times \modelDim\bN}$ are defined below. \\
    \begin{equation}\label{equation: bin rep}
    {\iMap}[i,j\bN+1:(j+1)\bN] := \bin(i)^{\top}   \text{ for all } 1 \leq i\le N\text{ and } 1\leq j\le\modelDim.\end{equation}
    \[\overline{\mK}:=\iMap\odot\overline{\mK}'.\]

    \item $\ip \mQ, \overline{\mK}, \mV$. \\
    $\op \mQ, \mK', \mV$ where $\mK' \in \R^{\inputLength \times \modelDim\bN}$ is defined below. \\
    \[\mK'[i,:] := \sum_{j=1}^{i} \overline{\mK}[j,:]  \text{ for all } 1 \leq i\leq N.\]
    \end{enumerate}

    \item $\ip \mQ, \mK', \mV$. \\
    $\op \overline{\mM}, \mV$ where $\overline{\mM}\in\R^{\inputLength \times \modelDim\bN}$ is defined as follows. \\ In steps 2.1-2.3, compute $\overline{\mM}[i,:]$ so that $\overline{\mM}[i,1:\bN]=\bin(j)^\top$ where $\mQ[i,:]=\mK[j,:]$ for every $1\leq i\leq\inputLength.$ Note that by assumption (ii) only one such $j$ exists.

    \begin{enumerate}[label=2.\arabic*]
        \item $\ip \mQ, \mK', \mV. \\\op \mQ', \mK', \mV.\\$Compute $\mQ'\in\R^{\inputLength\times\modelDim\bN}$ to be $\mQ$ with each column repeated $\bN$ times.
        \begin{align*}
            \mQ'=\primRepCol(\mQ, \bN).
        \end{align*}
        \item $\ip \mQ', \mK', \mV. \\\op \overline{\mM}', \mV.\\$ Compute $\overline{\mM}'$ as it is defined below:
            \[\overline{\mM}'=\mQ'\odot {\mK}'.\]
    Some column block of $\overline{\mM}'$ holds $\bin(j)$ such that $\mQ[i,:]$ matches $\mK[j,:]$, we now move it to the first column block in step 2.3.
    \item $\ip \overline{\mM}', \mV.\\ \op \overline{\mM}, \mV.$ Compute $\overline{\mM}$ as it is defined below:
    \[\overline{\mM}=\primSumColBlo(\overline{\mM}', \bN).\]
    \end{enumerate}
    
    The first $\bN$ entries in the $i$'th row $\overline{\mM}$ holds $\bin(j)$ such that $\mQ[i,:]$ matches $\mK[j,:]$.

\item $\ip \overline{\mM}, \mV$.\\
$\op \mL, \mV$ where $\mL \in \R^{\inputLength\times\inputLength}$. \\
Compute $\mL$ from $\overline{\mM}$ as defined below:
\[\mL=\mC \odot (\mQ\mK^{\top}).\]

Compute $\mL\in\R^{\inputLength\times\inputLength}$ from $\overline{\mM}$ such that the binary representation of $j$ in the $i$'th block is converted into $1$-hot encoding of $j$ in the $i'$th row of $\overline{\mM}$. Define
\[\mL \gets \primHot(\overline{\mM})\]




So we now have computed $\mL = \mC \odot (\mQ\mK^{\top})$. All that's left to do is to compute $\mL \times \mV$. While mathematically this is a simple operation, we will implement this in multiple steps so that it is easy to implement this with BaseConv layers on input $\begin{pmatrix}
    \mL\\
    \hline
    \mV
\end{pmatrix}$.

\item $\ip \mL, \mV$.\\ $\op \mL, \overline{\mV}$, where $\overline{\mV} \in \R^{\inputLength\bb\times \modelDim}$ is defined below.\\
In steps 4.1 and 4.2, compute $\overline{\mV}$ from $\mV$ where the $v$'th column holds the information for $\bin(v)\in \{0,1\}^{\bb}$, for every $1\leq v\leq d:$
\begin{enumerate}[label=4.\arabic*]
    \item $\ip \mL, \mV$.\\ $\op \mL, \overline{\mV}_1$, where $\overline{\mV}_1 \in \R^{\inputLength\bb\times \modelDim}$ is defined below.\\
    \[\overline{\mV}_1:=\primRepMat(\mV, \bb).\]
    \item $\ip \mL, \overline{\mV}_1$.\\ $\op \mL, \overline{\mV}$, where $\overline{\mV}_1 \in \R^{\inputLength\bb\times \modelDim}$ is defined below. First, define $\iMap'\in^{\inputLength\bb\times\modelDim}$ for $1\leq i\leq\bb, 1\leq k\leq\inputLength, 1\leq j\leq\modelDim$ as:\\

    \begin{equation}\label{eq: possible binary}
    \iMap'[(i,k),j] \equiv 
    \begin{cases}
        1 & \text{if } j\mod2^{i}\geq2^{i-1}\\
        0 & \text{otherwise.}
    \end{cases}
    \end{equation}
    Then
    \[\overline{\mV}:=\overline{\mV}_1\odot\iMap'.\]
\end{enumerate}

\item $\ip \mL, \overline{\mV}$.\\ $\op \mL, \mV'$, where $\mV' \in \R^{\inputLength\bb \times \modelDim}$ is defined below.\\
Compute $\mV'$ from $\overline{\mV}$ such that all the non-zero encodings of columns of $\overline{\mV}$ are moved to $1$st column and the other columns are zeroed out, specifically for $1\leq i\leq\inputLength\bN$:

    \[ \mV'[:,i] := \begin{cases}
        \sum_{j=1}^{\modelDim} \overline{\mV}[:,j] & \text{if } i=1 \\
        \bm{0}^{\inputLength\bb} &\text{otherwise}.
    \end{cases}
    \]
The summation of each column of $\overline{\mV}$ performs this desired move action because each row of $\overline{\mV}$ has at most 1 non zero value as each row of $\iMap'$ is a 1-hot encoding.
\item $\ip \mL, \mV'$.\\ $\op \mL, \mV_1$ where $\mV_1\in\R^{\inputLength\times\inputLength\bb}$ is defined below. \\
Compute $\mV_1$ from $\mV'$ in steps 6.1-6.6 such that for $1\leq k\leq\inputLength$ and $\bm{W}\in\R^{\inputLength\bb\times\inputLength\bb}$:
    \[\bm{W}[i,j] =  
    \begin{cases}
        1 & \text{if } i=\paren{\paren{(j-1)N+\left\lfloor\frac{i-1}{N}\right\rfloor}\mod{\inputLength\bb}}+1\\
        0 & \text{otherwise.}
    \end{cases}
    \]
\[ \mV_1[k,:] := (\mV'[:,1]^{\top})\bm{W}.\]
$\bm{W}$ is a permutation matrix which reorders the values such that values relating to the same index rather than values relating to the same copy of the matrix, made in step 4.1, are adjacent. 

\begin{enumerate}[label=6.\arabic*]
\item  $\ip \mL, \mV'$.\\ $\op \mL, \mV_2.\\$ Compute $\mV_2\in\R^{\inputLength\bb\times\inputLength\bb}$ from $\mV'$ where $\mV_2$ has the $1$st column of $\mV'$ repeated $\inputLength\bb$ times, as defined below for $1\leq j\leq\inputLength\bb:$
\[\mV_{2}[:,j] = \mV'[:,1]\]
\item $\ip \mL,  \mV_2$.\\ $\op \mL, \mV_3.\\$ Compute $\mV_3\in\R^{\inputLength\bb\times\inputLength\bb}$ from $\mV_2$ by zeroing out all but the diagonal:
\[ \mV_3=\bm{I}^{\inputLength\bN\times\inputLength\bN}\odot \mV_2. \]
\item $\ip \mL,  \mV_3$.\\ $\op \mL, \mV_4.\\$ Compute $\mV_4\in\R^{\inputLength\bb\times\inputLength\bb}$ from $\mV_3$ and $\bm{W}$ such that the values are permuted such that the values grouped by which row they represent, rather than the matrix repeat number they were in. Where $\mV_4$ is defined as: 

    \[\mV_4= \mV_3 \bm{W}.\]

\item $\ip \mL, \mV_4$.\\ $\op \mL, \mV_5.\\$ Compute $\mV_5\in\R^{\inputLength\bb\times\inputLength\bb}$ from $\mV_4$ such that $\mV_{5}[1,:]$ is the sum of all rows by making $\mV_{5}[\inputLength\bb,:]$ the sum of all rows of $\mV_{4}$ followed by zeroing the rest of rows. I.e. for all $1\leq i,j\leq\inputLength\bb:$
\[
\mV_5[i,:] :=
\begin{cases}
    \sum_{k=1}^{N\bb} \mV_4[k,:] &\text{if }i = \inputLength\bb\\
    \bm{0} &\text{otherwise}.
\end{cases}
\]

\item $\ip \mL, \mV_5$.\\ $\op \mL, \mV_6.\\$ Compute $\mV_6\in\R^{\inputLength\bb\times\inputLength\bb}$ from $\mV_5$ such that $\mV_{6}[1,:]=\mV_{5}[\inputLength\bb,:]$ and the rest of the rows be zeroed out. I.e. for all $1\leq i,j\leq\inputLength\bb:$
\[
\mV_6[i,:] :=
\begin{cases}
    \mV_5[\inputLength\bb,:] &\text{if }i = 1\\
    \bm{0} &\text{otherwise}.
\end{cases}
\]

\item $\ip \mL,  \mV_6$.\\ $\op \mL, \mV_1.\\$ Compute $\mV_{1}$ from $\mV_{6}$ by copying $\mV_{6}[1,:]$ to rest of the rows.
\begin{align*}
    \mV_{1}[i,:] := \mV_{6}[1,:] &\text{ for all }1\leq i\leq N.
\end{align*}
At this point, each row of $\bm{V}_1$ now has the same values as the first column of $\bm{V}'$ permuted in the way that was stated in 6.3.
\end{enumerate}

\item $\ip \mL, \mV_1$.\\ $\op \overline{\mL}$ where $\overline{\mL}\in R^{\inputLength\times\inputLength\bN}$ is defined below.\\
Compute $\overline{\mL}$ such that the single $1$ in $\mL [i,:]$ (say at position $0\leq j<N$) is replaced by $\bin(j)^{\top}$. In other words, the $i$'th query were to match a key in position $j$, if the $i$'th row of $\overline{\mL}$ would have a representation of the matching value at the $j$'th block.
\begin{enumerate}[label=7.\arabic*]
    \item $\ip \mL, \mV_1$.\\ $\op \mL',  \mV_1\\$ Compute $\mL'\in R^{\inputLength\times\inputLength\bb}$ by repeating each column of $\mL$ $\bb$ times:
    \begin{align*}
    \mL'[:,j]=\mL\left[:, \left\lceil\dfrac{j}{\bb}\right\rceil\right] & \text{ if } 1\leq j\leq\inputLength\bb.
    \end{align*}
    \item $\ip \mL', \mV_1$.\\ $\op \overline{\mL}\\$Compute $\overline{\mL}\leftarrow \mL' \odot \mV_{1}$.
\end{enumerate}

\item $\ip \overline{\mL}.\\\op \overline{\mL}_1=\mL\times\mV,$ the final output.
\begin{enumerate}[label=8.\arabic*]
    \item $\ip\overline{\mL}.\\\op\overline{\mL}_2\in\R^{\inputLength\times(\inputLength\times\bb)}$ by summing up $\bb$ chunks of columns of $\overline{\mL}$ and store the result in $1$st block column and zero out remaining columns to get $\overline{\mL}_{2}$, specifically for $1\le j \le\bb$:
\[
\overline{\mL}_2[:,(i,j)] \equiv
\begin{cases}
    \sum_{k=1}^{\inputLength} \mL[:,(k,j)] &\text{if }i=1\\
    \bm{0}^{\inputLength\times\bb} &\text{otherwise}.
\end{cases}
\]
    \item $\ip\overline{\mL}_2\\\op\overline{\mL}_1$ by replacing the binary representation in $1$st block column by corresponding $1$-hot encoding. Compute $\overline{\mL}_1$ from $\overline{\mL}_2$ such that 
    for $1\le i \le\inputLength,$ $\overline{\mL}_{1}[i,:]=\bm{e}_{\ell}^{(d)}$ where $\overline{\mL}_{2}[i,(1,:)]=\bin(\ell)^{\top}.$
\end{enumerate}
\end{enumerate}

Next, using the primitives defined in \Cref{prims}, we show how each step above would be implemented with $\baseconv$ layers. Instead of 3 separate matrices as shown in the math layout, we will have a single matrix, $\qkv\in\R^{3\inputLength\times\modelDim}$ which contains the information of all three. 
\[\qkv =\    
        \begin{pmatrix}
        \mQ\\
        \hline\\
        \mK\\
        \hline\\
        \mV\\
    \end{pmatrix}.\]


The internal dimension will be $(4\combDim\log(\combDim),\combDim\log(\combDim))$. For notational convenience, we define \[\posSize=\combDim\log(\combDim).\] Specifically, the input $\bm{Y}\in\R^{3\posSize\times\modelDim}$ has the top left matrix of $3\inputLength\times\modelDim$ holding
$\begin{pmatrix}
    \bm{Q}\\\bm{K}\\\bm{V}
\end{pmatrix}$.

Define
\[\qkv_0\gets\baseconv(\qkv, \bm{0}^{4\posSize\times\posSize}, \bm{b}_1^0, \bm{h}_0, \bm{0}^{4\posSize\times\posSize}),\]

where $\bm{b}_1^0\in\R^{4\posSize\times\posSize}$ and $\bm{h}_0\in\R^{4\posSize\times\posSize}$ are defined as:

\[\bm{b}_1^0\gets\begin{pmatrix}
    \bm{1}^{\inputLength\times\posSize}\\
    \hline\\
    \bm{0}^{(\posSize-\inputLength)\times\posSize}\\
    \hline\\
    \bm{1}^{\inputLength\times\posSize}\\
    \hline\\
    \bm{0}^{(\posSize-\inputLength)\times\posSize}\\
    \hline\\
    \bm{1}^{\inputLength\times\posSize}\\
    \hline\\
    \bm{0}^{(\posSize-\inputLength)\times\posSize}\\
    \hline\\
    \bm{0}^{\inputLength\times\posSize}\\
    \hline\\
    \bm{0}^{(\posSize-\inputLength)\times\posSize}\\

\end{pmatrix},
\bm{h}_0\gets\begin{pmatrix}\\
    \bm{e}_1^{(\posSize-\inputLength)}\\\\
    \hline\\\\
    \bm{e}_1^{(\posSize-\inputLength)}\\\\
    \hline\\\\
    \bm{e}_1^{(\posSize-\inputLength)}\\\\
    \hline\\\\
    \bm{0}^{\posSize+3\inputLength}\\\\
\end{pmatrix}.
\]

The vectors that make up the convolution have dimension $\posSize-\inputLength$ because $\mK$ and $\mV$ are not top left justified in the original matrix. This produces a matrix where $\bm{Q,K}$ and $\bm{V}$ all sit in the top left position in their own sub-matrix surrounded by zeros. The structure is:
\[\qkv_0 \equiv\    
        \begin{pmatrix}
        \mQ^{p}\\
        \hline\\
        \mK^{p}\\
        \hline\\
        \mV^{p}\\
        \hline\\
        \bm{0}^{\posSize\times\posSize}
    \end{pmatrix}.\]
where
\[
\mQ^{p}=\begin{pmatrix}
    \bm{Q}, \bm{0}^{\inputLength\times(\posSize)}\\
    \bm{0}^{(\posSize-\inputLength)\times(\posSize)}
\end{pmatrix},
\mK^{p}=\begin{pmatrix}
    \bm{K}, \bm{0}^{\inputLength\times(\posSize)}\\
    \bm{0}^{(\posSize-\inputLength)\times(\posSize)}
\end{pmatrix},
\mV^{p}=\begin{pmatrix}
    \bm{V}, \bm{0}^{\inputLength\times(\posSize)}\\
    \bm{0}^{(\posSize-\inputLength)\times(\posSize)}
\end{pmatrix}.
\]

\begin{enumerate}
    \item Compute $\qkv_1$ in steps 1.1-1.2.
    \begin{enumerate}[label=1.\arabic*]
        \item Compute $\qkv_1'\gets$ \texttt{remember($\qkv_0, z+1, 2\posSize, f_1$)} where $f_1$ is defined as:
        \[\texttt{repeat-columns$(\overline{\qkv}_0, \bN)$}.\]
        This results in 
        \[\qkv_1' \equiv\    
        \begin{pmatrix}
        \mQ^{p}\\
        \hline\\
        \overline{\mK}'\\
        \hline\\
        \mV^{p}\\
        \hline\\
        \bm{0}^{\posSize\times\posSize}
    \end{pmatrix}.\]
        This repeats the columns of $\mK$ $\bb$ times and doesn't change $\overline{\mQ}$ or $\overline{\mV}$ with $\BaseConv(4\posSize, O(1), \posSize, 4\posSize, \posSize)$ via \Cref{remember:extended}.
        \item Compute $\qkv_{1}''\gets$ \texttt{remember($\qkv'_1,\posSize+1, 2\posSize ,f'_1$)} where $f'_1$ is defined as:
        \[\baseconv (\overline{\qkv}'_1, \bm{I}^{\posSize\times\posSize}, \bm{0}^{\posSize\times\posSize}, \bm{0}^{\posSize\times\posSize}, \iMap).\]
        Where $\iMap$ is defined as it was in \eqref{equation: bin rep}. This results in
        \[\qkv_1'' \equiv\    
        \begin{pmatrix}
        \mQ^{p}\\
        \hline\\
        \overline{\mK}\\
        \hline\\
        \mV^{p}\\
        \hline\\
        \bm{0}^{\posSize\times\posSize}
    \end{pmatrix}.\]

        This replaces each $1$ in $\mK$ with the binary representation of the row with $\BaseConv(4\posSize, O(1), \posSize, 4\posSize, \posSize)$ via \Cref{prop: prim-remember}.
        \item Compute $\qkv_{1}\gets$ \texttt{remember($\qkv''_1,\posSize+1, 2\posSize ,f''_1$)} where $f''_1$ is defined as:
        \[\primCumSum(\overline{\qkv}''_1).\]
        This results in 
        \[\qkv_1 \equiv\    
        \begin{pmatrix}
        \mQ^{p}\\
        \hline\\
        \mK'\\
        \hline\\
        \mV^{p}\\
        \hline\\
        \bm{0}^{\posSize\times\posSize}
    \end{pmatrix}.\]
        \end{enumerate}

    \item Compute $\qkv_2$ in steps 2.1-2.3.
    \begin{enumerate}[label=2.\arabic*]
        \item Compute $\qkv'_2\gets$ \texttt{Remember($\qkv_1, 1, \posSize, f_2$)} where $f_2$ is defined as:
        \[\texttt{repeat-columns($\overline{\qkv}_1,\bN$)}.\]
        This results in 
        \[\qkv_2' \equiv\    
        \begin{pmatrix}
        \mQ'\\
        \hline\\
        \mK'\\
        \hline\\
        \mV^{p}\\
        \hline\\
        \bm{0}^{\posSize\times\posSize}
    \end{pmatrix}.\]
        This repeats the columns of $\mQ$ $\bb$ times with $\BaseConv(4\posSize, O(1), \posSize, 4\posSize, \posSize)$ via \Cref{remember:extended}.
        \item Compute $\qkv_2$ as the Hadamard product of the first and second position stored in the second position. This can be done with the following substeps:
        \[\qkv_{2}\gets \baseconv (\overline{\qkv}'_2, \bm{I}^{\posSize\times\posSize}, \bm{0}^{4\posSize\times\posSize}, \bm{h}^{2}, \bm{0}^{4\posSize\times\posSize}).\]
        Where $\bm{h}^2\in\R^{4\posSize\times\posSize}$ is defined as:
    \[
    \bm{h}^{2}
    \equiv
    \begin{pmatrix}
        \bm{0}^{\posSize\times\posSize}\\
        \hline\\
        \bm{e}_1^{(\posSize)}\\
        \hline\\
        \bm{0}^{\posSize\times\posSize}\\
        \hline\\
        \bm{0}^{\posSize\times\posSize}\\
    \end{pmatrix}.
    \]
This layer computes:
        \[
    \qkv''_2 = \qkv'_2 \odot \paren{\bm{h}^2 \ast \qkv'_2}
     =
    \begin{pmatrix}
        \mQ'\\
        \hline\\
        \mK'\\
        \hline\\
        \mV\\
        \hline\\
        \bm{1}^{\posSize}\\
    \end{pmatrix} \odot
    \begin{pmatrix}
        \bm{0}^{\posSize\times\posSize}\\
        \hline\\
        \mQ'\\
        \hline\\
        \mK'\\
        \hline\\
        \mV\\
    \end{pmatrix} \equiv 
    \begin{pmatrix}
        \bm{0}^{\posSize\times\posSize}\\
        \hline\\
        \mQ'\odot\mK'\\
        \hline\\
        \mK' \odot\mV\\
        \hline\\
        \mV\\
    \end{pmatrix}.
    \]

Next, we mask out the unnecessary position using
 \[\qkv^3_{2}\equiv \baseconv (\overline{\qkv}''_2, \bm{b}_1^{2}, \bm{0}^{4\posSize\times\posSize}, \bm{e}_{1}^{4\posSize\times\posSize}, \bm{0}^{4\posSize\times\posSize}).\]
Where $\bm{b}_1^{2}\in\R^{4\posSize\times\posSize}$ is defined as:
\[
    \bm{b}_1^{2} 
    \equiv
    \begin{pmatrix}
        \bm{0}^{\posSize\times\posSize}\\
        \hline\\
        \bm{1}^{\posSize\times\posSize}\\
        \hline\\
        \bm{0}^{\posSize\times\posSize}\\
        \hline\\
        \bm{1}^{\posSize\times\posSize}\\
    \end{pmatrix}.
    \]
This layer computes:
        \[
    \qkv^3_2 = \bm{b}_1^2 \odot \qkv''_2
     =
    \begin{pmatrix}
        \bm{0}^{\posSize\times\posSize}\\
        \hline\\
        \bm{1}^{\posSize\times\posSize}\\
        \hline\\
        \bm{0}^{\posSize\times\posSize}\\
        \hline\\
        \bm{1}^{\posSize}\\
    \end{pmatrix} \odot
    \begin{pmatrix}
        \bm{0}^{\posSize\times\posSize}\\
        \hline\\
        \mQ'\odot\mK'\\
        \hline\\
        \mK' \odot\mV\\
        \hline\\
        \mV\\
    \end{pmatrix} \equiv
    \begin{pmatrix}
        \bm{0}^{\posSize\times\posSize}\\
        \hline\\
        \mQ'\odot\mK'\\
        \hline\\
        \bm{0}^{\posSize\times\posSize}\\
        \hline\\
        \mV\\
    \end{pmatrix} \equiv
    \begin{pmatrix}
        \bm{0}^{\posSize\times\posSize}\\
        \hline\\
        \overline{\mM}'\\
        \hline\\
        \bm{0}^{\posSize\times\posSize}\\
        \hline\\
        \mV\\
    \end{pmatrix}.
    \]

Finally, we shift up:
\[\qkv^4_2 \gets \texttt{shift-up($\qkv^3_2$, $\posSize$)}.\]
This results in: 
\[  \qkv^4_2 =
    \begin{pmatrix}
        \overline{\mM}'\\
        \hline\\
        \bm{0}^{\posSize\times\posSize}\\
        \hline\\
        \mV\\
        \hline\\
        \bm{0}^{\posSize\times\posSize}\\
    \end{pmatrix}.
\]
This was done using $\BaseConv(4\posSize, O(1), \posSize, 4\posSize, \posSize)$ via the vanilla BaseConv layers and \Cref{prop: prim-shift}.

\item Sum up block columns in the first position to move the binary representations to the first column blocks through $\primRem(\qkv^4_2, 1, \posSize, f_2).$ Where $f_2$ is defined as:
\[\primSumColBlo(\overline{\qkv}_2^4, \bN).\]
This results in: 
\[  \qkv_2 \equiv
    \begin{pmatrix}
        \overline{\mM}\\
        \hline\\
        \bm{0}^{\posSize\times\posSize}\\
        \hline\\
        \mV\\
        \hline\\
        \bm{0}^{\posSize\times\posSize}\\
    \end{pmatrix}.
\]
    \end{enumerate}
    Through this step, we have computed $\overline{\mM}$ with $\BaseConv(4\posSize, O(1), \posSize, 4\posSize, \posSize)$ via \Cref{remember:extended}..
\item Compute $\qkv_3\gets$ \texttt{Remember($\qkv_2, 1, \posSize$, $f_3$)} where $f_3$ is defined by:
\[\texttt{\primHot($\overline{\qkv}_2$)}\]
This step was computed with $\BaseConv(4\posSize, O(\lceil\log\log\overline{N}\rceil, \posSize, 4\posSize, \posSize)$ via \Cref{remember:extended}.
This converts $\overline{\bm{M}}$ to be 1-hot encoded in $O(\combDim\log\combDim)$ BaseConv layers. This results in:
\[ \qkv_3 \equiv
    \begin{pmatrix}
        \mL\\
        \hline\\
        \bm{0}^{\posSize\times\posSize}\\
        \hline\\
        \mV\\
        \hline\\
        \bm{0}^{\posSize\times\posSize}\\
    \end{pmatrix}.
\]
We move this binary representation to the first column block with $\BaseConv(4\posSize, O(1), \posSize, 4\posSize, \posSize)$ via \Cref{prop: prim-shift}.

\item Compute $\qkv_4$ in steps 4.1 and 4.2.
\begin{enumerate}[label=4.\arabic*]
    \item Compute $\qkv'_4 \gets$ \texttt{remember($\qkv_3, 2\posSize+1, 3\posSize, f_4'$)} where $f_4'$ is defined as:
    \[\texttt{repeat-matrix($\overline{\qkv}_3, \bb)$}\]
     This results in:
\[
    \begin{pmatrix}
        \mL\\
        \hline\\
        \bm{0}^{\posSize\times\posSize}\\
        \hline\\
        \overline{\mV}'\\
        \hline\\
        \bm{0}^{\posSize\times\posSize}\\
    \end{pmatrix}.
\]

    We repeat $\bm{V}$ matrix $\bb$ times with $\BaseConv(4\posSize, O(1), \posSize, 4\posSize, \posSize)$ via \Cref{remember:extended}.
    \item Compute $\qkv_4\gets$ \texttt{remember($\qkv'_4, 2\posSize+1, 3\posSize, f_4$)} where $f_4$ is defined as:
    \[\texttt{BaseConv($\overline{\qkv}'_4, \bm{I}^{\posSize\times\posSize}, \bm{0}^{\posSize\times\posSize}, \bm{0}^{\posSize\times\posSize}, \iMap')$}.\]
    Where $\iMap'$ is defined as it was in \eqref{eq: possible binary}.

    This results in:
\[
    \begin{pmatrix}
        \mL\\
        \hline\\
        \bm{0}^{\posSize\times\posSize}\\
        \hline\\
        \overline{\mV}\\
        \hline\\
        \bm{0}^{\posSize\times\posSize}\\
    \end{pmatrix}.
\]

This step can be done with $\BaseConv(4\posSize, O(1), \posSize, 4\posSize, \posSize)$ via \Cref{prop: prim-remember}.

\end{enumerate}

\item Compute $\qkv_5\gets$ \texttt{remember($\qkv_4, 2\posSize+1, 3\posSize, f_5$)})where $f_5$ is defined as:
\[\texttt{sum-all-columns($\overline{\qkv}_4$)} \]
This results in:
\[  \qkv_5\equiv
    \begin{pmatrix}
        \mL\\
        \hline\\
        \bm{0}^{\posSize\times\posSize}\\
        \hline\\
        \mV'\\
        \hline\\
        \bm{0}^{\posSize\times\posSize}\\
    \end{pmatrix}.
\]
Now each row of $\bm{V}'$ has that one moved to the first column if it existed with $\BaseConv(4\posSize, O(1), \posSize, 4\posSize, \posSize)$ via \Cref{remember:extended}.

\item Compute $\qkv_6$ through in steps 6.1-6.6.
\begin{enumerate}[label=6.\arabic*]
\item  Compute $\qkv'_6\gets$ \texttt{remember($\qkv_5, 2\posSize+1, 3\posSize, f_6'$)} where $f_6'$ is defined as:
\[\texttt{repeat\_columns($\overline{\qkv}_5, \inputLength \bb$)}.\]
This results in:
\[  \qkv_6'\equiv
    \begin{pmatrix}
        \mL\\
        \hline\\
        \bm{0}^{\posSize\times\posSize}\\
        \hline\\
        \mV_2\\
        \hline\\
        \bm{0}^{\posSize\times\posSize}\\
    \end{pmatrix}.
\]
Here we repeat the columns of $\bm{V}'$ with $\BaseConv(4\posSize, O(1), \posSize, 4\posSize, \posSize)$ via \Cref{remember:extended}.

\item  Compute $\qkv''_6\gets$ \texttt{remember($\qkv'_6, 2\posSize+1, 3\posSize, f_6''$)} where $f_6''$ is defined as:
 \[\texttt{BaseConv($\overline{\qkv}'_6, \bm{0}, \bm{I}^{\posSize\times\posSize}, \bm{e}_1^{(z)}, \bm{0}^{})$}.\]
 
This results in:
\[  \qkv_6''\equiv
    \begin{pmatrix}
        \mL\\
        \hline\\
        \bm{0}^{\posSize\times\posSize}\\
        \hline\\
        \mV_3\\
        \hline\\
        \bm{0}^{\posSize\times\posSize}\\
    \end{pmatrix}.
\]
 
Here we zeroed out everything except the main diagonal in $\bm{V}$ with $\BaseConv(4\posSize, O(1), \posSize, 4\posSize, \posSize)$ via \Cref{prop: prim-remember}. 

\item Compute $\qkv^3_6\gets$ \texttt{remember($\qkv_6'', 2\posSize+1, 3\posSize, f_6^3$)} where $f_6^3$ is defined as:
 \[\texttt{BaseConv($\overline{\qkv}''_6, \bm{W_6}, \bm{0}^{4\posSize\times\posSize}, \bm{0}^{4\posSize\times\posSize}, \bm{1}^{4\posSize\times\posSize})$}.\]
Where $\bm{W}_6$ is defined as it was in math step 6.

This results in:
\[  \qkv_6^3\equiv
    \begin{pmatrix}
        \mL\\
        \hline\\
        \bm{0}^{\posSize\times\posSize}\\
        \hline\\
        \mV_4\\
        \hline\\
        \bm{0}^{\posSize\times\posSize}\\
    \end{pmatrix}.
\]

Here, because we can only repeat whole matrices and not each row, we are permuting the values on the main diagonal to reorganize them to be as if they were repeated each row at a time with $\BaseConv(4\posSize, O(1), \posSize, 4\posSize, \posSize)$ via \Cref{remember:extended}.

\item Compute $\qkv^4_6\gets$ \texttt{remember($\qkv^3_6, 2\posSize, 3\posSize-1, f_6^4$)} where $f_6^4$ is defined as:
\[ \texttt{cumulative\_sum($\overline{\qkv}^3_6$)}.\]

This results in:
\[  \qkv_6^4\equiv
    \begin{pmatrix}
        \mL\\
        \hline\\
        \bm{0}^{\posSize\times\posSize}\\
        \hline\\
        \mV_5\\
        \hline\\
        \bm{0}^{\posSize\times\posSize}\\
    \end{pmatrix}.
\]

Here, we use $\primCumSum$ so that the final row in the matrix stores each value in each column with $\BaseConv(4\posSize, O(1), \posSize, 4\posSize, \posSize)$ via \Cref{remember:extended}.

\item Compute $\qkv^5_6\gets$ \texttt{remember($\qkv_6^4, 2\posSize, 3\posSize-1, f_6^5$)} where $f_6^5$ is defined as:
\[\texttt{shift\_up($\overline{\qkv}_6^4, \inputLength\bb-1$)}.\]

This results in:
\[  \qkv_6^5\equiv
    \begin{pmatrix}
        \mL\\
        \hline\\
        \bm{0}^{\posSize\times\posSize}\\
        \hline\\
        \mV_6\\
        \hline\\
        \bm{0}^{\posSize\times\posSize}\\
    \end{pmatrix}.
\]

Here, we shift this final row up to be in the first row with $\BaseConv(4\posSize, O(1), \posSize, 4\posSize, \posSize)$ via \Cref{prop: prim-shift}.

\item Compute $\qkv_6\gets$ \texttt{remember($\qkv_6^5,2\posSize,3\posSize-1, f_6^6$)} where $f_6^6$ is defined as:
\[ \texttt{cumulative\_sum($\overline{\qkv}^5_6$)}.\]

This results in:
\[  \qkv_6\equiv
    \begin{pmatrix}
        \mL\\
        \hline\\
        \bm{0}^{\posSize\times\posSize}\\
        \hline\\
        \mV_1\\
        \hline\\
        \bm{0}^{\posSize\times\posSize}\\
    \end{pmatrix}.
\]

Here we copied the first row to each row in $\bm{V}$ with $\BaseConv(4\posSize, O(1), \posSize, 4\posSize, \posSize)$ via \Cref{remember:extended}.
This results in:
\[ \bm{\qkv_6}\equiv
    \begin{pmatrix}
        \mL\\
        \hline\\
        \bm{0}^{\posSize\times\posSize}\\
        \hline\\
        \mV_1\\
        \hline\\
        \bm{0}^{\posSize\times\posSize}\\
    \end{pmatrix}.
\]
\end{enumerate}
This step is computed with with $\BaseConv(4\posSize, O(1), \posSize, 4\posSize, \posSize)$ via \Cref{remember:extended}.

\item Compute $\qkv_7$ in steps 7.1-7.3.
\begin{enumerate}[label=7.\arabic*]
    \item Compute $\qkv'_7\gets$ \texttt{remember($\qkv_6, 1, \posSize, f_7$)} where $f_7$ is defined as:
    \[ \texttt{repeat\_columns($\overline{\qkv}_6, \bb$)}.\]

    This results in:
    \[ \bm{\qkv_7'}\equiv
        \begin{pmatrix}
            \mL'\\
            \hline\\
            \bm{0}^{\posSize\times\posSize}\\
            \hline\\
            \mV_1\\
            \hline\\
            \bm{0}^{\posSize\times\posSize}\\
        \end{pmatrix}.
    \]
    
    Here we repeat the columns of $\bm{L}$ with $\BaseConv(4\posSize, O(1), \posSize, 4\posSize, \posSize)$ via \Cref{remember:extended}.
    \item We want to Hadamard the first and third position which we can do with:
    \[\qkv''_{7}\equiv \baseconv (\qkv'_7, \bm{I}^{\posSize\times\posSize}, \bm{0}^{4\posSize\times\posSize}, \bm{h}^{7}, \bm{0}^{4\posSize\times\posSize}).\]
        Where $\bm{h}^7\in\R^{\posSize\times\posSize}$ is defined as:
    \[
    \bm{h}^{7} 
    \equiv
    \begin{pmatrix}
        \bm{0}^{\posSize\times\posSize}\\
        \hline\\
        \bm{0}^{\posSize}\\
        \hline\\
        \bm{e}_{1}^{\posSize}\\
        \hline\\
        \bm{0}^{\posSize}\\
    \end{pmatrix}.
    \] 
    Here we Hadamard the $\bm{L}$ and $\bm{V}_1$ together. Which results in:

    \[
        \bm{Y}_7''=\begin{pmatrix}
        \bm{0}^{\posSize\times\posSize}\\
        \hline\\
        \bm{0}^{\posSize\times\posSize}\\
        \hline\\
        \overline{\mL}\\
        \hline\\
        \bm{0}^{\posSize\times\posSize}\\
    \end{pmatrix}\]
This is done with with $\BaseConv(4\posSize, 1, \posSize, 4\posSize, \posSize)$ via a single BaseConv layer.

    \item Finally, we \texttt{shift-up($\qkv''_{7}, 2\posSize$)}.
    This results in:
    \[
    \qkv_7 \equiv
    \begin{pmatrix}
        \overline{\mL}\\
        \hline\\
        \bm{0}^{\posSize}\\
        \hline\\
        \bm{0}^{\posSize}\\
        \hline\\
        \bm{0}^{\posSize}\\
    \end{pmatrix}.
\]
This is done with with $\BaseConv(4\posSize, O(1), \posSize, 4\posSize, \posSize)$ via \Cref{prop: prim-shift}.
\end{enumerate}

\item Convert $\qkv_7$ to the final output.
\begin{enumerate}[label=8.\arabic*]
    \item Here we will move everything to the first block using:
    \[\bm{\qkv}'_8 \equiv\primSumColBlo(\overline{\qkv}_{7}, \bN).\]
    \[ \bm{\qkv_8'}\equiv
    \begin{pmatrix}
        \overline{\mL}_2\\
        \hline\\
        \bm{0}^{\posSize\times\posSize}\\
        \hline\\
        \bm{0}^{\posSize\times\posSize}\\
        \hline\\
        \bm{0}^{\posSize\times\posSize}\\
    \end{pmatrix}.
\]
    
    This is done with with $\BaseConv(4\posSize, 1, \posSize, 4\posSize, \posSize)$ via \Cref{remember:extended}.
    
    \item Compute $\qkv_8\gets$ \texttt{remember($\qkv'_8,1,\posSize, f_8$)} where $f_8$ is defined as:
\[ \texttt{\primHot($\overline{\qkv}'_8$)}.\]
This converts $\overline{\bm{L}}$ to be 1-hot encoded in $O(\combDim\log\combDim)$ BaseConv layers. This results in:
\[ \bm{\qkv_8}\equiv
    \begin{pmatrix}
        \overline{\mL}_1\\
        \hline\\
        \bm{0}^{\posSize\times\posSize}\\
        \hline\\
        \bm{0}^{\posSize\times\posSize}\\
        \hline\\
        \bm{0}^{\posSize\times\posSize}\\
    \end{pmatrix}.
\]
\end{enumerate}
    This is done with with $\BaseConv(4\posSize, O(\lceil\log\log\overline{N}\rceil), \posSize, 4\posSize, \posSize)$ via \Cref{remember:extended}.

\paragraph{Overall cost:} We can solve the MQAR problem with $\BaseConv(4\posSize, O(\lceil\log\log\overline{N}\rceil)+O(1), \posSize, 4\posSize, \posSize)$ by stacking the layers presented above.
\end{enumerate}

\clearpage

\begin{table}[h!]
    \caption{\sysname Training Settings}
    \centering
    \begin{tabular}{rccc}
    \toprule
    {} & 355M & 1.4B  \\
    \midrule
    Optimizer & \multicolumn{3}{c}{Adam} \\
    Optimizer momentum & \multicolumn{3}{c}{$\beta_1, \beta_2=0.9, 0.95$} \\
    Optimizer eps & \multicolumn{3}{c}{$1e-8$} \\
    Precision &  \multicolumn{3}{c}{BFloat16} \\
    \midrule
    Warmup & \multicolumn{3}{c}{1\%} \\
    Learning rate decay & \multicolumn{3}{c}{Cosine} \\
    Learning rate (min, base) & \multicolumn{3}{c}{8e-5, 8e-4} \\
    Global batch size & \multicolumn{3}{c}{256} \\ 
    Weight decay & \multicolumn{3}{c}{0.1} \\
    \midrule
    Num Layers  & 27   & 36   \\
    Hidden Size & 1024 & 1792 \\
    MLP Activation   &\multicolumn{3}{c}{SwiGLU}   \\
    MLP Width   & \multicolumn{3}{c}{2}  \\
    \midrule
    Num. Linear Attn Layers  & 5    & 7   \\
    Num. Linear Attn Heads   & \multicolumn{2}{c}{16}  \\
    Taylor Feature Dimension & \multicolumn{2}{c}{16} \\
    Linear Attn Positional Encodings & \multicolumn{2}{c}{None} \\
    \midrule
    Num. Sliding Window Layers & 5    & 7  \\
    Sliding Window Size & 64 & 16 \\
    Sliding Window Heads & \multicolumn{2}{c}{16} \\
    Sliding Window Positional Encodings & \multicolumn{2}{c}{Rotary} \\
    \midrule
    Num. \BaseConv\ Layers & 17  & 22 \\
    \BaseConv\ Projection Expansion Factor & \multicolumn{2}{c}{4} \\
    \BaseConv\ Filter Size & \multicolumn{2}{c}{3} \\
     \BaseConv\ Activation & \multicolumn{2}{c}{SiLU} \\
    \bottomrule 
    \end{tabular}
    \label{tab:based-training-details}
\end{table}
\begin{table}[h!]
    \caption{Attention Training Settings}
    \centering
    \begin{tabular}{rccc}
    \toprule
    {} & 355M & 1.4B  \\
    \midrule
    Optimizer & \multicolumn{3}{c}{Adam} \\
    Optimizer momentum & \multicolumn{3}{c}{$\beta_1, \beta_2=0.9, 0.95$} \\
    Optimizer eps & \multicolumn{3}{c}{$1e-8$} \\
    Precision &  \multicolumn{3}{c}{BFloat16} \\
    \midrule
    Warmup & \multicolumn{3}{c}{1\%} \\
    Learning rate decay & \multicolumn{3}{c}{Cosine} \\
    Learning rate (min, base) & \multicolumn{3}{c}{8e-5, 8e-4} \\
    Global batch size & \multicolumn{3}{c}{256} \\ 
    Weight decay & \multicolumn{3}{c}{0.1} \\
    \midrule
    Num Layers  & 24    & 36 \\
    Hidden Size & 1024  & 1680 \\
    Num Heads   & 16 & 24  \\
    RMSNorm     & \multicolumn{3}{c}{True} \\
    MLP Bias    & \multicolumn{3}{c}{False}  \\
    Flash Attn  & \multicolumn{3}{c}{True}  \\
    Rotary Emb. Fraction & \multicolumn{3}{c}{0.5}  \\
    MLP Activation & \multicolumn{3}{c}{SwiGLU} \\
    MLP Width & \multicolumn{3}{c}{4}  \\
    \bottomrule 
    \end{tabular}
    \label{tab:attn-training-details}
\end{table}

\begin{table}[h!]
    \caption{Mamba Training Settings}
    \centering
    \begin{tabular}{rccc}
    \toprule
    {} & 355M & 1.4B  \\
    \midrule
    Optimizer & \multicolumn{3}{c}{Adam} \\
    Optimizer momentum & \multicolumn{3}{c}{$\beta_1, \beta_2=0.9, 0.95$} \\
    Optimizer eps & \multicolumn{3}{c}{$1e-8$} \\
    Precision &  \multicolumn{3}{c}{BFloat16} \\
    \midrule
    Warmup & \multicolumn{3}{c}{1\%} \\
    Learning rate decay & \multicolumn{3}{c}{Cosine} \\
    Learning rate (min, base) & \multicolumn{3}{c}{8e-5, 8e-4} \\
    Global batch size & \multicolumn{3}{c}{256} \\ 
    Weight decay & \multicolumn{3}{c}{0.1} \\
    \midrule
    Num Layers   & \multicolumn{3}{c}{46}    \\
    Hidden Size  & 1024    & 2048   \\
    RMSNorm      & \multicolumn{3}{c}{True}  \\
    Norm Epsilon & \multicolumn{3}{c}{$1e-5$}  \\
    Dt State     & \multicolumn{3}{c}{$16$} \\ 
    Dt (Min, Max) & \multicolumn{3}{c}{$(0.001, 0.1)$} \\
    Dt Init. Strategy & \multicolumn{3}{c}{Random} \\
    Dt Init. Floor & \multicolumn{3}{c}{$1e-4$} \\
    Dt Scale & \multicolumn{3}{c}{$1.0$} \\
    Dt Softplus & \multicolumn{3}{c}{True} \\
    Projection Expansion Factor & \multicolumn{3}{c}{2} \\
    Short Conv Filter Size & \multicolumn{3}{c}{4} \\
    \bottomrule 
    \end{tabular}
    \label{tab:mamba-training-details}
\end{table}
\begin{table}[h!]
    \caption{Hyena Training Settings}
    \centering
    \begin{tabular}{rc}
    \toprule
    {} & 355M   \\
    \midrule
    Optimizer                 & Adam \\
    Optimizer momentum        & $\beta_1, \beta_2=0.9, 0.95$ \\
    Optimizer eps             & $1e-8$ \\
    Precision                 &  BFloat16 \\
    \midrule
    Warmup                    & 1\%    \\
    Learning rate decay       & Cosine \\
    Learning rate (min, base) & 8e-5, 8e-4 \\
    Global batch size         & 256 \\ 
    Weight decay              & 0.1  \\
    \midrule
    Num Layers  & 29  \\
    Hidden Size & 1024   \\
    Num Heads & 1 \\
    MLP Width & 2 \\
    Short Conv. Filter Size      & 3 \\
    Exp. Mod. Decay (Fast, Slow) & 0.3, 1.2 \\
    Filter Sine Freq. (w)        & 14 \\
    Filter Order                 & 64 \\
    Filter Inner MLP             & 2 \\
    \bottomrule 
    \end{tabular}
    \label{tab:hyena-training-details}
\end{table}
\begin{table}[h!]
    \caption{Hyena Training Settings}
    \centering
    \begin{tabular}{rc}
    \toprule
    {} & 355M   \\
    \midrule
    Optimizer                 & Adam \\
    Optimizer momentum        & $\beta_1, \beta_2=0.9, 0.99$ \\
    Optimizer eps             & $1e-8$ \\
    Precision                 &  BFloat16 \\
    \midrule
    Warmup                    & 1\%    \\
    Learning rate decay       & Cosine \\
    Learning rate (min, base) & 8e-5, 8e-4 \\
    Global batch size         & 256 \\ 
    Weight decay              & 0.1  \\
    \midrule
    Num Layers  & 24 (No Attention Layers) \\
    Hidden Size & 1024   \\
    Num Heads & 16 \\
    MLP Width & 4 \\
    \bottomrule 
    \end{tabular}
    \label{tab:h3-training-details}
\end{table}
\begin{table}[h!]
    \caption{Hyena Training Settings}
    \centering
    \begin{tabular}{rc}
    \toprule
    {} & 355M   \\
    \midrule
    Optimizer                 & Adam \\
    Optimizer momentum        & $\beta_1, \beta_2=0.9, 0.99$ \\
    Optimizer eps             & $1e-8$ \\
    Precision                 &  BFloat16 \\
    \midrule
    Warmup                    & 1\%    \\
    Learning rate decay       & Cosine \\
    Learning rate (min, base) & 8e-5, 8e-4 \\
    Global batch size         & 256 \\ 
    Weight decay              & 0.1  \\
    \midrule
    Num Layers  & 19  \\
    Hidden Size & 1024   \\
    MLP Width & 3.5 \\
    \bottomrule 
    \end{tabular}
    \label{tab:rwkv-training-details}
\end{table}
\begin{table}[h!]
    \caption{Gated Linear Attention (GLA) Training Settings}
    \centering
    \begin{tabular}{rc}
    \toprule
    {} & 355M   \\
    \midrule
    Optimizer          & Adam \\
    Optimizer momentum & $\beta_1, \beta_2=0.9, 0.95$ \\
    Optimizer eps      & $1e-8$ \\
    Precision          &  BFloat16 \\
    \midrule
    Warmup                    & 1\%    \\
    Learning rate decay       & Cosine \\
    Learning rate (min, base) & 8e-5, 8e-4 \\
    Global batch size         & 256 \\ 
    Weight decay              & 0.1  \\
    \midrule
    Num Layers  & 24  \\
    Hidden Size & 1024   \\
    Num Heads   & 4  \\
    MLP Width & 2 \\
    \bottomrule 
    \end{tabular}
    \label{tab:gla-training-details}
\end{table}

\end{document}